\documentclass{article}

\PassOptionsToPackage{numbers, compress}{natbib}

\usepackage[preprint]{neurips_2026}

\input{preamble.tex}

\definecolor{bgcolor}{rgb}{0.8,1,1}
\definecolor{bgcolor2}{rgb}{0.8,1,0.8}

\title{SILAGE: Memory-Efficient, Full-Gradient-Free Nonconvex Optimization for Nested Finite Sums}

\author{%
Igor Sokolov \\
KAUST\thanks{Center of Excellence for Generative AI, King Abdullah University of Science and Technology (KAUST), Thuwal, Saudi Arabia. Correspondence to: Igor Sokolov \texttt{<igor.sokolov.1@kaust.edu.sa>}.}
\And
Laurent Condat \\
KAUST
\And
Peter Richt\'arik \\
KAUST
}

\begin{document}

\maketitle
\begin{abstract}

Empirical risk minimization on massive datasets naturally exhibits a nested double finite-sum structure, where $N=nm$ total samples are logically or physically partitioned into $n$ blocks of size $m$ (e.g., in pooled data silos, out-of-core learning, or deliberate stratification). While variance-reduced methods achieve optimal oracle complexities for nonconvex objectives, they suffer from severe scaling bottlenecks in this centralized regime. Recursive estimators, such as \algname{PAGE}, require periodic global full-gradient refreshes over all $nm$ samples, which are computationally expensive. Conversely, single-loop methods, such as \algname{SILVER}, avoid such refreshes but require an impractical $\cO\rb{nm}$ memory footprint to store a control variate for every sample. In this paper, we propose \algname{SILAGE}, a variance-reduced algorithm that addresses this trade-off. By actively exploiting the double-sum structure, \algname{SILAGE} eliminates periodic global full-gradient refreshes over all $nm$ components---evaluating at most one local group gradient per iteration---while requiring only $\cO\rb{n}$ memory. Furthermore, we provide a tight convergence analysis that avoids pessimistic worst-case Lipschitz constants. Instead, \algname{SILAGE}'s complexity natively adapts to the underlying data geometry via nested functional similarities: across-group ($\delta_1$) and within-group ($\delta_2$) heterogeneity. Our results improve existing state-of-the-art bounds in several practically relevant regimes.

\end{abstract}

\section{Introduction}\label{sec:intro}

\subsection{The centralized double-sum paradigm}

Nonconvex optimization is ubiquitous across all areas of machine learning and artificial intelligence. In particular, empirical risk minimization (ERM) is the dominant paradigm in modern large-scale settings for supervised model training. While ERM is typically formulated as a flat finite sum of $N$ components, modern centralized training pipelines increasingly exhibit a \emph{nested double-sum structure} $N = nm$, where data are logically or physically partitioned into $n$ groups, blocks, or silos, each containing $m$ samples \citep{kon16a, mcm17, evgeniou2004regularized, caruana1997multitask, sadievdon, yi2024cohort, sokolov2024marina}.

Thus, we consider nonconvex double finite-sum problems of the form
\begin{equation}\label{eq:prob}
\min \limits_{x\in \R^d}  \left[f(x) \coloneqq \frac{1}{n}\sum \limits_{i=1}^n{f_i(x)} \right],\quad \text{with} \quad f_i(x) \coloneqq \frac{1}{m}\sum_{j=1}^m{f_{i,j}(x)} ,
\end{equation}
where $d$ is the model dimension, $[n]\coloneqq\{1,\ldots,n\}$, and $[m]\coloneqq\{1,\ldots,m\}$. The possibly nonconvex function $f_{i,j}:\R^d\to\R$ typically represents the loss on the $j$-th sample in the $i$-th group, and $f_i$ is the averaged objective function of that specific group.

Our analysis assumes that the aggregate objective $f$ is lower bounded and $L$-smooth for some $L>0$. Crucially, however, we do \textbf{not} require every individual $f_{i,j}$ to be smooth, in contrast to many existing methods (see Section~\ref{sec:assumptions} for our detailed assumptions). While standard gradient descent solves \eqref{eq:prob} with a gradient evaluation complexity\footnote{Here and throughout, the notation $g(t)=\cO\rb{h(t)}$ means that there exist a constant $C>0$ and a threshold $t_0$ such that $|g(t)|\le C\,h(t)$ for all $t\ge t_0$; that is, $h$ bounds $g$ up to a universal multiplicative constant.} of $\cO\rb{nm L \frac{\Delta_0}{\epsilon}}$, variance-reduced (VR) stochastic gradient methods are now standard tools for reducing this complexity \citep{johnson2013accelerating, fang2018spider, li21, sadiev2025improved, sokolov2025bernoulli, sokolov2022non}.

Exploiting this two-layer partition is vital because it naturally maps to several critical, single-machine settings where standard flat VR methods fail to scale \citep{shai_book, Bottou2012, goodfellow2016deep}:
\begin{itemize}[leftmargin=*]
    \item \textbf{Data Silos and Pooled Data Lakes:} Imagine a centralized data center where data are legally or physically partitioned into $n$ distinct ``silos.'' For example, a healthcare consortium aggregating patient records from $n$ different hospitals, or an enterprise training on $n$ different data modalities (text, images, logs) stored on a centralized cluster. While the computation is centralized—avoiding the communication latency and synchronization bottlenecks of federated learning—the data inherently retain their siloed structure. Each silo $i$ contains $m$ samples that share regional demographics or domain-specific features.
    \item \textbf{Out-of-Core and Memory-Constrained Learning:} In massive-scale training, such as optimizing foundation models on billions of tokens \citep{radford2019language, brown2020language}, the dataset $N$ vastly exceeds GPU VRAM or main memory. Data must be stored on slower NVMe/SSDs and loaded in chunks. Here, $n$ represents the number of physical storage blocks (or shards), and $m$ is the block size. Algorithms that require maintaining a state (control variate) for every single sample $j \in [N]$ incur an $\cO\rb{N}$ memory footprint, which is impractical in this regime. Conversely, storing just $n$ control variates (one per block) reduces this memory overhead by a factor of $m$.
    \item \textbf{Deliberate Stratification and Clustering:} Even if the data are provided as a flat array, an optimizer can actively induce a double-sum structure to accelerate training. By clustering data via lightweight preprocessing—such as grouping by class labels or applying $k$-means to input features—the dataset is partitioned into $n$ highly homogeneous clusters. This stratification creates a favorable variance geometry that our algorithm can natively exploit.
\end{itemize}

We emphasize that our setting is grouped (nested) finite-sum optimization in the centralized regime: federated and client-partitioned data serve as motivating special cases rather than the object of study, and our guarantees bound gradient (oracle) complexity and memory, not communication complexity.

\subsection{Related work and contributions}

All existing VR estimators exhibit scaling limitations at massive sizes, forcing a difficult trade-off between computation and memory. 

The first category consists of methods that necessitate periodic global \textbf{full-gradient refreshes} over all $nm$ samples.
This includes estimators such as \algname{SVRG} \citep{johnson2013accelerating,kovalev2020don}, \algname{SCSG} \citep{lei2017non}, and \algname{SSRGD} \citep{li2019ssrgd}, as well as recursive estimators such as \algname{SARAH}  \citep{ngu17},
 \algname{SPIDER} \citep{fang2018spider}, \algname{SpiderBoost} \citep{wang2018spiderboost}, and \algname{PAGE} \citep{li21, PAGE-AB} (also independently proposed as Loopless-SARAH (\algname{L2S})
 \citep{li20b}). Periodic global full-gradient refreshes are expensive and limit the scalability of these methods on very large datasets.

The second category comprises memory-based alternatives, such as classical \algname{SAG}/\algname{SAGA}-type estimators \citep{schmidt2017minimizing, defazio2014saga, bibi2018improving, demidovich2024guide} and recent methods like \algname{ZeroSARAH} \citep{li2021zerosarah} and \algname{SILVER} \citep{oko24}. These methods successfully avoid periodic global full-gradient refreshes. However, when applied to a flattened dataset, they suffer from an impractical \textbf{$\cO\rb{nmd}$ memory footprint} required to store a $d$-dimensional control variate for every individual sample. 

Thus, in the centralized large-scale regime considered here, the primary bottlenecks are \emph{gradient evaluation complexity} and \emph{memory footprint}. The $n$-group structure lets us navigate this trade-off. We introduce \textbf{\algname{SILAGE}} (\emph{SIngle Loop Average Gradient Estimator}), a novel variance-reduced algorithm tailored for double finite-sum problems. Our contributions are organized along three structural axes--- memory footprint,  avoidance of periodic global full-gradient refreshes, and  dependence of the convergence rate on the data geometry---together with a unification result that places \algname{SILAGE} within the existing landscape:
\begin{itemize}
    \item \textbf{Memory Efficiency via Nested Structure:} By maintaining $n$ control-variate table entries (one per block), plus $\cO(1)$ auxiliary $d$-vectors---namely the running aggregate $g^t=\frac{1}{n}\sum_{i=1}^n g_i^t$ that supplies the descent direction in Algorithms~\ref{algsilage1} and~\ref{algsilage2a}---\algname{SILAGE} requires only $\cO\rb{n}$ memory\footnote{Here ``$\cO\rb{n}$ memory'' counts, beyond the iterate itself, the $n+\cO(1)$ additional $d$-dimensional vectors just described. In \texttt{float64} precision this is approximately $8nd$ bytes for large $n$, ignoring lower-order implementation overhead such as sampled indices, array metadata, and temporary workspace.}---a substantial reduction compared to the prohibitively large $\cO\rb{nm}$ footprint of flattened methods such as \algname{SILVER}.
    \item \textbf{No Periodic Global Full-Gradient Refresh:} \algname{SILAGE} never performs a global full-gradient pass over all $nm$ components: each anchor reset evaluates only a single \emph{local} group gradient $\nabla f_{i^t}(x^{t+1})$, at an $\cO\rb{m}$ per-iteration cost. The method also permits arbitrary initial estimators $g_i^0$, so it requires neither the global full gradient $\nabla f(x^0)$ nor the local full gradients $\nabla f_i(x^0)$ at initialization; all initialization error enters the complexity only through the Lyapunov term $\Psi^0$ (Corollaries~\ref{cor:silage_mge_n_arbitrary_init} and~\ref{cor:silage_ngt_m_arbitrary_init}).
    \item \textbf{Exploitation of Nested Similarity:} We derive two-regime nonconvex convergence rates (for $m \ge n$ and $n > m$) that explicitly leverage nested similarity constants: $\delta_1$ (across-group) and $\delta_2$ (within-group). This structured decomposition is fundamentally more expressive and informative than standard flattened similarity bounds (detailed in Section~\ref{sec:assumptions}).
    \item \textbf{Exact Limiting Reductions (unification):} Under matching parameter choices, \algname{SILAGE} unifies existing approaches by exactly recovering \algname{PAGE} (when $n=1$) and \algname{SILVER} (when $m=1, b=1$), hence the algorithm's name.
\end{itemize}

At the level of building blocks, \algname{SILAGE} inherits two standard ideas: \algname{PAGE}-style recursive gradient tracking and the stored control variates of memory-based (\algname{SAGA}/\algname{SILVER}-style) estimators. Its novelty lies in how these are combined and analyzed for the nested problem: (i) an \emph{estimator coupling} that drives all $n$ group estimators through a single global coin and one uniformly sampled group, so that at most one group is fully refreshed per iteration; (ii) storing one control variate \emph{per group} rather than per sample, giving the $\cO\rb{n}$ footprint in place of the flattened $\cO\rb{nm}$; (iii) a \emph{shared-drift} update that handles the $n>m$ regime by correcting un-sampled groups without fresh per-group gradients; and (iv) a \emph{two-level} nonconvex analysis in the nested similarity constants $\delta_1,\delta_2$.

Table~\ref{tab:method-comparison} contrasts \algname{SILAGE} with representative baselines for nonconvex (double) finite-sum optimization along these axes, summarizing the contributions above.

\begin{table}[t]
\centering
\footnotesize
\caption{Positioning of \algname{SILAGE} relative to representative baselines for nonconvex (double) finite-sum optimization. A checkmark (\cmark) indicates that the property is satisfied; a cross (\xmark) that it is not. The \emph{Memory} column reports the control-variate memory footprint, up to universal constants. See the notes below for the precise meaning of each column.}
\label{tab:method-comparison}
\begin{threeparttable}
\renewcommand{\arraystretch}{1.1}
\begin{tabular}{|cc|ccc|c|p|}
\hline
\textbf{Algorithm} &
\begin{tabular}{c}\textbf{Full-}\\\textbf{gradient-}\\\textbf{free}\,\textsuperscript{\textcolor{blue}{(1)}}\end{tabular} &
\begin{tabular}{c}\textbf{Nested}\\\textbf{structure}\,\textsuperscript{\textcolor{blue}{(2)}}\end{tabular} &
\begin{tabular}{c}\textbf{Memory}\\\textbf{footprint}\,\textsuperscript{\textcolor{blue}{(3)}}\end{tabular} &
\textbf{Similarity}\,\textsuperscript{\textcolor{blue}{(4)}} &
\centering\arraybackslash \textbf{Note}\,\textsuperscript{\textcolor{blue}{(5)}} \\
\hline\hline
\begin{tabular}{c}\algname{GD}\end{tabular} & \xmark & \xmark & $\cO\rb{1}$ & N/A & Classic full-gradient baseline. \\
\hline
\begin{tabular}{c}\algname{SILVER}\\\citep{oko24}\end{tabular} & \cmark & \xmark & $\cO\rb{nm}$ & Flattened & Stores $nm$ control variates. \\
\hline
\begin{tabular}{c}\algname{PAGE}\\\citep{li21}\end{tabular}& \xmark & \xmark & $\cO\rb{1}$ & Flattened & Periodic global full-gradient refreshes. \\
\hline
\begin{tabular}{c}\algname{D-PAGE}\\\citep{PAGE-AB}\end{tabular}& \xmark & \cmark & $\cO\rb{1}$ & Flattened & Periodic global full-gradient refreshes. \\
\hline
\begin{tabular}{c}\algname{ZeroSARAH}\\\citep{li2021zerosarah} \end{tabular}& \cmark & \xmark & $\cO\rb{nm}$ & None & Relies on worst-case $L_{\max}$. \\
\hline
\begin{tabular}{c}\algname{D-ZeroSARAH}\\\citep{li2021zerosarah}\end{tabular} & \cmark & \cmark & $\cO\rb{nm}$ & None & Nested distributed variant; relies on $L_{\max}$. \\
\hline
\rowcolor{bgcolor}
\begin{tabular}{c}\algname{SILAGE}\\ {\bf[OURS]} \end{tabular} & \cmark & \cmark & $\cO\rb{n}$ & Nested & Proposes a novel nested similarity ($\delta_1,\delta_2$). \\
\hline
\end{tabular}
\begin{tablenotes}
\footnotesize
\item [{\color{blue}(1)}] \emph{Full-gradient-free}: the method's theoretical guarantees do not require periodic \emph{global} full-gradient refreshes over all $nm$ components to reach an $\epsilon$-stationary point. A full-gradient-free method may still evaluate at most one \emph{local} group gradient $\nabla f_i$ per iteration; \algname{SILAGE} does so at each anchor reset, at $\cO\rb{m}$ cost. The one-time $nm$ cost of exact initialization, used only for the like-for-like comparison in Section~\ref{sec:comp-exact-init} (Table~\ref{tab:oracle-complexities-full-init}), is separate and is not required by \algname{SILAGE} (Corollaries~\ref{cor:silage_mge_n_arbitrary_init} and~\ref{cor:silage_ngt_m_arbitrary_init}).
\item [{\color{blue}(2)}] \emph{Nested structure}: the method explicitly exploits the nested double finite-sum structure of the objective.
\item [{\color{blue}(3)}] \emph{Memory footprint}: the worst-case number of additional $d$-dimensional vectors (control variates) maintained throughout execution, reported up to universal constants.
\item [{\color{blue}(4)}] \emph{Similarity}: the type of functional-similarity assumption underlying the convergence analysis (flattened, nested, or none).
\item [{\color{blue}(5)}] \emph{Note}: a brief additional remark on the method.
\end{tablenotes}
\end{threeparttable}
\end{table}

Table~\ref{tab:properties_comparison} complements this overview with the limiting behavior of \algname{SILAGE}: for \algname{SILVER}, \algname{PAGE}, and \algname{SILAGE}, it reports the algorithm exactly recovered at the structural boundaries $m=1$ and $n=1$, highlighting how \algname{SILAGE} formally unifies \algname{SILVER} and \algname{PAGE}.

\begin{table}[t]
\centering
\caption{Limiting behavior of \algname{SILAGE} at the structural boundaries. For \algname{SILVER}, \algname{PAGE}, and \algname{SILAGE}, the table reports the algorithm exactly recovered when $m=1$ or $n=1$.}
\label{tab:properties_comparison}
\vspace{2mm}
\renewcommand{\arraystretch}{1.4}
\begin{tabular}{l|cc}
\hline
\bf Algorithm & $\boldsymbol{m=1}$ & $\boldsymbol{n=1}$ \\
\hline \hline
\algname{SILVER} on all $f_{i,j}$ &  \algname{SILVER} &  \algname{SILVER} \\
\algname{SILVER} on the $f_i$ &  \algname{SILVER} &  \algname{GD} \\
\hline
\algname{PAGE} on all $f_{i,j}$ &  \algname{PAGE} &  \algname{PAGE} \\
\algname{PAGE} on the $f_i$ &  \algname{PAGE} &  \algname{GD} \\
\hline
\rowcolor{bgcolor} 
\textbf{\algname{SILAGE}[OURS]}  &  \algname{SILVER} & \algname{PAGE} \\
\hline
\end{tabular}
\end{table}

\paragraph{Terminology.}
We use the following terms consistently throughout. A \emph{periodic global full-gradient refresh} is the repeated evaluation, once every few iterations, of the full gradient $\nabla f(x)=\frac{1}{nm}\sum_{i=1}^{n}\sum_{j=1}^{m}\nabla f_{i,j}(x)$ over all $nm$ components, as required by \algname{PAGE}-style estimators; we reserve the word \emph{refresh} for this recurring operation. It is distinct from a single \emph{local group gradient} $\nabla f_i$, which touches only the $m$ components of one group at an $\cO\rb{m}$ cost and is the only full-gradient computation \algname{SILAGE} performs (at each anchor reset). We further distinguish two one-time choices for the initial estimators $g_i^0$: under \emph{exact initialization} one sets $g_i^0=\nabla f_i(x^0)$ for all $i\in[n]$, at a single $nm$ component-gradient cost, whereas under \emph{arbitrary initialization} the $g_i^0$ are chosen freely (e.g., $g_i^0=0$) with no initial full-gradient pass. Accordingly, the \emph{full-gradient-free} property recorded in Table~\ref{tab:method-comparison} refers solely to the absence of periodic global full-gradient refreshes and is independent of the initialization choice.

\section{Assumptions}\label{sec:assumptions}

We first recall standard assumptions from nonconvex optimization. The following two assumptions are classical and are used in essentially the same form in the analysis of gradient descent and modern variance-reduced methods for nonconvex finite sums \citep{nes04, li21, li2021zerosarah, PAGE-AB}.

\begin{assumption}[Lower boundedness]\label{ass:lower-bounded}
There exists $\finf\in\R$ such that $f(x)\ge \finf$ for all $x\in\R^d$.
\end{assumption}

\begin{assumption}[Smoothness of the average objective]\label{ass:f-smooth}
There exists $L\ge 0$ such that, for all $x,y\in\R^d$,
\begin{align}
    \squeeze
    \norm{\nabla f(x)-\nabla f(y)}\le L\norm{x-y}.
    \notag
\end{align}
\end{assumption}

A key point of our setting is that Assumption~\ref{ass:f-smooth} is imposed strictly on the average objective $f$. Unless explicitly stated otherwise, we do not require every component $f_{i,j}$ to be individually smooth; they are merely assumed to be differentiable. Instead, we control the variance geometry of the double finite sum in \eqref{eq:prob} through two nested similarity conditions.

\begin{assumption}[Across-group similarity]\label{ass:delta1}
There exists $\delta_1\ge 0$ such that, for all $x,y\in\R^d$,
\begin{align}
    \frac1n\sum_{i=1}^n
    \sqn{
    \nabla f_i(x)-\nabla f(x)-\nabla f_i(y)+\nabla f(y)}
    \le \delta_1^2\sqn{x-y}.
    \label{eqs1}
\end{align}
\end{assumption}

\begin{assumption}[Within-group similarity]\label{ass:delta2}
There exists $\delta_2\ge 0$ such that, for all $x,y\in\R^d$,
\begin{align}
    \frac1{nm}\sum_{i=1}^n\sum_{j=1}^m
    \sqn{
    \nabla f_{i,j}(x)-\nabla f_i(x)-\nabla f_{i,j}(y)+\nabla f_i(y)}
    \le \delta_2^2\sqn{x-y}.
    \label{eqs2}
\end{align}
\end{assumption}

Assumption~\ref{ass:delta1} measures how much the group gradients deviate from the global gradient, while Assumption~\ref{ass:delta2} measures the average deviation of sample gradients from their respective group averages.

These are structural conditions of a different nature than component-wise smoothness: the latter is one sufficient condition for them, yet they can be much smaller than the worst-case component smoothness in favorable grouped and statistical regimes. For example, if all $f_i$ are $L_1$-smooth, then \eqref{eqs1} holds with $\delta_1\le L_1$; similarly, a mean-squared smoothness bound on the individual functions $f_{i,j}$ implies \eqref{eqs2}. In many distributed and statistical regimes, similarity constants can be substantially smaller than pessimistic smoothness bounds \citep{chayti2021linear,kovalev2022optimal,beznosikov2022compression,tian2022acceleration,horvath2022fedshuffle,khaled2022faster,karimireddy2020mime,CK22}. Furthermore, in i.i.d. ERM settings, related similarity quantities decrease proportionately with the local sample size under standard statistical assumptions \citep{vapnik1971uniform,zhang2018communication}. Importantly, \algname{SILAGE} still assumes the average objective $f$ is $L$-smooth (Assumption~\ref{ass:f-smooth}), and only avoids dependence on component-wise or worst-case smoothness.

While similarity assumptions are a well-established tool in variance reduction \citep{karimireddy2020mime, oko24}, the vast majority of existing baselines, including \algname{SILVER} \citep{oko24} and \algname{PAGE} \citep{PAGE-AB}, consider only a ``flattened'' similarity assumption of the form:
\begin{align}
    \frac1{nm}\sum_{i=1}^n\sum_{j=1}^m
    \sqn{\nabla f_{i,j}(x)-\nabla f_{i,j}(y)-\nabla f(x)+\nabla f(y)}
    \le \delta_{\mathrm{flat}}^2\sqn{x-y}.
    \label{eq:flat-similarity}
\end{align}

Note that \eqref{eq:flat-similarity} is linked to \eqref{eqs1} and \eqref{eqs2} via a simple algebraic equality:
\begin{align*}
    &\frac{1}{nm}\sum_{i=1}^n \sum_{j=1}^m \sqn{\nabla f_{i,j}(x)-\nabla f(x)-\nabla f_{i,j}(y)+\nabla f(y)} \\
    &\quad= \frac{1}{nm}\sum_{i=1}^n \sum_{j=1}^m \sqn{\nabla f_{i,j}(x)-\nabla f_i(x)-\nabla f_{i,j}(y)+\nabla f_i(y)} \\
    &\qquad+ \frac{1}{n}\sum_{i=1}^n\sqn{\nabla f_{i}(x)-\nabla f(x)-\nabla f_{i}(y)+\nabla f(y)} .
\end{align*}
Hence Assumptions \ref{ass:delta1}--\ref{ass:delta2} imply \eqref{eq:flat-similarity} with
\begin{equation*}
\delta_{\mathrm{flat}} \le \sqrt{\delta_1^2+\delta_2^2}.
\end{equation*}
Conversely, a flat bound controls both nested residuals, but the minimal constants on the two sides need not be related by a clean equality, since a supremum over $x,y$ of a sum need not equal the sum of the individual suprema. The two-parameter representation by the pair $(\delta_1,\delta_2)$ is therefore strictly more informative and provides a significantly finer algorithmic characterization. For instance, if there is no within-group variation, then $\delta_2=0$ and $\delta_{\mathrm{flat}}=\delta_1$; conversely, if there is no across-group variation, then $\delta_1=0$ and $\delta_{\mathrm{flat}}=\delta_2$. As we will demonstrate, \algname{SILAGE} successfully decouples these dependencies depending on the operational regime.

In the optimization literature, component-level smoothness assumptions are typically formulated as follows \citep{li21, PAGE-AB, li2021zerosarah, dasha23, zhao21}:
\begin{align}
    \squeeze
    \norm{\nabla f_i(x)-\nabla f_i(y)} &\leq L_i\norm{x-y}, \quad \forall x, y \in \R^d, \label{eq:smooth_i} \\
    \norm{\nabla f_{i,j}(x)-\nabla f_{i,j}(y)} &\leq L_{i,j}\norm{x-y}, \quad \forall x, y \in \R^d, \label{eq:smooth_ij}
\end{align}
for some constants $L_i, L_{i,j} < \infty$. The worst-case component smoothness is commonly denoted by $L_{\max} \eqdef \max_{i,j} L_{i,j}$. Additionally, some analyses utilize an averaged smoothness assumption:
\begin{align}
    \frac{1}{n} \sum_{i=1}^n\sqn{\nabla f_i(x)-\nabla f_i(y)} \leq L_{+}^2\sqn{x-y}, \quad \forall x, y \in \R^d,
    \label{eq:smooth_plus}
\end{align}
for some $L_{+} < \infty$. These classical assumptions naturally imply a hierarchy of smoothness constants: \[L^2 \le L_{+}^2 \leq \frac{1}{n} \sum_{i=1}^n L_i^2 \le L_{\max}^2.\]

Relying on these inequalities, existing methods yield oracle complexities dominated by the larger constants in this hierarchy. For instance, \algname{PAGE} \citep{li21} and \algname{ZeroSARAH} \citep{li2021zerosarah} achieve theoretical oracle complexities of $\cO\rb{nm + L_{+}\sqrt{nm}\frac{\Delta_0}{\epsilon}}$ and $\cO\rb{nm + L_{\max}\sqrt{nm}\frac{\Delta_0}{\epsilon}}$, respectively.

The Hessian-variance framework proposed by \citet{PAGE-AB} tightens this dependence. Under uniform weights in the flattened regime, it yields a complexity of \[\cO\rb{nm + \rb{L+\sqrt{nm}\delta_{\mathrm{flat}}}\frac{\Delta_0}{\epsilon}}.\] Furthermore, in the grouped regime under full client participation, it yields \[\cO\rb{nm+\rb{nL+\sqrt{nm}\bar\delta_2}\frac{\Delta_0}{\epsilon}},\] where $\bar\delta_2^2\eqdef \frac1n\sum_{i=1}^n\delta_{2,i}^2$, and $\delta_{2,i} \ge 0$ is a localized within-group constant such that, for all $x,y\in\R^d$,
\begin{align}
    \frac1m\sum_{j=1}^m
    \sqn{
    \nabla f_{i,j}(x)-\nabla f_i(x)
    -\nabla f_{i,j}(y)+\nabla f_i(y)}
    \leq
    \delta_{2,i}^2\sqn{x-y}.
    \label{eq:delta2i}
\end{align}

We emphasize that the total oracle complexity of \algname{SILAGE} depends strictly on $L$, $\delta_1$, and $\delta_2$, entirely bypassing any dependence on the potentially huge worst-case constant $L_{\max}$. This structural advantage makes our theoretical rates significantly tighter in heterogeneous regimes compared to existing component-smoothness approaches.

The fundamental novelty of considering two layers of similarity ($\delta_1$ and $\delta_2$) is not that their combination is always strictly smaller than the flattened similarity. Rather, Assumptions~\ref{ass:delta1} and \ref{ass:delta2} rigorously decompose the flattened residual into distinct between-group and within-group components. As we demonstrate in our complexity analysis, \algname{SILAGE} natively exploits this explicit separation: depending on the dataset dimensions, either $\delta_2$ alone dictates the leading complexity (in the $m \ge n$ regime), or both constants enter as multipliers of different dimensional terms (in the $n > m$ regime). To the best of our knowledge, prior variance-reduction guarantees for double finite sums do not yield this two-level dependence on both $\delta_1$ and $\delta_2$ while maintaining an $\cO\rb{n}$ memory footprint via a single estimator per group. 

This dual-parameterization $(\delta_1,\delta_2)$ provides a fundamentally more informative view of the geometry of~\eqref{eq:prob} than a single flattened similarity constant. By decoupling distinct sources of heterogeneity---where $\delta_1$ measures the variation of the group objectives $f_i$ around the global objective $f$, and $\delta_2$ measures the average variation of the component functions $f_{i,j}$ around their local group averages $f_i$---this framework naturally captures several practically relevant data regimes:
\begin{itemize}
    \item \textbf{Globally and locally homogeneous data ($\delta_1$ small, $\delta_2$ small).}
    This regime arises when the groups are mutually similar and each group is internally coherent. Examples include repeated measurements collected under nearly identical experimental conditions, data acquired across standardized sites following the same protocol, or stable temporal windows of a slowly varying process. In such settings, both across-group and within-group deviations are mild.

    \item \textbf{Redundant blocks ($\delta_1$ small, $\delta_2$ relatively large).}
    This regime is induced by the random sharding of a large dataset: each group behaves as a representative subsample of the global population. Consequently, different groups are statistically close to one another ($\delta_1$ is small), while the internal sample-level diversity inside each group remains substantial ($\delta_2$ is large relative to $\delta_1$). This situation is common in distributed ERM when data are partitioned purely for computational scaling rather than semantic grouping.

    \item \textbf{Homogeneous silos ($\delta_1$ large, $\delta_2$ small).}
    Here, each group is internally coherent, but different groups correspond to genuinely distinct populations, domains, tasks, or modalities. Examples include clustered federated learning, client-specific personalization, or task-based partitions. Within-group variability is minimal, but the discrepancy between local objectives is pronounced.

    \item \textbf{Heterogeneous silos ($\delta_1$ large, $\delta_2$ large).}
    This regime captures challenging large-scale settings in which different groups correspond to distinct populations and each individual group is itself internally diverse. Examples include hierarchical federated learning architectures \citep{liu2020client} spanning institutions with different demographics and broad case mixes, cross-market recommendation systems, or multi-domain corpora. In these cases, both levels of heterogeneity are substantial.
\end{itemize}

A single flattened similarity constant inherently obscures these qualitatively different geometries. By contrast, the $(\delta_1,\delta_2)$ pair makes the multiscale heterogeneity explicit, allowing the resulting oracle complexity bounds to cleanly separate the structural impact of across-group and within-group variation.

Crucially, \algname{SILAGE} is 
 \emph{agnostic} to these parameters during execution; it does not require prior knowledge of whether $\delta_1$ or $\delta_2$ dominates. It naturally adapts to whatever underlying data geometry is present, yielding theoretical complexities that strictly improve a posteriori based on the relative tightness of $\delta_1$ and $\delta_2$.

\section{\algname{SILAGE} Algorithm and Convergence Results}
\label{sec:silage_algorithms}

\begin{figure}[t]
\centering
\begin{minipage}[t]{0.485\textwidth}
\begin{algorithm}[H]
\small
\caption{\algname{SILAGE} (\textbf{new}), case $m\geq n$}
\label{algsilage1}
\begin{algorithmic}[1]
    \REQUIRE Stepsize $\gamma > 0$, probability $p \in (0,1]$
    \STATE \textbf{Initialize:} $x^0, g_1^0, \ldots, g_n^0 \in \mathbb{R}^d$,  $g^0 \eqdef \frac{1}{n}\sum_{i} g_i^0$
    \FOR{$t=0,1,\ldots$}
        \STATE $x^{t+1} \eqdef x^t - \gamma g^t$
        \STATE Sample $\theta^t \sim \text{Bernoulli}(p)$
        \IF{$\theta^t = 1$}
            \STATE Sample $i^t \sim \mathcal{U}([n])$
            \STATE $g_{i^t}^{t+1} \eqdef  \nabla f_{i^t}(x^{t+1})$ \hfill \COMMENT{Anchor reset}
            \STATE $\mathcal{I}^t \eqdef  [n] \setminus \{i^t\}$
        \ELSE
            \STATE $\mathcal{I}^t \eqdef  [n]$
        \ENDIF
        \FOR{$i \in \mathcal{I}^t$}
            \STATE Sample $j_i^t \sim \mathcal{U}([m])$
            \STATE $\Delta_i^t \eqdef  \nabla f_{i,j_i^t}(x^{t+1}) - \nabla f_{i,j_i^t}(x^t)$
            \STATE $g_i^{t+1} \eqdef  g_i^t + \Delta_i^t$ \hfill \COMMENT{Recursive update}
        \ENDFOR
        \STATE $g^{t+1} \eqdef  \frac{1}{n}\sum_{i=1}^n g_i^{t+1}$
    \ENDFOR
\end{algorithmic}
\end{algorithm}
\end{minipage}\hfill
\begin{minipage}[t]{0.485\textwidth}
\begin{algorithm}[H]
\small
\caption{\algname{SILAGE} (\textbf{new}), case $n>m$}
\label{algsilage2a}
\begin{algorithmic}[1]
    \REQUIRE Stepsize $\gamma > 0$, active group count $b_{\mathrm{grp}} \in [n]$
   \STATE \textbf{Initialize:} $x^0, g_1^0, \ldots, g_n^0 \in \mathbb{R}^d$,  $g^0 \eqdef \frac{1}{n}\sum_{i} g_i^0$
    \FOR{$t=0,1,\ldots$}
        \STATE $x^{t+1} \eqdef x^t - \gamma g^t$
        \STATE Sample $i^t \sim \mathcal{U}([n])$
        \STATE $g_{i^t}^{t+1} \eqdef \nabla f_{i^t}(x^{t+1})$ \hfill \COMMENT{Anchor reset}
        \STATE Sample $\Omega^t \subset [n] \!\setminus \!\{i^t\}$ unif., with $|\Omega^t| = b_{\mathrm{grp}}-1$%
        \STATE $\tilde{\Omega}^t \eqdef \Omega^t \cup \{i^t\}$
        \FOR{$i \in \Omega^t$}
            \STATE Sample $j_i^t \sim \mathcal{U}([m])$
            \STATE $\Delta_i^t \eqdef \nabla f_{i,j_i^t}(x^{t+1}) - \nabla f_{i,j_i^t}(x^t)$
            \STATE $g_i^{t+1} \eqdef g_i^t + \Delta_i^t$ \hfill \COMMENT{Recursive update}
        \ENDFOR
        \STATE Sample $j_{i^t}^t \sim \mathcal{U}([m])$
        \STATE $\Delta_{i^t}^t \eqdef \nabla f_{i^t,j_{i^t}^t}(x^{t+1}) - \nabla f_{i^t,j_{i^t}^t}(x^t)$
        \STATE $d^t \eqdef \frac{1}{b_{\mathrm{grp}}}\sum_{i \in \tilde{\Omega}^t} \Delta_i^t$ \hfill \COMMENT{Compute shared drift}
        \FOR{$i \in [n] \setminus \tilde{\Omega}^t$}
            \STATE $g_i^{t+1} \eqdef g_i^t + d^t$ 
        \ENDFOR
        \STATE $g^{t+1} \eqdef \frac{1}{n}\sum_{i=1}^n g_i^{t+1}$
    \ENDFOR
\end{algorithmic}
\end{algorithm}
\end{minipage}
\end{figure}

To obtain a single-loop estimator that entirely avoids periodic global full-gradient refreshes while exploiting the two-level structure of~\eqref{eq:prob}, \algname{SILAGE} relies on a carefully constructed fusion of two variance-reduction philosophies: it applies a \algname{SILVER}-like memory structure across the $n$ groups, and a \algname{PAGE}-like recursive tracking within the $m$ components of each group. However, naively deploying $n$ independent \algname{PAGE} estimators (one for each group) would fail to yield a practical algorithm. In standard \algname{PAGE}, an estimator is refreshed with the exact gradient with probability $1/m$. If we ran $n$ independent coin flips per iteration, the number of groups triggering a full refresh simultaneously would be a binomial random variable. Consequently, the number of component gradients evaluated per iteration would be $nm$ in the worst case, which is precisely what we want to avoid. 
%
To guarantee a strictly bounded maximum batch size, \algno carefully couples the two strategies. The group estimators are correlated in a specific way
that depends on the aspect ratio of the dataset: 

\begin{itemize}[leftmargin=*]
    \item \textbf{The $m \ge n$ regime (Algorithm~\ref{algsilage1}):} Our design target is that each of the $n$ group estimators be fully refreshed with average probability $1/m$ per iteration, matching the refresh rate of a \algname{PAGE} estimator on a block of $m$ components. To meet this target while never refreshing more than one group per iteration, we couple a \emph{single global coin} with a uniform group selection: with probability $p$ the coin lands heads, and conditioned on heads we select exactly \emph{one} group $i^t$ uniformly at random to undergo a full local refresh $\nabla f_{i^t}(x^{t+1})$. Under this scheme, the marginal probability that any fixed group is refreshed in a given iteration is $p\cdot\frac{1}{n}$; equating this with the target rate $1/m$ yields $p = n/m$, which is a valid probability precisely because $m \ge n$. Thus Algorithm~\ref{algsilage1} evaluates at most one local group gradient $\nabla f_{i^t}$ per iteration, and only on reset iterations (those for which the coin lands heads); consequently, the maximum batch size stays at $\cO\rb{m+n}=\cO\rb{m}$. In this regime, the \algname{PAGE} estimator is fully operational on each block independently; thus, the blocks do not interfere with each other, and across-group heterogeneity ($\delta_1$) does not affect the leading complexity.
        \item \textbf{The $n > m$ regime (Algorithm~\ref{algsilage2a}):} The previous strategy fails because $n/m > 1$.
        To maintain an $\cO\rb{m}$ per-iteration cost, we cannot compute multiple (ideally $n/m$) group refreshes $\nabla f_{i}$, so we compute exactly one.
        Moreover, updating the remaining $n-1$ estimators with fresh local data would cost $\cO\rb{n}$, breaking our $\cO\rb{m}$ budget. We circumvent this by selecting only a small subset of $b$ groups (typically $b=m$) to receive fresh gradient evaluations. There is no gradient evaluation for the remaining un-sampled groups; instead, their estimators are updated using a \emph{shared drift} computed exclusively from the $b$ active groups. We pay a theoretical price for this cost-saving design: because the un-sampled groups rely on updates generated by other groups, the across-group heterogeneity ($\delta_1$) enters the convergence complexity.
        \end{itemize}
At the exact boundary where $n=m$, both formulations naturally coincide under the right parameter choices. In the following subsections, we state the convergence guarantees for both regimes. Both versions store exactly one estimator $g_i^t$ for each group $f_i$ and use their average $g^t \eqdef \frac1n\sum_{i=1}^n g_i^t$ as the descent direction.

\subsection{Convergence in the $m \ge n$ regime}
\label{subsec:silage_mge_n_convergence}

We first state our main convergence result covering the $m\ge n$ regime.
\begin{restatable}[General convergence, $m\ge n$]{theorem}{silageMgeNNC}
\label{theo1}
Let \Cref{ass:lower-bounded} (lower-boundednes), \Cref{ass:f-smooth} ($L$-smoothness), and \Cref{ass:delta2} (sithin-group $\delta_2$-similarity) hold, with $m\ge n$. In \algno (Algorithm~\ref{algsilage1}),
assume the stepsize satisfies \[0<\gamma\le \rb{L+\delta_2\sqrt{\frac{n-p}{np}}}^{-1}.\] For any $T \ge 0$, 
\begin{equation}\label{eq:80y0f9u0df}  \Exp{\sqn{\nf{\tilde{x}^T}}} \le \frac{2\Psi^0}{\gamma(T+1)},\end{equation} where
 $\tilde{x}^T$ is chosen uniformly at random from the iterates $\{x^0,\ldots,x^T\}$, and
\[\Psi^0 \eqdef f(x^0)-\finf +\frac{\gamma n}{4p}\sqn{g^0-\nf{x^0}} +\frac{\gamma}{4np}\sum_{i=1}^n\sqn{g_i^0-\nfi{x^0}}.\]

\end{restatable}

The above result directly implies the following iteration complexity bound.
\begin{restatable}[Complexity, 
$m\ge n$]{corollary}{silageMgeNArbitraryInit}
\label{cor:silage_mge_n_arbitrary_init}
Under the conditions of Theorem~\ref{theo1}, choose the largest admissible stepsize, i.e., \[\gamma=\rb{L+\delta_2\sqrt{\frac{n-p}{np}}}^{-1}.\] Then the iteration complexity of \algno to reach an $\epsilon$-approximate stationary point is
\begin{equation}\label{eq:n8fy0dy80fd}\cO\rb{\rb{L+\frac{\delta_2}{\sqrt p}}\frac{\Psi^0}{\epsilon}},\end{equation}
and there are $\cO\rb{pm+n}$ component-gradient evaluations per iteration.
In particular, by choosing $p=\frac nm$, the gradient complexity is \[\cO\rb{\rb{nL+\sqrt{nm}\,\delta_2}\frac{\Psi^0}{\epsilon}}.\]
Furthermore, with
exact initialization ($g_i^0 = \nabla f_i(x^0)\ \forall i \in [n]$), we have $\Psi^0 = \Delta_0 \eqdef f(x^0) - \finf$, and the gradient complexity is
\[\cO\rb{nm + \rb{nL+\sqrt{nm}\,\delta_2}\frac{\Delta_0}{\epsilon}}.\]
\end{restatable}

Crucially, because Algorithm~\ref{algsilage1} utilizes the single-coin strategy and updates estimators strictly with their own local data, the estimator error relies exclusively on the within-group variability. Consequently, only $\delta_2$ appears in the bound.

\subsection{Convergence in the $n > m$ regime}
\label{subsec:silage_ngt_m_convergence}

For the $n > m$ regime, in \algno (Algorithm~\ref{algsilage2a}), we refresh exactly one group $i^t$ and select an active subset $\Omega^t$ of size $b_{\mathrm{grp}}-1$ for recursive updates. For the remaining un-sampled groups, the algorithm applies the \emph{shared} recursive drift $d^t$ computed from the active subset. Because this shared correction forces estimators to track directions generated by different groups, its variance relies on both the across-group ($\delta_1$) and within-group ($\delta_2$) heterogeneities.

\begin{restatable}[General Convergence, $n>m$]{theorem}{silageNgtMNC}
\label{theo2} Consider the $n>m$ regime, and let \Cref{ass:lower-bounded} (lower-boundednes), \Cref{ass:f-smooth} ($L$-smoothness), \Cref{ass:delta1} (across-group $\delta_1$ similarity) and \Cref{ass:delta2} (within-group $\delta_2$-similarity) hold. In \algno (Algorithm~\ref{algsilage2a}),
assume the stepsize satisfies \[0<\gamma\le \rb{L+\sqrt{\delta_1^2\frac{(n-b_{\mathrm{grp}})(5n+4)}{n-1} +\delta_2^2\frac{n^2+2n-7b_{\mathrm{grp}}+4b_{\mathrm{grp}}^2}{2b_{\mathrm{grp}}n}}}^{-1}.\] For any $T \ge 0$, 
\begin{equation}\label{eq:b98fdy980fdg}  \Exp{\sqn{\nf{\tilde{x}^T}}} \le \frac{2\Psi^0}{\gamma(T+1)},\end{equation} where
 $\tilde{x}^T$ is chosen uniformly at random from the iterates, 
and
\[\Psi^0 \eqdef f(x^0)-\finf +\frac{\gamma n}{4}\sqn{g^0-\nf{x^0}} +\frac{\gamma}{n}\sum_{i=1}^n\sqn{g_i^0-\nfi{x^0}}.\]
\end{restatable}

The above result directly implies the following iteration complexity bound.
\begin{restatable}[Complexity, $n>m$]{corollary}{silageNgtMArbitraryInit}
\label{cor:silage_ngt_m_arbitrary_init}
Under the conditions of Theorem~\ref{theo2}, choose the largest admissible stepsize, i.e., 
\[\gamma=\rb{L+\sqrt{\delta_1^2\frac{(n-b_{\mathrm{grp}})(5n+4)}{n-1}+\delta_2^2\frac{n^2+2n-7b_{\mathrm{grp}}+4b_{\mathrm{grp}}^2}{2b_{\mathrm{grp}}n}}}^{-1}.\]
Then the iteration complexity\footnote{Here and throughout, we write $\Psi_{b_{\mathrm{grp}}}^0$, with the explicit subscript, whenever we wish to emphasize the dependence of the initial Lyapunov quantity on the batch-size parameter $b_{\mathrm{grp}}$; the plain symbol $\Psi^0$ denotes the same quantity with this dependence left implicit.} of \algno to reach an $\epsilon$-approximate stationary point is
\begin{equation}\label{eq:09709fy09gf}\cO\rb{\rb{L+\sqrt{\delta_1^2\frac{(n-b_{\mathrm{grp}})(5n+4)}{n-1}+\delta_2^2\frac{n^2+2n-7b_{\mathrm{grp}}+4b_{\mathrm{grp}}^2}{2b_{\mathrm{grp}}n}}}\frac{\Psi_{b_{\mathrm{grp}}}^0}{\epsilon}},\end{equation}
and there are $\cO\rb{m+b_{\mathrm{grp}}}$ component-gradient evaluations per iteration.

In particular, by choosing $b_{\mathrm{grp}}=m$, the gradient complexity is \[\cO\rb{\rb{mL+m\sqrt{n-m}\,\delta_1+\sqrt{nm}\,\delta_2}\frac{\Psi_m^0}{\epsilon}}.\]
Furthermore, with
exact initialization ($g_i^0 = \nabla f_i(x^0)\ \forall i \in [n]$), we have $\Psi_m^0 = \Delta_0 \eqdef f(x^0) - \finf$, and the gradient complexity is
\[\cO\rb{nm + \rb{mL+m\sqrt{n-m}\,\delta_1+\sqrt{nm}\,\delta_2}\frac{\Delta_0}{\epsilon}}.\]

\end{restatable}

The drift update in Algorithm~\ref{algsilage2a} is written to touch all $n$ estimators, which is convenient for the analysis but would suggest an $\cO\rb{n}$ per-iteration cost. Appendix~\ref{sec:equiv-n>m} removes this overhead: storing each estimator in the shifted form $g_i^t=h_i^t+q^t$ with a single shared accumulator $q^t$, the equivalent Algorithm~\ref{algsilage2} applies the shared drift via one update of $q^t$ and refreshes only the $\tilde{\Omega}^t$ active entries; Lemma~\ref{lem:equiv-alg2a-alg2} proves the two forms generate identical iterates, so the implementation realizes the $\cO\rb{m+b_{\mathrm{grp}}}$ per-iteration cost.

Note that when $b_{\mathrm{grp}}=m$, the coefficient multiplying $\delta_1$ becomes $m\sqrt{n-m}$. This term vanishes at the exact boundary $n=m$, yielding algebraic equivalence with the bounds from the $m \ge n$ regime. At the opposite extreme, when $m=1$ and $b_{\mathrm{grp}}=1$, the subset $\Omega^t$ is empty, and
\algno reverts to \algn{SILVER}, with our complexity matching the one in \citet{oko24}.

\subsection{Comparison with closely related baselines}
\label{sec:comparison_baselines}

In this section, we compare the theoretical guarantees of \algno against closely related variance-reduced baselines under exact initialization. This is the fairest main-text comparison to methods that begin from the true full gradient, while keeping the core discussion focused. The complementary comparison under arbitrary initialization is deferred to Appendix~\ref{sec:app-comparison-baselines-arbitrary-init}.

\subsubsection{Exact initialization (full-gradient)}\label{sec:comp-exact-init}

We now compare the algorithms under exact initialization, where $g_i^0=\nabla f_i(x^0)$ for all $i\in[n]$. In this setting, $g^0=\nabla f(x^0)$ and the initial Lyapunov quantity perfectly collapses to $\Psi^0=\Delta_0 \eqdef f(x^0)-\finf$, at the cost of $nm$ component-gradient evaluations.

We emphasize that \algname{SILAGE} does not require exact initialization to enjoy its theoretical convergence guarantees: arbitrary initial estimators $g_i^0$ are admissible, and the initialization error enters the complexity only through the Lyapunov quantity $\Psi^0$. The exact initialization adopted here is used solely for a like-for-like comparison with baselines that themselves rely on it.

To avoid overloading batch-size notation in the comparison, we write $b_{\mathrm{fp}}$ for the flattened \algname{PAGE} minibatch size, $b_{\mathrm{sv}}$ for the \algname{SILVER} minibatch size, $b_{\mathrm{zs}}$ for the \algname{ZeroSARAH} minibatch size, $(s_{\mathrm{dp}},b_{\mathrm{dp}})$ for the outer/client and local batch sizes of \algname{D-PAGE}, $(s_{\mathrm{dzs}},b_{\mathrm{dzs}})$ for the client and local batch sizes of \algname{D-ZeroSARAH}, and $b_{\mathrm{grp}}$ for the group-batch parameter of \algname{SILAGE} in the $n>m$ regime; this is the parameter that appears in Algorithm~\ref{algsilage2a}, Theorem~\ref{theo2}, and Corollary~\ref{cor:silage_ngt_m_arbitrary_init}.

The results of this comparison are illustrated in Table~\ref{tab:oracle-complexities-full-init}, which reports the specialization of each oracle complexity under the canonical batch-size choice for the corresponding method. The canonical choices themselves, together with the full tunable form of each bound, are deferred to Appendix~\ref{sec:app-comparison-baselines-general-table} (Table~\ref{tab:oracle-complexities-general}).

\begin{table}[t]
\centering
\small
\setlength{\tabcolsep}{4pt}
\caption{Oracle-complexity comparison under exact full-gradient initialization, listing the same algorithms in the same order as Table~\ref{tab:method-comparison}. Every bound counts component-gradient evaluations $\nabla f_{i,j}$ and includes the one-time initialization cost $nm$. The \algname{SILAGE} bounds are taken from its convergence corollaries, and the baseline bounds from the discussion in Section~\ref{sec:comp-exact-init}.}
\label{tab:oracle-complexities-full-init}
\begin{threeparttable}
\renewcommand{\arraystretch}{1.4}
\begin{tabular}{|c|p{0.65\textwidth}|}
\hline
\textbf{Algorithm}
&
\centering\arraybackslash \textbf{Oracle complexity\tnote{$\star$}}
\\
\hline\hline
\algname{GD}
&
$\cO\rb{nm\,L\,\frac{\Delta_0}{\epsilon}}$
\\
\hline
\begin{tabular}{c}\algname{SILVER}\\\citep{oko24}\end{tabular}
&
$\cO\rb{nm+\rb{L_+\vee \sqrt{nm}\,\delta_{\mathrm{flat}}}\frac{\Delta_0}{\epsilon}}$
\\
\hline
\begin{tabular}{c}\algname{PAGE}\\\citep{li21}\end{tabular}
&
$m\ge n$: $\cO\rb{nm+\rb{nL+\sqrt{nm}\,\delta_{\mathrm{flat}}}\frac{\Delta_0}{\epsilon}}$
\newline
$n>m$: $\cO\rb{nm+\rb{mL+\sqrt{nm}\,\delta_{\mathrm{flat}}}\frac{\Delta_0}{\epsilon}}$
\\
\hline
\begin{tabular}{c}\algname{D-PAGE}\\\citep{PAGE-AB}\end{tabular}
&
$m\ge n$: $\cO\rb{nm+\rb{nL+\sqrt{nm}\,\bar{\delta}_2}\frac{\Delta_0}{\epsilon}}$
\newline
$n>m$: $\cO\big(nm+\big(mL+\sqrt{m(n-m)}\,\delta_1+\sqrt{nm}\,\bar{\delta}_2\big)\frac{\Delta_0}{\epsilon}\big)$
\\
\hline
\begin{tabular}{c}\algname{ZeroSARAH}\\\citep{li2021zerosarah}\end{tabular}
&
$\cO\rb{nm+\sqrt{nm}\,L_{\max}\frac{\Delta_0}{\epsilon}}$
\\
\hline
\begin{tabular}{c}\algname{D-ZeroSARAH}\\\citep{li2021zerosarah}\end{tabular}
&
$\cO\rb{nm+\sqrt{nm}\,L_{\max}\frac{\Delta_0}{\epsilon}}$
\\
\hline
\rowcolor{bgcolor}
\begin{tabular}{c}\textbf{\algname{SILAGE}($m\ge n$)}\\\textbf[OURS]\end{tabular}
&
$\cO\rb{nm+\rb{nL+\sqrt{nm}\,\delta_2}\frac{\Delta_0}{\epsilon}}$
\\
\hline
\rowcolor{bgcolor}
\begin{tabular}{c}\textbf{\algname{SILAGE}($n>m$)}\\ \textbf[OURS] \end{tabular}
&
$\cO\big(nm+\big(mL+m\sqrt{n-m}\,\delta_1+\sqrt{nm}\,\delta_2\big)\frac{\Delta_0}{\epsilon}\big)$
\\
\hline
\end{tabular}

\begin{tablenotes}
\footnotesize
\item [{$\star$}] Specializations under the canonical batch-size choices for each method. The choices themselves, as well as the full tunable form of each bound, are given in Appendix~\ref{sec:app-comparison-baselines-general-table}, Table~\ref{tab:oracle-complexities-general}.
\item [{\color{blue}(1)}] All rows include the exact-initialization cost $nm$. For \algname{SILAGE}, this corresponds to setting $g_i^0=\nabla f_i(x^0)$ for all $i\in[n]$, so that $\Psi^0=\Delta_0\eqdef f(x^0)-\finf$.
\item [{\color{blue}(2)}] Here $\delta_{\mathrm{flat}}$ is the flattened similarity constant over all $nm$ components, while $\delta_1$ and $\delta_2$ are the across-group and within-group similarity constants used by \algname{SILAGE}. The quantity $\bar{\delta}_2$ denotes the within-group similarity appearing in the double-sampling \algname{D-PAGE} bound.
\item [{\color{blue}(3)}] The constants $L_+$ and $L_{\max}$ are baseline-specific worst-case smoothness quantities. In contrast, \algname{SILAGE}'s bounds are stated in terms of the smoothness $L$ of the averaged objective together with the nested similarity constants $\delta_1,\delta_2$.
\end{tablenotes}
\end{threeparttable}
\end{table}

Under Corollary~\ref{cor:silage_mge_n_arbitrary_init} and~\ref{cor:silage_ngt_m_arbitrary_init}, \algno achieves an oracle complexity of \[\cO\rb{nm+\rb{nL+\sqrt{nm}\,\delta_2}\frac{\Delta_0}{\epsilon}}\] for $m\ge n$, and \[\cO\rb{nm+\rb{mL+m\sqrt{n-m}\,\delta_1+\sqrt{nm}\,\delta_2}\frac{\Delta_0}{\epsilon}}\] for $n>m$ with $b_{\mathrm{grp}}=m$.

\paragraph{Comparison with \algname{ZeroSARAH} variants.}
With exact initialization, both \algname{ZeroSARAH} \citep{li2021zerosarah}  and \algname{D-ZeroSARAH} \citep{li2021zerosarah} achieve  \[\cO\rb{nm+\sqrt{nm}\,L_{\max}\frac{\Delta_0}{\epsilon}}\] iteration complexity. Mirroring the arbitrary initialization case, \algno systematically improves upon this by replacing $L_{\max}$ with the structured terms $L$, $\delta_1$, and $\delta_2$. In the $m\ge n$ regime, \algno is superior whenever the within-group variability $\delta_2$ and global smoothness $L$ are small relative to the worst-case component smoothness $L_{\max}$.

\paragraph{Comparison with \algname{SILVER}.}
Run on the flattened problem, \algname{SILVER} \citep{oko24} with exact initialization has a tuned complexity of \[\cO\rb{nm+\rb{L_+\vee \sqrt{nm}\,\delta_{\mathrm{flat}}}\frac{\Delta_0}{\epsilon}}.\] Crucially, this flattened implementation stores one estimator per sample, demanding an impractical $\cO\rb{nm}$ memory. \algno reduces this to $\cO\rb{n}$ memory while simultaneously improving the gradient complexity whenever the nested constants are substantially smaller than the flattened quantities (e.g., $\sqrt{n/m}L+\delta_2 \ll \delta_{\mathrm{flat}}$ when $m \ge n$). Consequently, \algno can be viewed as the nested analogue of \algname{SILVER}: it recovers \algname{SILVER} when $m=1$, but for genuine double finite sums, it structurally separates the nested geometry ($\delta_1, \delta_2$) instead of collapsing it into $\delta_{\mathrm{flat}}$.

\paragraph{\normalcolor Comparison with \algname{PAGE}\normalcolor variants.}
{\algname{PAGE}}-type baselines warrant a separate analysis because, unlike single-loop methods, they structurally require periodic global full-gradient refreshes \emph{even after} exact initialization. Let \[\bar{\delta}_2^2\eqdef \frac{1}{n}\sum_{i=1}^n\delta_{2,i}^2\] denote the averaged within-group similarity. Because a supremum of an average is at most the average of the suprema, these constants always satisfy $\delta_2 \le \bar{\delta}_2$, with a gap that can be large when the worst-case within-group geometry occurs at different points across groups.

For flattened \algname{PAGE} (with Nice minibatch size $b_{\mathrm{fp}}$ and refresh probability $p(b_{\mathrm{fp}})=b_{\mathrm{fp}}/(nm+b_{\mathrm{fp}})$), the complexity is \[\cO\rb{nm + \frac{nm\,b_{\mathrm{fp}}}{nm+b_{\mathrm{fp}}}\rb{L+\frac{\sqrt{nm(nm-b_{\mathrm{fp}})}}{b_{\mathrm{fp}}\sqrt{nm-1}}\,\delta_{\mathrm{flat}}}\frac{\Delta_0}{\epsilon}}.\] Taking $b_{\mathrm{fp}}=n$ (when $m\ge n$) and $b_{\mathrm{fp}}=m$ (when $n>m$) yields leading factors of $nL+\sqrt{nm}\,\delta_{\mathrm{flat}}$ and $mL+\sqrt{nm}\,\delta_{\mathrm{flat}}$, respectively. Like flattened \algname{SILVER}, this inherently merges nested variability into a single constant.

The structured double-Nice \algname{D-PAGE} baseline is considerably stronger. With outer/client batch size $s_{\mathrm{dp}}$, local batch size $b_{\mathrm{dp}}$, and refresh probability $p=s_{\mathrm{dp}}b_{\mathrm{dp}}/(nm+s_{\mathrm{dp}}b_{\mathrm{dp}})$, its theoretical cost is \[\cO\rb{nm + \frac{nm\,s_{\mathrm{dp}}b_{\mathrm{dp}}}{nm+s_{\mathrm{dp}}b_{\mathrm{dp}}}\rb{L+\sqrt{\frac{nm(m-b_{\mathrm{dp}})}{s_{\mathrm{dp}}^2b_{\mathrm{dp}}^2(m-1)}\,\bar{\delta}_2^2+\frac{nm(n-s_{\mathrm{dp}})}{s_{\mathrm{dp}}^2b_{\mathrm{dp}}(n-1)}\,\delta_1^2}}\frac{\Delta_0}{\epsilon}}.\]
\begin{itemize}
    \item \textbf{Regime $m\ge n$:} The useful full-outer-sampling choice $s_{\mathrm{dp}}=n$ and $b_{\mathrm{dp}}=1$ gives a leading factor of $nL+\sqrt{nm}\,\bar{\delta}_2$. This matches the leading \algno dependence up to the distinction between $\delta_2$ and $\bar{\delta}_2$. However, \algname{D-PAGE} pays for this with periodic global full-gradient refreshes. Furthermore, since $\delta_2 \le \bar{\delta}_2$ (above), with $\bar{\delta}_2$ averaging client-wise worst-case constants and $\delta_2$ a global within-group average, \algno is weakly preferable in this regime, and the gap can be substantial in heterogeneous settings.
    \item \textbf{Regime $n>m$:} The choice $s_{\mathrm{dp}}=m$ and $b_{\mathrm{dp}}=1$ gives a leading factor of $mL+\sqrt{m(n-m)}\,\delta_1+\sqrt{nm}\,\bar{\delta}_2$, against $mL+m\sqrt{n-m}\,\delta_1+\sqrt{nm}\,\delta_2$ for the $b_{\mathrm{grp}}=m$ specialization of \algno. \algname{D-PAGE} is a strong full-refresh baseline here, and the two are \emph{not} ordered coefficient-wise: \algname{D-PAGE} has the smaller across-group coefficient ($\sqrt{m(n-m)}$ versus \algno's $m\sqrt{n-m}$, a factor of $\sqrt{m}$), whereas \algno has the smaller within-group constant, since $\delta_2 \le \bar{\delta}_2$. Thus \algno improves on \algname{D-PAGE} only when its within-group advantage outweighs this across-group penalty, i.e.\ whenever $m\sqrt{n-m}\,\delta_1+\sqrt{nm}\,\delta_2 \ll \sqrt{m(n-m)}\,\delta_1+\sqrt{nm}\,\bar{\delta}_2$, and it does so without the periodic global full-gradient refreshes \algname{D-PAGE} requires.
\end{itemize}

Ultimately, the cleanest distinction between \algno and \algname{PAGE}-type methods is algorithmic rather than a uniform oracle-complexity improvement. \algname{D-PAGE} can attain competitive leading factors and keeps a flat $\cO\rb{1}$ memory footprint, but, like every \algname{PAGE}-type estimator, it requires periodic global full-gradient refreshes over all $nm$ components. \algno is instead single-loop: it removes these refreshes entirely, at the cost of $\cO\rb{n}$ group-level memory (one estimator per group), while still exploiting the nested similarity structure. Its advantage over \algname{D-PAGE} is therefore the elimination of periodic global refreshes---most valuable when such refreshes dominate cost---rather than a claim of improving every oracle-complexity term.

The constants $L$, $\delta_1$, and $\delta_2$ are objects of the \emph{analysis}: they determine the largest provably admissible stepsize and the resulting convergence rate, but they do not appear in the algorithm itself. \algname{SILAGE} samples groups, flips the global coin, and performs anchor resets and recursive updates without ever referencing them, so a single run is identical regardless of their values. Calibrating a stepsize to problem constants is standard for variance-reduced methods---$L$ for gradient descent, $L_{\max}$ for \algname{ZeroSARAH}---and the distinguishing feature of \algname{SILAGE} is not independence from such constants but dependence on the \emph{structured} quantities $L,\delta_1,\delta_2$ rather than the pessimistic $L_{\max}$ or the flattened $\delta_{\mathrm{flat}}$; this is precisely what yields the larger admissible stepsize and the improved rates established above. In practice, the constants may be estimated whenever curvature access is available---as for the logistic-regression objectives in our experiments, where Hessian--vector products make $L,\delta_1,\delta_2$ cheap to probe (Appendix~\ref{sec:experiments_similarity_logreg})---replaced by any valid upper bound, or sidestepped by a standard stepsize search, all without altering the algorithm. Our experiments adopt the first route: probe-set estimates are used in the theoretically admissible stepsize while batch sizes are tuned per method, so that no exact prior knowledge of these constants is assumed.

\begin{figure}[t]
    \centering
    \includegraphics[width=\textwidth]{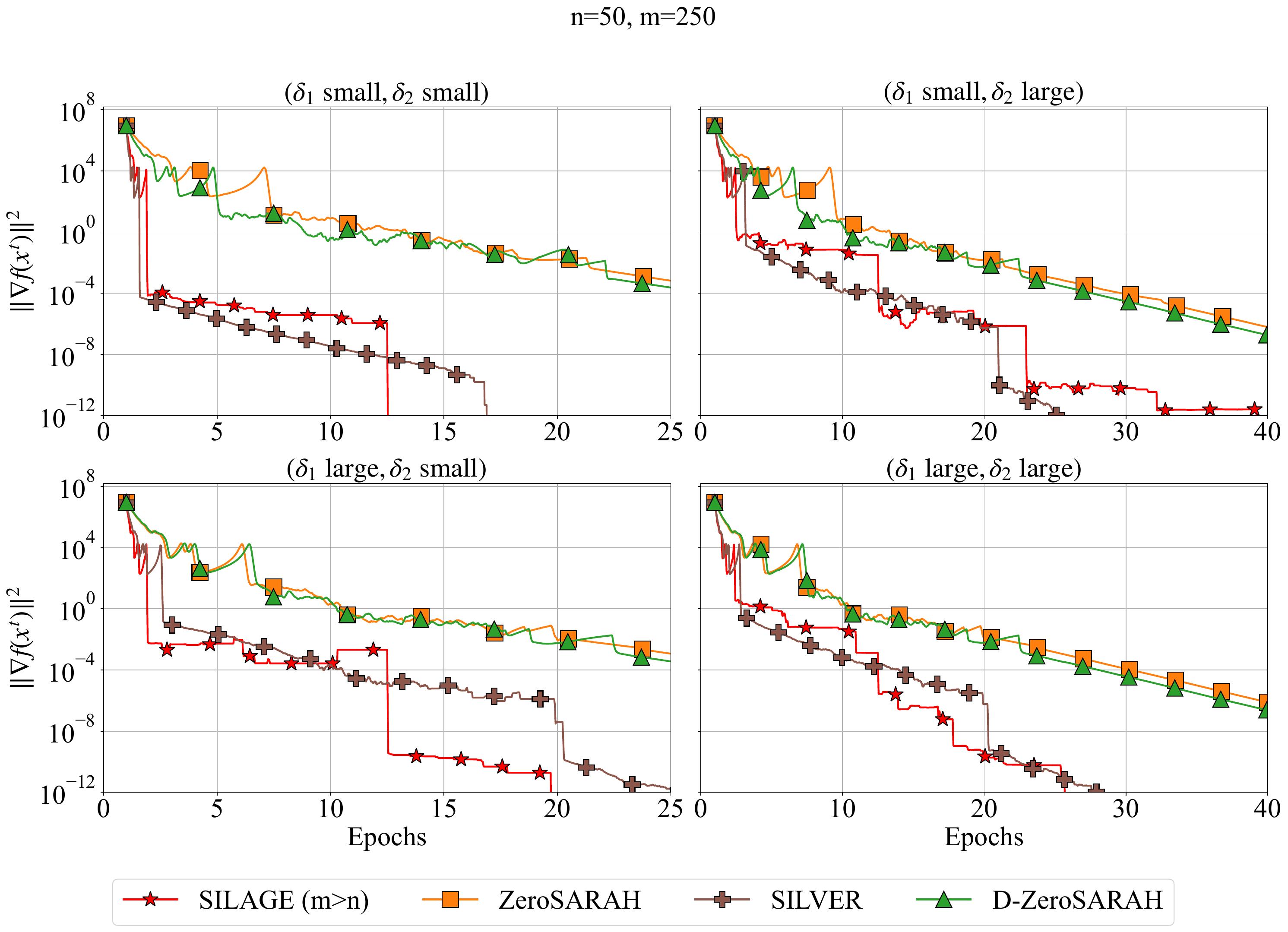}
    \caption{Optimization trajectories on the controlled nonconvex-regularized logistic regression task, $m\ge n$ regime ($n=50$, $m=250$). Each panel fixes one of the four similarity regimes; the corresponding measured constants $(\delta_1,\delta_2,L,L_{\max})$ are reported in Table~\ref{tab:synthetic-regimes}.}
    \label{fig:trajectory-mge-n}
\end{figure}

\section{Experiments}
\label{sec:experiments}

In this section, we empirically validate whether the theoretically predicted dependence on nested similarity parameters ($\delta_1, \delta_2$) translates to practical optimization performance. We consider a binary logistic regression task of the form \eqref{eq:prob} with a smooth, nonconvex regularizer applied identically across all components.  In particular, for feature-label pairs $(a_{i,j},y_{i,j}) \in \R^d \times \{-1,1\}$, the individual sample loss is \[f_{i,j}(x) \eqdef \log(1+\exp(-y_{i,j}a_{i,j}^{\top}x)) + \lambda\sum_{\ell=1}^d \frac{x_\ell^2 }{ (1+x_\ell^2)}.\] We set $\lambda=200$ across all runs. This heavy regularization ensures the nonconvex penalty dominates the global smoothness scale.
\begin{figure}[t]
    \centering
    \includegraphics[width=\textwidth]{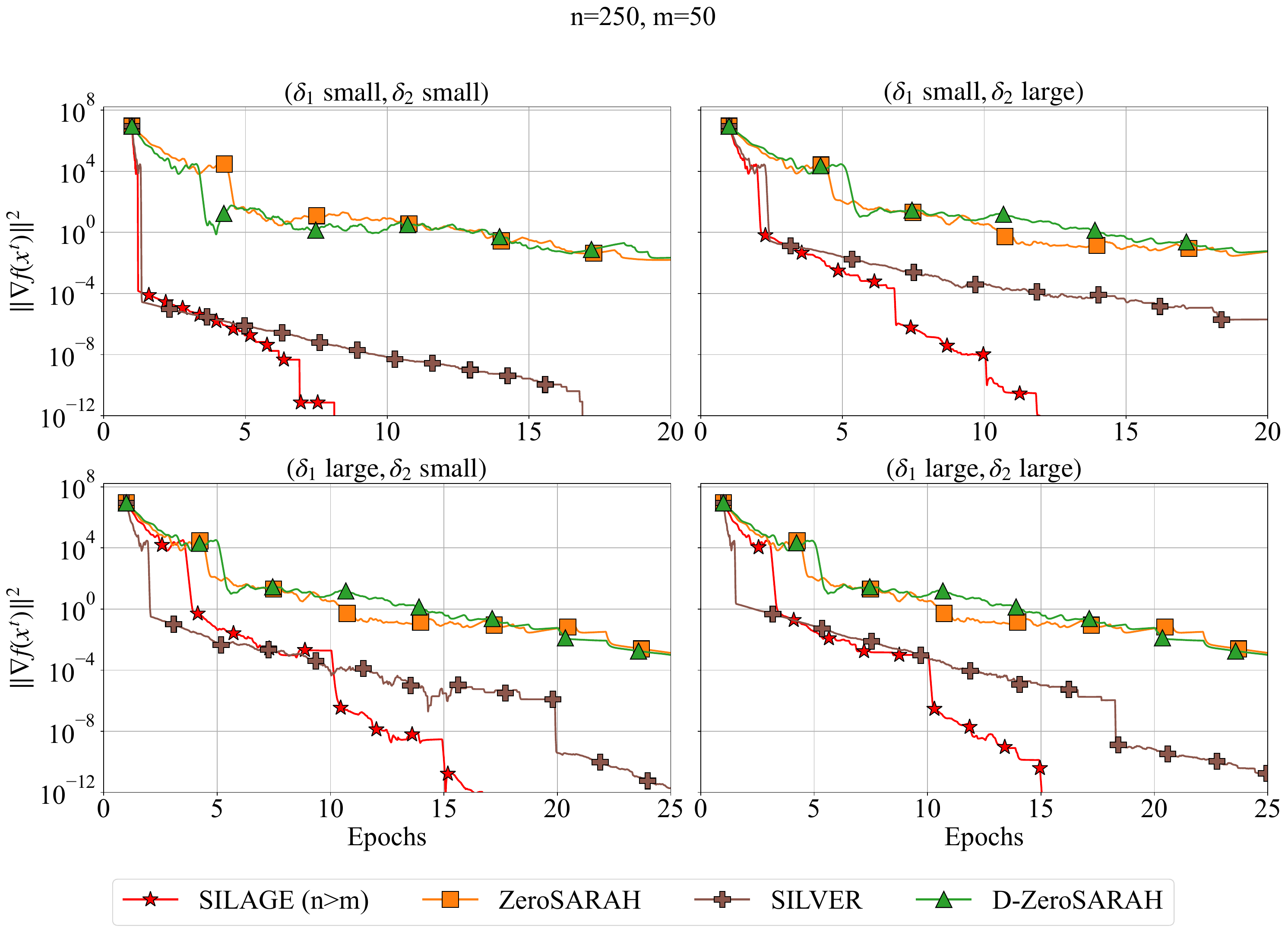}
    \caption{Optimization trajectories on the controlled nonconvex-regularized logistic regression task, $n>m$ regime ($n=250$, $m=50$). Each panel fixes one of the four similarity regimes; the corresponding measured constants $(\delta_1,\delta_2,L,L_{\max})$ are reported in Table~\ref{tab:synthetic-regimes}.}
    \label{fig:trajectory-ngt-m}
\end{figure}
We construct eight synthetic grouped datasets to test two distinct shape regimes: $m \ge n$ with $(n,m)=(50,250)$, and $n > m$ with $(n,m)=(250,50)$. Within each regime, we synthesize four controlled datasets to independently realize small or large across-group similarity ($\delta_1$) alongside small or large within-group similarity ($\delta_2$). The source code, together with instructions for reproducing Figures~\ref{fig:trajectory-mge-n} and~\ref{fig:trajectory-ngt-m}, is available at \href{https://github.com/IgorSokoloff/silage-experiments}{\texttt{github.com/IgorSokoloff/silage-experiments}}. Full details regarding data generation, Hessian-vector product estimation of $L, \delta_1, \delta_2$, and implementation specifics are deferred to Appendix~\ref{sec:experiments_similarity_logreg}. Table~\ref{tab:synthetic-regimes} summarizes the empirically measured constants for all eight datasets, confirming that our construction realizes the intended relative magnitudes of $\delta_1$ and $\delta_2$; it also gives the precise quantitative meaning of the qualitative ``small/large'' labels used in the panels of Figures~\ref{fig:trajectory-mge-n} and~\ref{fig:trajectory-ngt-m}.

\begin{table}[t]
    \centering
    \setlength{\tabcolsep}{5pt}
    \begin{tabular}{llrrrr}
        \toprule
        \textbf{Setting} & \textbf{Regime} & $L$ & $L_{\max}$ & $\delta_1$ & $\delta_2$ \\
        \midrule
        $m\ge n$ & $(\delta_1 \text{ small}, \delta_2 \text{ small})$ & 400.03 & 416.17 & 0.04 & 0.66 \\
        $m\ge n$ & $(\delta_1 \text{ small}, \delta_2 \text{ large})$ & 409.24 & 557.39 & 7.04 & 116.89 \\
        $m\ge n$ & $(\delta_1 \text{ large}, \delta_2 \text{ small})$ & 408.86 & 556.74 & 117.11 & 2.83 \\
        $m\ge n$ & $(\delta_1 \text{ large}, \delta_2 \text{ large})$ & 403.66 & 626.29 & 130.62 & 102.82 \\
        \midrule
        $n>m$ & $(\delta_1 \text{ small}, \delta_2 \text{ small})$ & 400.03 & 416.09 & 0.05 & 0.32 \\
        $n>m$ & $(\delta_1 \text{ small}, \delta_2 \text{ large})$ & 425.50 & 626.32 & 21.30 & 165.92 \\
        $n>m$ & $(\delta_1 \text{ large}, \delta_2 \text{ small})$ & 424.88 & 625.28 & 168.75 & 1.20 \\
        $n>m$ & $(\delta_1 \text{ large}, \delta_2 \text{ large})$ & 404.24 & 626.28 & 129.89 & 102.67 \\
        \bottomrule
    \end{tabular}
    \caption{
    Empirical geometric diagnostics measured on the synthetic datasets. The values confirm that our construction successfully controls the relative magnitudes of across-group similarity ($\delta_1$) and within-group similarity ($\delta_2$), while maintaining a consistently large global smoothness scale ($L \approx 400$).
    }
    \label{tab:synthetic-regimes}
\end{table}

We compare \algname{SILAGE} against three baselines that successfully avoid periodic global full-gradient refreshes: flattened \algname{ZeroSARAH} \citep{li2021zerosarah}, flattened \algname{SILVER} \citep{oko24}, and the nested distributed variant \algname{D-ZeroSARAH} \citep{li2021zerosarah}. We intentionally omit \algname{PAGE} and \algname{D-PAGE}, whose estimators fundamentally rely on periodic global full-gradient refreshes over all $nm$ components \citep{li21,PAGE-AB}; the empirical comparison thus isolates the no-refresh design class, while the oracle-cost of those refreshes is quantified analytically in Section~\ref{sec:comp-exact-init} (Table~\ref{tab:oracle-complexities-full-init}). We deploy Algorithm~\ref{algsilage1} for the $m\ge n$ regime and Algorithm~\ref{algsilage2a} for the $n>m$ regime. All methods use the largest theoretically admissible constant stepsize from their respective analyses, evaluated with empirical estimates of the relevant smoothness and similarity constants; batch-size parameters are tuned per method, with the full protocol (probe-set estimation, tuning grids, and fallback rule) described in Appendix~\ref{sec:hyperparameter_tuning}. For \algname{SILAGE} in the $n>m$ regime, the headline specialization $b_{\mathrm{grp}}=m$ gives a clean closed-form rate but is not necessarily oracle-optimal: Theorem~\ref{theo2} holds for every $b_{\mathrm{grp}}\in[n]$, the oracle-optimal choice is problem-dependent (Appendix~\ref{sec:app-opt-bgrp}), and we accordingly tune $b_{\mathrm{grp}}$ over a grid, so the selected value can be smaller than $m$. We report the exact squared gradient norm, $\sqn{\nabla f(x^t)}$, recomputed every two iterations, until reaching $\sqn{\nabla f(x^t)} \le 10^{-12}$ or exhausting the epoch budget.

Across these controlled synthetic runs, Figures~\ref{fig:trajectory-mge-n} and~\ref{fig:trajectory-ngt-m} show \algname{SILAGE} to be the fastest, or among the fastest, of the tested methods, with the empirical trajectories closely tracking our theoretical predictions. In the $m \ge n$ regime (Figure~\ref{fig:trajectory-mge-n}), our theory indicates the leading similarity term depends entirely on $\delta_2$, ignoring $\delta_1$. The plots are consistent with this: in both small-$\delta_2$ panels, \algname{SILAGE} converges rapidly even when the across-group heterogeneity $\delta_1$ is large. While a large $\delta_2$ naturally increases the problem difficulty, \algname{SILAGE} still outperforms both \algname{ZeroSARAH} variants and matches or exceeds \algname{SILVER} in these runs. In the $n > m$ regime (Figure~\ref{fig:trajectory-ngt-m}), where both $\delta_1$ and $\delta_2$ factor into the bound, \algname{SILAGE} outperforms the tested baselines across all four similarity profiles shown. Taken together, these controlled experiments support our theoretical predictions, indicating a practical benefit from explicitly separating and exploiting nested similarity rather than relying on flattened structures or pessimistic worst-case component smoothness.

We emphasize that these experiments are a controlled synthetic validation of the predicted similarity-dependence rather than a systems study: the $\cO\rb{n}$ memory and scalability advantages are established by construction and summarized in Table~\ref{tab:method-comparison}, not measured here, and a real grouped-data or wall-clock evaluation is left to future work.

\section{Additional Results}\label{sec:additional-results}

Several complementary results are deferred to the appendix. A full description of the experimental setup---covering data-generation protocols, the Hessian-vector-product estimation of $L$, $\delta_1$, and $\delta_2$, hyperparameter tuning, and implementation details---is provided in Appendix~\ref{sec:experiments_similarity_logreg}. The Polyak--\L{}ojasiewicz convergence analyses of \algname{SILAGE} appear in Appendix~\ref{sec:app-pl-mge-n} for the $m\ge n$ regime and in Appendix~\ref{sec:app-pl-ngt-m} for the $n>m$ regime. Finally, the corresponding complexity comparison against the same baselines under arbitrary (zero-full-gradient) initialization is given in Appendix~\ref{sec:app-comparison-baselines-arbitrary-init}.

\section*{Acknowledgements}
This work was supported by funding from King Abdullah University of Science and Technology (KAUST): i) KAUST Baseline Research Scheme,
ii) Center of Excellence for Generative AI (award no.\ 5940),
iii) Competitive Research Grant (CRG) Program (award no.\ 6460).


\bibliography{bib}
\bibliographystyle{abbrvnat}


\appendix

\newpage

\tableofcontents

\newpage

\section{Additional Comparison with Closely Related Baselines}
\label{sec:app-comparison-baselines-arbitrary}

\subsection{Exact initialization (full-gradient)}
\label{sec:app-comparison-baselines-general-table}

For readability, the main-text Table~\ref{tab:oracle-complexities-full-init} reports only the specialization of each method's oracle complexity under canonical batch-size choices, and omits the choices themselves. For completeness, Table~\ref{tab:oracle-complexities-general} below reproduces the same comparison but now also includes the \emph{general tunable form} of each bound, as a function of the method-specific batch-size parameters introduced in Section~\ref{sec:comp-exact-init}: $b_{\mathrm{sv}}$ for \algname{SILVER}, $b_{\mathrm{fp}}$ for flattened \algname{PAGE}, $(s_{\mathrm{dp}},b_{\mathrm{dp}})$ for \algname{D-PAGE}, $b_{\mathrm{zs}}$ and $\lambda$ for \algname{ZeroSARAH}, $(s_{\mathrm{dzs}},b_{\mathrm{dzs}})$ for \algname{D-ZeroSARAH}, and $b_{\mathrm{grp}}$ for \algname{SILAGE} in the $n>m$ regime. The right-hand column of Table~\ref{tab:oracle-complexities-general} additionally records the canonical batch-size choice producing each specialization; these annotations are the only difference relative to the main-text Table~\ref{tab:oracle-complexities-full-init}.

\begin{table}[t]
\centering
\scriptsize
\setlength{\tabcolsep}{4pt}
\caption{Full oracle-complexity comparison under exact full-gradient initialization, listing the same algorithms in the same order as Table~\ref{tab:method-comparison}. Every bound counts component-gradient evaluations $\nabla f_{i,j}$ and includes the one-time initialization cost $nm$. The first complexity column reports the general (tunable) form; the second reports a useful specialization. The \algname{SILAGE} bounds are taken from its convergence corollaries, and the baseline bounds from the discussion in Section~\ref{sec:comp-exact-init}.}
\label{tab:oracle-complexities-general}
\begin{threeparttable}
\renewcommand{\arraystretch}{1.4}
\begin{tabular}{|c|p{0.46\textwidth}|p{0.30\textwidth}|}
\hline
\textbf{Algorithm}
&
\centering\arraybackslash \textbf{General (tunable) form}
&
\centering\arraybackslash \textbf{Specialization}
\\
\hline\hline
\algname{GD}
&
$\cO\rb{nm\,L\,\frac{\Delta_0}{\epsilon}}$
&
No tuning.
\\
\hline
\begin{tabular}{c}\algname{SILVER}\\\citep{oko24}\end{tabular}
&
$\cO\rb{
nm+
\rb{
\sqrt{nm}\,\delta_{\mathrm{flat}}
\vee
L_+ b_{\mathrm{sv}}
}
\frac{\Delta_0}{\epsilon}
}$
&
$\cO\rb{
nm+
\rb{
L_+
\vee
\sqrt{nm}\,\delta_{\mathrm{flat}}
}
\frac{\Delta_0}{\epsilon}
}$
for $b_{\mathrm{sv}}=1$.
\\
\hline
\begin{tabular}{c}\algname{PAGE}\\\citep{li21}\end{tabular}
&
$\cO\rb{
nm+
\frac{nm\,b_{\mathrm{fp}}}{nm+b_{\mathrm{fp}}}
\rb{
L+
\frac{\sqrt{nm(nm-b_{\mathrm{fp}})}}{b_{\mathrm{fp}}\sqrt{nm-1}}\,
\delta_{\mathrm{flat}}
}
\frac{\Delta_0}{\epsilon}
}$
&
$m\ge n$ ($b_{\mathrm{fp}}=n$):
\newline
$\cO\rb{nm+\rb{nL+\sqrt{nm}\,\delta_{\mathrm{flat}}}\frac{\Delta_0}{\epsilon}}$.
\newline
$n>m$ ($b_{\mathrm{fp}}=m$):
\newline
$\cO\rb{nm+\rb{mL+\sqrt{nm}\,\delta_{\mathrm{flat}}}\frac{\Delta_0}{\epsilon}}$.
\\
\hline
\begin{tabular}{c}\algname{D-PAGE}\\\citep{PAGE-AB}\end{tabular}
&
$\cO\bigg(
nm+
\frac{nm\,s_{\mathrm{dp}}b_{\mathrm{dp}}}{nm+s_{\mathrm{dp}}b_{\mathrm{dp}}}
\bigg(
L+$
\newline
$\quad
\sqrt{
\frac{nm(m-b_{\mathrm{dp}})}{s_{\mathrm{dp}}^2b_{\mathrm{dp}}^2(m-1)}
\,\bar{\delta}_2^2
+
\frac{nm(n-s_{\mathrm{dp}})}{s_{\mathrm{dp}}^2b_{\mathrm{dp}}(n-1)}
\,\delta_1^2
}
\;\bigg)
\frac{\Delta_0}{\epsilon}
\bigg)$
&
$m\ge n$ ($s_{\mathrm{dp}}=n$, $b_{\mathrm{dp}}=1$):
\newline
$\cO\rb{nm+\rb{nL+\sqrt{nm}\,\bar{\delta}_2}\frac{\Delta_0}{\epsilon}}$.
\newline
$n>m$ ($s_{\mathrm{dp}}=m$, $b_{\mathrm{dp}}=1$):
\newline
$\cO\big(nm+\big(mL+\sqrt{m(n-m)}\,\delta_1$
\newline
$\quad+\sqrt{nm}\,\bar{\delta}_2\big)\frac{\Delta_0}{\epsilon}\big)$.
\\
\hline
\begin{tabular}{c}\algname{ZeroSARAH}\\\citep{li2021zerosarah}\end{tabular}
&
$\cO\rb{
nm+
b_{\mathrm{zs}}L_{\max}
\rb{
1+
\sqrt{
\frac{2}{\lambda b_{\mathrm{zs}}}
+
\frac{8\lambda (nm)^2}{b_{\mathrm{zs}}^3}
}
}
\frac{\Delta_0}{\epsilon}
}$
&
$\cO\rb{
nm+
\sqrt{nm}\,L_{\max}
\frac{\Delta_0}{\epsilon}
}$
for $b_{\mathrm{zs}}=\sqrt{nm}$ and
$\lambda=b_{\mathrm{zs}}/(2nm)$.
\\
\hline
\begin{tabular}{c}\algname{D-ZeroSARAH}\\\citep{li2021zerosarah}\end{tabular}
&
$\cO\bigg(
nm+
s_{\mathrm{dzs}}\,b_{\mathrm{dzs}}\,L_{\max}
\bigg(
1+$
\newline
$\quad
\sqrt{
\frac{2}{\lambda s_{\mathrm{dzs}} b_{\mathrm{dzs}}}
+
\frac{8\lambda n^2m^2}{s_{\mathrm{dzs}}^3 b_{\mathrm{dzs}}^3}
}
\;\bigg)
\frac{\Delta_0}{\epsilon}
\bigg)$
&
$\cO\rb{
nm+
\sqrt{nm}\,L_{\max}
\frac{\Delta_0}{\epsilon}
}$
for $s_{\mathrm{dzs}}=\sqrt n$, $b_{\mathrm{dzs}}=\sqrt m$, and
$\lambda=s_{\mathrm{dzs}}b_{\mathrm{dzs}}/(2nm)$.
\\
\hline
\rowcolor{bgcolor}
\begin{tabular}{c}\textbf{\algname{SILAGE}($m\ge n$)}\\\textbf[OURS]\end{tabular}
&
$\cO\rb{
nm+
\rb{
nL+\sqrt{nm}\,\delta_2
}
\frac{\Delta_0}{\epsilon}
}$
&
No batch-size/client-size tuning in this regime.
\\
\hline
\rowcolor{bgcolor}
\begin{tabular}{c}\textbf{\algname{SILAGE}($n>m$)}\\ \textbf[OURS] \end{tabular}
&
$\cO\bigg(
nm+
\rb{m+b_{\mathrm{grp}}}
\bigg(
L+$
\newline
$\quad
\sqrt{
\delta_1^2\frac{(n-b_{\mathrm{grp}})(5n+4)}{n-1}
+
\delta_2^2\frac{n^2+2n-7b_{\mathrm{grp}}+4b_{\mathrm{grp}}^2}{2b_{\mathrm{grp}}n}
}
\;\bigg)
\frac{\Delta_0}{\epsilon}
\bigg)$
&
$\cO\big(nm+\big(mL+m\sqrt{n-m}\,\delta_1$
\newline
$\quad+\sqrt{nm}\,\delta_2\big)\frac{\Delta_0}{\epsilon}\big)$
\newline
for $b_{\mathrm{grp}}=m$.
\\
\hline
\end{tabular}

\begin{tablenotes}
\footnotesize
\item [{\color{blue}(1)}] All rows include the exact-initialization cost $nm$. For \algname{SILAGE}, this corresponds to setting $g_i^0=\nabla f_i(x^0)$ for all $i\in[n]$, so that $\Psi^0=\Delta_0\eqdef f(x^0)-\finf$.
\item [{\color{blue}(2)}] Here $\delta_{\mathrm{flat}}$ is the flattened similarity constant over all $nm$ components, while $\delta_1$ and $\delta_2$ are the across-group and within-group similarity constants used by \algname{SILAGE}. The quantity $\bar{\delta}_2$ denotes the within-group similarity appearing in the double-sampling \algname{D-PAGE} bound.
\item [{\color{blue}(3)}] The tunable parameters are: the minibatch size $b_{\mathrm{sv}}$ in \algname{SILVER}; the flattened \algname{PAGE} minibatch size $b_{\mathrm{fp}}$; the outer/client batch size $s_{\mathrm{dp}}$ and local batch size $b_{\mathrm{dp}}$ in \algname{D-PAGE}; the minibatch size $b_{\mathrm{zs}}$ and weight $\lambda$ in \algname{ZeroSARAH}; the client batch size $s_{\mathrm{dzs}}$ and local batch size $b_{\mathrm{dzs}}$ in \algname{D-ZeroSARAH}; and the group batch size $b_{\mathrm{grp}}$ in \algname{SILAGE} ($n>m$).
\item [{\color{blue}(4)}] The constants $L_+$ and $L_{\max}$ are baseline-specific worst-case smoothness quantities. In contrast, \algname{SILAGE}'s bounds are stated in terms of the smoothness $L$ of the averaged objective together with the nested similarity constants $\delta_1,\delta_2$.
\end{tablenotes}
\end{threeparttable}
\end{table}

\subsection{Arbitrary initialization (zero full-gradient)}
\label{sec:app-comparison-baselines-arbitrary-init}

In this setting, we assume that the estimators $g_i$ are initialized in an arbitrary manner (e.g., $g_i^0 = 0$), bypassing the need to compute any full gradients at the start of the algorithm. We compare \algno against methods that similarly admit arbitrary initialization.

As in the main-text comparison, $b_{\mathrm{grp}}$ denotes the group-batch parameter of \algname{SILAGE} in the $n>m$ regime, as defined in Algorithm~\ref{algsilage2a} and used in Theorem~\ref{theo2} and Corollary~\ref{cor:silage_ngt_m_arbitrary_init}.

For \algno, Corollary~\ref{cor:silage_mge_n_arbitrary_init} establishes an oracle complexity of $\cO\rb{\rb{nL+\sqrt{nm}\,\delta_2}\frac{\Psi^0}{\epsilon}}$ in the $m\ge n$ regime. In the $n>m$ regime, the general bound yields $\cO\rb{\rb{m+b_{\mathrm{grp}}}\rb{L+\delta_1\sqrt{n-b_{\mathrm{grp}}}+\delta_2\sqrt{\tfrac{n}{b_{\mathrm{grp}}}}}\frac{\Psi_{b_{\mathrm{grp}}}^0}{\epsilon}}$. Specifically, the choice $b_{\mathrm{grp}}=m$ gives $\cO\rb{\rb{mL+m\sqrt{n-m}\,\delta_1+\sqrt{nm}\,\delta_2}\frac{\Psi_m^0}{\epsilon}}$, while the endpoint $b_{\mathrm{grp}}=n$ removes the across-group term entirely, yielding $\cO\rb{\rb{nL+n\delta_2}\frac{\Psi_n^0}{\epsilon}}$.

\paragraph{Comparison with \algname{ZeroSARAH} variants.}
Flattened \algname{ZeroSARAH} and \algname{D-ZeroSARAH} achieve complexities of $\cO\rb{\sqrt{nm}\rb{L_{\max}\Delta_0+G_0}/\epsilon}$ and $\cO\rb{\sqrt{nm}\rb{L_{\max}\Delta_0+G'_0}/\epsilon}$, respectively. Here, $L_{\max}\eqdef \max_{i,j}L_{i,j}$ is the worst-case component smoothness, $G_0 \eqdef \frac{1}{n} \sum_{i=1}^n\sqn{\nabla f_i(x^0)}$ is the initial estimator error for \algname{ZeroSARAH}, and $G_0^{\prime} \eqdef \frac{1}{nm} \sum_{i=1}^n \sum_{j=1}^m\sqn{\nabla f_{i,j}(x^0)}$ is the corresponding error for \algname{D-ZeroSARAH}.

Thus, up to the differing initial Lyapunov quantities, \algno completely replaces the pessimistic worst-case component smoothness $L_{\max}$ with the structured global parameters $L, \delta_1$, and $\delta_2$. 
\begin{itemize}[leftmargin=*]
    \item For $m\ge n$, \algno is favorable whenever $nL+\sqrt{nm}\,\delta_2 \ll \sqrt{nm}\,L_{\max}$.
    \item For $n>m$, the choice $b_{\mathrm{grp}}=m$ is favorable whenever $mL+m\sqrt{n-m}\,\delta_1+\sqrt{nm}\,\delta_2 \ll \sqrt{nm}\,L_{\max}$. If $\delta_1$ is large but the global $L$ is much smaller than $L_{\max}$, the endpoint choice $b_{\mathrm{grp}}=n$ can be preferable as it strictly removes the $\delta_1$ dependence.
\end{itemize}

\paragraph{Comparison with \algname{SILVER}.}
The oracle complexity of flattened \algname{SILVER} is $\cO\rb{\rb{\Delta_0\rb{\delta_{\mathrm{flat}}\sqrt{nm}\vee Lb_{\mathrm{sv}}}+\frac{nm}{b_{\mathrm{sv}}}\sigma_{\mathrm{flat}}^2} / \epsilon}$, where $b_{\mathrm{sv}}\in[nm]$ is the flattened minibatch size and $\sigma_{\mathrm{flat}}^2 \eqdef \frac{1}{nm}\sum_{r=1}^{nm}\sqn{\nabla \tilde f_r(x^0)-\nabla f(x^0)}$ measures the initial variance.

Assuming a low-initial-variance regime where $\Psi^0$ and $\Delta_0$ are comparable, \algno improves upon optimized, zero-initialized flattened \algname{SILVER} under the following conditions:
\begin{itemize}[leftmargin=*]
    \item \textbf{Regime $m \ge n$:} \algno is strictly better when $nL+\sqrt{nm}\,\delta_2 \ll \sqrt{nm}\,\delta_{\mathrm{flat}}$, or equivalently $\sqrt{n/m}L+\delta_2 \ll \delta_{\mathrm{flat}}$. This occurs when flattening is inherently costly---for instance, when the flat heterogeneity $\delta_{\mathrm{flat}}$ is primarily driven by across-group variation ($\delta_1$), which \algno entirely avoids paying for in this regime. If instead $\delta_{\mathrm{flat}} \approx \delta_2$ and $\sigma_{\mathrm{flat}}^2$ is small, the theoretical rate matches \algname{SILVER}, but \algno retains its massive $\cO\rb{n}$ rather than $\cO\rb{nm}$ memory advantage.
    \item \textbf{Regime $n > m$:} Comparing the leading factors for a chosen group batch size $b_{\mathrm{grp}}\in[m,n]$, \algno is favorable when $\rb{m+b_{\mathrm{grp}}}\rb{L+\sqrt{n-b_{\mathrm{grp}}}\,\delta_1+\sqrt{\tfrac{n}{b_{\mathrm{grp}}}}\,\delta_2} \ll \sqrt{nm}\,\delta_{\mathrm{flat}}$. For our targeted choice $b_{\mathrm{grp}}=m$, this simplifies to $\sqrt{m/n}L + \sqrt{m(n-m)/n}\,\delta_1 + \delta_2 \ll \delta_{\mathrm{flat}}$. Thus, the small-batch choice is highly effective when the scaled across-group term is small. If $\delta_1$ is large, increasing $b_{\mathrm{grp}}$ to $n$ eliminates the across-group term entirely, proving preferable whenever $m\sqrt{n-m}\,\delta_1 \gtrsim (n-m)L+(n-\sqrt{nm})\delta_2$ (ignoring initial $\Psi^0$ differences).
\end{itemize}

\newpage

\section{Auxiliary results}

This section collects several auxiliary results that we repeatedly use in the proofs.
Most of them are elementary identities for smooth functions, conditional second moments, and uniform sampling without replacement; we state them explicitly to keep the main proof steps focused on the estimator recursions.

\begin{lemma}[Inexact-gradient descent; \citep{li21}]\label{lem:inexact_descent}
Let $f$ be $L$-smooth, let $\gamma>0$, and let $x^+=x-\gamma g$ for some $x,g\in\R^d$. Then
\begin{align*}
f(x^+)
\leq f(x)
+\frac{\gamma}{2}\sqn{g-\nabla f(x)}
-\frac{\gamma}{2}\sqn{\nabla f(x)}
+\left(\frac{L}{2}-\frac{1}{2\gamma}\right)\sqn{x^+-x}.
\end{align*}
\end{lemma}
\begin{proof}
By $L$-smoothness,
\begin{align*}
f(x^+)
\leq f(x)+\inner{\nabla f(x)}{x^+-x}+\frac{L}{2}\sqn{x^+-x}.
\end{align*}
Since $x^+-x=-\gamma g$, the identity
\begin{align*}
\inner{\nabla f(x)}{x^+-x}
=\frac{\gamma}{2}\sqn{g-\nabla f(x)}
-\frac{\gamma}{2}\sqn{\nabla f(x)}
-\frac{1}{2\gamma}\sqn{x^+-x}
\end{align*}
gives the claim.
\end{proof}

\begin{lemma}[Conditional second-moment identities]\label{lem:conditional_second_moment}
Let $\mathcal{G}$ be a sigma-algebra. For any square-integrable random vector $Z$,
\begin{align}
\Expb{\sqn{Z}}{\mathcal{G}}
=\sqn{\Expb{Z}{\mathcal{G}}}
+\Expb{\sqn{Z-\Expb{Z}{\mathcal{G}}}}{\mathcal{G}}.
\label{eq:conditional_variance_decomposition}
\end{align}
Moreover, let $X$ be $\mathcal{G}$-measurable and let $Y_1,\ldots,Y_K$ be conditionally independent given $\mathcal{G}$, with finite second moments and $\Expb{Y_k}{\mathcal{G}}=0$ for all $k$. Then, for any deterministic scalars $a_1,\ldots,a_K$,
\begin{align}
\Expb{\sqn{X+\sum_{k=1}^K a_kY_k}}{\mathcal{G}}
=\sqn{X}+\sum_{k=1}^K a_k^2\Expb{\sqn{Y_k}}{\mathcal{G}}.
\label{eq:conditional_zero_mean_sum}
\end{align}
\end{lemma}
\begin{proof}
The first identity follows by expanding
$Z=\Expb{Z}{\mathcal{G}}+\left(Z-\Expb{Z}{\mathcal{G}}\right)$ and using that the second term has conditional mean zero.
For the second identity, expand the squared norm. The cross term with $X$ vanishes because each $Y_k$ has zero conditional mean, and the cross terms
$\Expb{\inner{Y_k}{Y_\ell}}{\mathcal{G}}$ vanish for $k\neq \ell$ by conditional independence and zero conditional means.
\end{proof}

\begin{lemma}[Finite-population variance; Proposition~1 in \citep{con22mu}]\label{lem:prop1}
Let $N\geq 1$, $B\in [N]$, and vectors $(v_i)_{i=1}^N$. Let $v\eqdef \frac{1}{N}\sum_{i=1}^N v_i$ and let $\Omega$ be a subset of $[N]$ of size $B$ chosen uniformly at random.
Then 
\begin{align}
\Exp{\sqn{\frac{1}{B}\sum_{i\in\Omega}v_i-v}}=\frac{N-B}{BN(N-1)}\sum_{i=1}^N \sqn{v_i-v}
\end{align}
if $N\geq 2$, or equals zero if $N=B=1$.
\end{lemma}

\begin{lemma}[Uniformity of anchored subset sampling]\label{lem:anchored_subset_uniform}
Let $N\geq 1$ and $B\in[N]$. Draw $i\sim\mathcal{U}([N])$ and, conditionally on $i$, draw $\Omega$ uniformly among all subsets of $[N]\setminus\{i\}$ of size $B-1$. Set $\tilde{\Omega}\eqdef \Omega\cup\{i\}$. Then $\tilde{\Omega}$ is uniformly distributed over all subsets of $[N]$ of size $B$.
Moreover, for any $k\in[N]$ with $B<N$,
\begin{align*}
\Prob(k\notin\tilde{\Omega})=\frac{N-B}{N},
\end{align*}
and, conditionally on $\{k\notin\tilde{\Omega}\}$, the set $\tilde{\Omega}$ is uniformly distributed over all subsets of $[N]\setminus\{k\}$ of size $B$.
\end{lemma}
\begin{proof}
Fix $S\subset[N]$ with $|S|=B$. The event $\{\tilde{\Omega}=S\}$ occurs exactly when $i=s$ for some $s\in S$ and $\Omega=S\setminus\{s\}$. Hence
\begin{align*}
\Prob(\tilde{\Omega}=S)
=\sum_{s\in S}\frac{1}{N}\frac{1}{\binom{N-1}{B-1}}
=\frac{B}{N\binom{N-1}{B-1}}
=\frac{1}{\binom{N}{B}},
\end{align*}
which proves uniformity. The probability of excluding a fixed $k$ is therefore
\begin{align*}
\Prob(k\notin\tilde{\Omega})
=\frac{\binom{N-1}{B}}{\binom{N}{B}}
=\frac{N-B}{N}.
\end{align*}
For any $S\subset[N]\setminus\{k\}$ with $|S|=B$, Bayes' rule gives
\begin{align*}
\Prob(\tilde{\Omega}=S\mid k\notin\tilde{\Omega})
=\frac{\binom{N}{B}^{-1}}{(N-B)/N}
=\frac{1}{\binom{N-1}{B}},
\end{align*}
so the conditional law is uniform over the $B$-subsets of $[N]\setminus\{k\}$.
\end{proof}

\begin{lemma}[Young-type bounds]\label{lem:young_bounds}
For any $u,v\in\R^d$ and any $\alpha>0$,
\begin{align}
-2\inner{u}{v}\leq \frac{1}{\alpha}\sqn{u}+\alpha\sqn{v}.
\label{eq:young_cross_term}
\end{align}
Consequently, for any $\alpha>0$,
\begin{align}
\sqn{u+v}\leq (1+\alpha)\sqn{u}+\left(1+\frac{1}{\alpha}\right)\sqn{v}.
\label{eq:young_two_vector}
\end{align}
In particular, $\sqn{u+v}\leq 2\sqn{u}+2\sqn{v}$.
\end{lemma}
\begin{proof}
The first bound follows from
$0\leq \sqn{\alpha^{-1/2}u+\alpha^{1/2}v}$.
The second follows by expanding $\sqn{u+v}$ and applying
$2\inner{u}{v}\leq \alpha\sqn{u}+\alpha^{-1}\sqn{v}$.
\end{proof}

\begin{lemma}[\citep{li21, fat21, ric21, richtarik20223pc}]\label{lem:quadratic_step_bound}
Let $a\geq 0$ and $b>0$. If
\begin{align*}
0\leq \gamma \leq \frac{1}{\sqrt{a}+b},
\end{align*}
then
\begin{align*}
a\gamma^2+b\gamma \leq 1.
\end{align*}
Moreover, if $a>0$, the bound is tight up to a factor of $2$ since
\begin{align*}
\frac{1}{\sqrt{a}+b}\leq \min\left\{\frac{1}{\sqrt{a}},\frac{1}{b}\right\}\leq \frac{2}{\sqrt{a}+b}.
\end{align*}
\end{lemma}

\newpage
\section{\algname{SILAGE}, case $m \geq n$}

\subsection{Proofs for general nonconvex functions}

We first prove the one-step estimator recursions needed for the nonconvex analysis.
This is the main helper lemma for the case $m\ge n$: it tracks both the average table error and the error of the aggregated estimator used in the descent step.

\begin{lemma}\label{lem:silage_variance_bounds}
Consider Algorithm~\ref{algsilage1} (case $m \geq n$). Let $\mathcal{F}^t$ be the filtration generated by the algorithm up to iteration $t$. Under the assumption that the internal similarity condition \eqref{eqs2} holds with parameter $\delta_2$, the following bounds hold for the gradient estimators:

1.  \textbf{Component-wise Recursion:} For any $i \in [n]$,
    \begin{align*}
    \Expc{\sqn{g_i^{t+1}-\nabla f_i(x^{t+1})}} &= \left(1-\frac{p}{n}\right) \sqn{g_i^{t}-\nabla f_i(x^{t})}\\
    &\quad+\left(1-\frac{p}{n}\right)\frac{1}{m}\sum_{j=1}^m \sqn{\nabla f_{i,j}(x^{t+1})-\nabla f_i(x^{t+1})-\nabla f_{i,j}(x^{t})+\nabla f_i(x^{t})}.
    \end{align*}

2.  \textbf{Average Component Error:}
    \begin{align*}
    \Expc{\frac{1}{n}\sum_{i=1}^n\sqn{g_i^{t+1}-\nabla f_i(x^{t+1})}} &\leq \left(1-\frac{p}{n}\right)\frac{1}{n}\sum_{i=1}^n\sqn{g_i^{t}-\nabla f_i(x^{t})} +\left(1-\frac{p}{n}\right)\delta_2^2\sqn{x^{t+1}-x^t}.
    \end{align*}

3.  \textbf{Global Estimator Error:}
    \begin{align*}
    \Expc{\sqn{g^{t+1}-\nabla f(x^{t+1})}} &\leq \left(1-\frac{2p}{n}\right)\sqn{g^t-\nabla f(x^t)}+\frac{p}{n^3}\sum_{i=1}^n\sqn{g_{i}^t-\nabla f_{i}(x^{t})} \\
    &\quad +\frac{n-p}{n^2}\delta_2^2\sqn{x^{t+1}-x^t}.
    \end{align*}
\end{lemma}

\begin{proof}
Let $\mathcal{F}^t \eqdef \sigma\rb{x^0, \{g_i^0\}_{i=1}^n, \dots, x^t, \{g_i^t\}_{i=1}^n}$ be the $\sigma$-algebra generated by the iterates and the individual gradient estimators (table entries) up to iteration $t$; in particular, the averaged estimator $g^t = \frac{1}{n}\sum_{i=1}^n g_i^t$ is $\mathcal{F}^t$-measurable. The transition from $x^t$ to $x^{t+1}$ is deterministic given $\mathcal{F}^t$ (since $x^{t+1} = x^t - \gamma g^t$). However, the construction of $g^{t+1}$ depends on the independent random variables realized at step $t$: the coin flip $\theta^t \in \{0,1\}$, the index $i^t \sim \mathcal{U}([n])$ (if $\theta^t=1$), and the minibatch indices $j_i^t \sim \mathcal{U}([m])$.

Fix an index $i \in [n]$. To analyze the error $\Expc{\sqn{g_i^{t+1}-\nabla f_i(x^{t+1})}}$, we apply the \textbf{law of total expectation} by conditioning on the update rule applied to index $i$. From Algorithm~\ref{algsilage1}, there are two mutually exclusive cases for updating $g_i^{t+1}$:
\begin{enumerate}
    \item \textbf{Exact Reset (Case 1):} This occurs if $\theta^t = 1$ and $i^t = i$. The probability of this event is:
    $$ \Prob(\theta^t=1) \cdot \Prob(i^t=i \mid \theta^t=1) = p \cdot \frac{1}{n}. $$
    In this case, $g_i^{t+1} \eqdef \nabla f_i(x^{t+1})$, so the error is zero.
    
    \item \textbf{Recursive Update (Case 2):} This occurs if $\theta^t = 0$, or if $\theta^t = 1$ but $i^t \neq i$. The probability of this event is the complement of Case 1:
    $$ 1 - \frac{p}{n}. $$
    In this case, $g_i^{t+1} \eqdef g_i^t + \nabla f_{i, j_i^t}(x^{t+1}) - \nabla f_{i, j_i^t}(x^t)$.
\end{enumerate}

Applying the law of total expectation over the random choices at iteration $t$, we obtain:
\begin{align*}
&\Expc{\sqn{g_i^{t+1}-\nabla f_i(x^{t+1})}} \\
&= \frac{p}{n} \cdot 0 + \left(1 - \frac{p}{n}\right) \Expb{\sqn{g_i^t + \nabla f_{i, j_i^t}(x^{t+1}) - \nabla f_{i, j_i^t}(x^t) - \nabla f_i(x^{t+1})}}{\mathcal{F}^t} \\
&= \left(1 - \frac{p}{n}\right) \Expb{\left\|g_i^t - \nabla f_i(x^t) + \underbrace{\nabla f_{i, j_i^t}(x^{t+1}) - \nabla f_{i, j_i^t}(x^t) - (\nabla f_i(x^{t+1}) - \nabla f_i(x^t))}_{:=\, \xi_{i,j}^t}\right\|^2}{\mathcal{F}^t} \\
&= \left(1-\frac{p}{n}\right) \sqn{g_i^{t}-\nabla f_i(x^{t})}\\
&\quad+\left(1-\frac{p}{n}\right)\frac{1}{m}\sum_{j=1}^m \sqn{\nabla f_{i,j}(x^{t+1})-\nabla f_i(x^{t+1})-\nabla f_{i,j}(x^{t})+\nabla f_i(x^{t})}.
\end{align*}
The first equality separates the case where $g_i$ is reset exactly (with probability $p/n$) from the recursive update case. The final equality relies on the standard variance decomposition $\Exp{\sqn{X+Y}} = \sqn{X} + \Exp{\sqn{Y}}$, valid here because the random vector $\xi_{i,j}^t$ has zero mean conditioned on $\mathcal{F}^t$.

Therefore,
 \begin{align*}
&\Expc{\frac{1}{n}\sum_{i=1}^n\sqn{g_i^{t+1}-\nabla f_i(x^{t+1})}} = \left(1-\frac{p}{n}\right) \frac{1}{n}\sum_{i=1}^n\sqn{g_i^{t}-\nabla f_i(x^{t})}\\
&\quad+\left(1-\frac{p}{n}\right)\frac{1}{nm}\sum_{i=1}^n\sum_{j=1}^m \sqn{\nabla f_{i,j}(x^{t+1})-\nabla f_i(x^{t+1})-\nabla f_{i,j}(x^{t})+\nabla f_i(x^{t})}\\
&\leq \left(1-\frac{p}{n}\right)\frac{1}{n}\sum_{i=1}^n\sqn{g_i^{t}-\nabla f_i(x^{t})} +\left(1-\frac{p}{n}\right)\delta_2^2\sqn{x^{t+1}-x^t}.
\end{align*}

Similarly, for the global gradient estimator $g^{t+1}$, we condition on the coin flip $\theta^t$. When $\theta^t=0$, all indices update recursively using independent minibatches $j_i^t$. Exploiting this independence, the variance of the sum becomes the sum of the variances:
\begin{align*}
&\Expc{\sqn{g^{t+1}-\nabla f(x^{t+1})}} \\
&= (1-p)\Expb{\left\|g^{t}-\nabla f(x^{t}) + \frac{1}{n}\sum_{i=1}^n \underbrace{\left( \nabla f_{i,j_i^t}(x^{t+1})-\nabla f_{i,j_i^t}(x^{t}) - (\nabla f_i(x^{t+1})-\nabla f_i(x^{t})) \right)}_{\xi_{i,j}^t}\right\|^2}{\mathcal{F}^t} \\
&\quad + p \Expb{\sqn{g^{t+1}-\nabla f(x^{t+1})}}{\mathcal{F}^t,\theta^t=1} \\
&=(1-p)\left(\sqn{g^{t}-\nabla f(x^{t})} + \frac{1}{n^2}\sum_{i=1}^n \Expb{\sqn{\xi_{i,j}^t}}{\mathcal{F}^t}\right) \\
&\quad + p \Expb{\left\|g^{t}-\frac{1}{n}g_{i^t}^t+\frac{1}{n}\sum_{i\in[n]\backslash\{i^t\}}\left(\nabla f_{i,j_i^t}(x^{t+1})-\nabla f_{i}(x^{t+1})-\nabla f_{i,j_i^t}(x^{t})\right)\right\|^2}{\mathcal{F}^t,\theta^t=1} \\
&=(1-p)\sqn{g^{t}-\nabla f(x^{t})} \\
&\quad+\frac{1-p}{n^2 m}\sum_{i=1}^n \sum_{j=1}^m \sqn{\nabla f_{i,j}(x^{t+1})-\nabla f_i(x^{t+1})-\nabla f_{i,j}(x^{t})+\nabla f_i(x^{t})} \\
&\quad + p \Expb{\left\|g^{t}-\frac{1}{n}g_{i^t}^t+\frac{1}{n}\sum_{i\in[n]\backslash\{i^t\}}\left(\nabla f_{i,j_i^t}(x^{t+1})-\nabla f_{i}(x^{t+1})-\nabla f_{i,j_i^t}(x^{t})\right)\right\|^2}{\mathcal{F}^t,\theta^t=1}.
\end{align*}
The second equality applies the identity $\Exp{\sqn{X+\sum Y_i}} = \sqn{X} + \sum \Exp{\sqn{Y_i}}$ for independent zero-mean variables $Y_i$. The final equality expands the expectation of the stochastic noise terms $\xi_{i,j}^t$ and substitutes the specific update rule for the $\theta^t=1$ case.

To analyze the expectation conditioned on $\theta^t=1$, we first simplify the term inside the norm. By adding and subtracting $\nabla f_i(x^t)$ within the summation to recover the noise terms $\xi_{i,j}^t$, and applying the identity $\sum_{i \neq k} a_i = \sum_{\text{all}} a_i - a_k$, we derive the following algebraic relation:
\begin{align*}
&g^{t}-\frac{1}{n}g_{i^t}^t+\frac{1}{n}\sum_{i \in [n]\setminus\{i^t\}}\left(\nabla f_{i,j_i^t}(x^{t+1})-\nabla f_{i}(x^{t+1})-\nabla f_{i,j_i^t}(x^{t})\right) \\
&= g^{t}-\frac{1}{n}g_{i^t}^t + \frac{1}{n}\sum_{i \neq i^t} \left( \xi_{i,j}^t + \nabla f_i(x^t) - \nabla f_i(x^t) \right) \\
&= g^{t}-\frac{1}{n}g_{i^t}^t + \frac{1}{n}\sum_{i \neq i^t} \left( \nabla f_i(x^t) + \xi_{i,j}^t \right) - \frac{n-1}{n}\nabla f(x^t) \\ 
&= \frac{1}{n}\sum_{i \neq i^t} \left( g_i^t - \nabla f_i(x^t) \right) + \frac{1}{n}\sum_{i \neq i^t} \xi_{i,j}^t \\
&= \frac{1}{n} \left[ \sum_{i=1}^n (g_i^t - \nabla f_i(x^t)) - (g_{i^t}^t - \nabla f_{i^t}(x^t)) \right] + \frac{1}{n}\sum_{i \neq i^t} \xi_{i,j}^t \\
&= \left( g^t - \nabla f(x^t) \right) - \frac{1}{n}\left( g_{i^t}^t - \nabla f_{i^t}(x^t) \right) + \frac{1}{n}\sum_{i \neq i^t} \xi_{i,j}^t \\
&= \frac{n-1}{n}\left(g^t-\nabla f(x^t)\right) + \frac{1}{n}\left(g^{t}-g_{i^t}^t+\nabla f_{i^t}(x^{t})-\nabla f(x^{t})\right) + \frac{1}{n}\sum_{i \neq i^t} \xi_{i,j}^t.
\end{align*}
Substituting this decomposition back into the main inequality, and expanding the square, we obtain:
\begin{align*}
&\Expc{\sqn{g^{t+1}-\nabla f(x^{t+1})}} \\
&\leq (1-p)\sqn{g^{t}-\nabla f(x^{t})} + \frac{1-p}{n} \delta_2^2\sqn{x^{t+1}-x^t} \\
&\quad + p \Expb{\left\| \frac{n-1}{n}\left(g^t-\nabla f(x^t)\right) + \frac{1}{n}\left(g^{t}-g_{i^t}^t+\nabla f_{i^t}(x^{t})-\nabla f(x^{t})\right) + \frac{1}{n}\sum_{i \neq i^t} \xi_{i,j}^t \right\|^2}{\mathcal{F}^t,\theta^t=1} \\
&= (1-p)\sqn{g^{t}-\nabla f(x^{t})}+\frac{1-p}{n} \delta_2^2\sqn{x^{t+1}-x^t} + p \left(\frac{n-1}{n}\right)^2 \sqn{g^t-\nabla f(x^t)} \\
&\quad +\frac{p}{n^2}\Expb{\sqn{g^{t}-g_{i^t}^t+\nabla f_{i^t}(x^{t})-\nabla f(x^{t})}}{\mathcal{F}^t,\theta^t=1} \\
&\quad+\frac{p}{n^2}\Expb{\sum_{i\in[n]\backslash\{i^t\}}\sqn{\xi_{i,j}^t}}{\mathcal{F}^t,\theta^t=1}.
\end{align*}
where the last inequality follows from the fact that the $j_i^t$ are independent of each other, so that the variance of the sum is the sum of the variances. 

To resolve the sum over the random set $[n]\setminus\{i^t\}$, we rewrite the expectation using indicator variables $\mathbb{I}(i \neq i^t)$. Since $i^t$ is uniform on $[n]$, $\Prob(i \neq i^t) = \frac{n-1}{n}$. Furthermore, the randomness of $j_i^t$ is independent of $i^t$. Thus:
\begin{align*}
&\Expb{\sum_{i\in[n]\backslash\{i^t\}}\sqn{\xi_{i,j}^t}}{\mathcal{F}^t,\theta^t=1}\\
&= \Expb{\sum_{i=1}^n \mathbb{I}(i \neq i^t) \sqn{\xi_{i,j}^t}}{\mathcal{F}^t,\theta^t=1} \\
&= \sum_{i=1}^n \Prob(i \neq i^t) \Expb{\sqn{\xi_{i,j}^t}}{\mathcal{F}^t} \\
&= \sum_{i=1}^n \frac{n-1}{n} \cdot \frac{1}{m}\sum_{j=1}^m \sqn{\nabla f_{i,j}(x^{t+1})-\nabla f_{i}(x^{t+1})-\nabla f_{i,j}(x^{t})+\nabla f_{i}(x^{t})}\\
&=\frac{n-1}{nm}\sum_{i=1}^n \sum_{j=1}^m \sqn{\nabla f_{i,j}(x^{t+1})-\nabla f_{i}(x^{t+1})-\nabla f_{i,j}(x^{t})+\nabla f_{i}(x^{t})}\\
&\leq (n-1)\delta_2^2\sqn{x^{t+1}-x^t}.
\end{align*}

Note that,
\begin{align*}
&\Expb{\sqn{g^{t}-g_{i^t}^t+\nabla f_{i^t}(x^{t})-\nabla f(x^{t})}}{\mathcal{F}^t,\theta^t=1} \\
&= \Expb{\sqn{(g^t - \nabla f(x^t)) - (g_{i^t}^t - \nabla f_{i^t}(x^t))}}{\mathcal{F}^t,\theta^t=1} \\
&= \frac{1}{n}\sum_{i=1}^n \sqn{g_i^t - \nabla f_i(x^t)} - \sqn{g^t - \nabla f(x^t)}.
\end{align*}

Hence,
 \begin{align*}
\Expc{\sqn{g^{t+1}-\nabla f(x^{t+1})}} &\leq (1-p)\sqn{g^{t}-\nabla f(x^{t})}+ \frac{p(n-1)^2}{n^2}\sqn{g^t-\nabla f(x^t)}\\
&\quad+\frac{p}{n^3}\sum_{i=1}^n\sqn{g_{i}^t-\nabla f_{i}(x^{t})}-\frac{p}{n^2}\sqn{g^t-\nabla f(x^{t})}\\
&\quad+\frac{1-p}{n} \delta_2^2\sqn{x^{t+1}-x^t}+\frac{p(n-1)}{n^2}\delta_2^2\sqn{x^{t+1}-x^t}\\
&= \left(1-\frac{2p}{n}\right)\sqn{g^t-\nabla f(x^t)}+\frac{p}{n^3}\sum_{i=1}^n\sqn{g_{i}^t-\nabla f_{i}(x^{t})}\\
&\quad+\frac{n-p}{n^2}\delta_2^2\sqn{x^{t+1}-x^t}.
\end{align*}
\end{proof}

\subsubsection*{Proof of Theorem~\ref{theo1}}

We next restate the main nonconvex convergence result for Algorithm~\ref{algsilage1}.
The proof combines the descent inequality for smooth functions with Lemma~\ref{lem:silage_variance_bounds} to build a Lyapunov recursion.

\silageMgeNNC*

\begin{proof}[Proof of Theorem~\ref{theo1}]

We rely on the variance bounds established in Lemma~\ref{lem:silage_variance_bounds}. Specifically, item 3 of the Lemma provides the bound for the global estimator error:
\begin{align*}
\Expc{\sqn{g^{t+1}-\nabla f(x^{t+1})}} &\leq \left(1-\frac{2p}{n}\right)\sqn{g^t-\nabla f(x^t)}+\frac{p}{n^3}\sum_{i=1}^n\sqn{g_{i}^t-\nabla f_{i}(x^{t})} \\
&\quad +\frac{n-p}{n^2}\delta_2^2\sqn{x^{t+1}-x^t}.
\end{align*}
Additionally, item 2 of the Lemma provides the bound for the average component error:
\begin{align*}
\Expc{\frac{1}{n}\sum_{i=1}^n\sqn{g_i^{t+1}-\nabla f_i(x^{t+1})}} &\leq \left(1-\frac{p}{n}\right)\frac{1}{n}\sum_{i=1}^n\sqn{g_i^{t}-\nabla f_i(x^{t})} +\left(1-\frac{p}{n}\right)\delta_2^2\sqn{x^{t+1}-x^t}.
\end{align*}

We have the descent lemma
\begin{align*}
f(x^{t+1})
&\leq f(x^t)
+\frac{\gamma}{2}\sqn{g^{t}-\nabla f(x^t)}
-\frac{\gamma}{2}\sqn{\nabla f(x^t)}\\
&\quad+\left(\frac{L}{2}-\frac{1}{2\gamma}\right)\sqn{x^{t+1}-x^t}.
\end{align*}
We assume $p>0$. We define the Lyapunov function 
$$\Psi^t\eqdef f(x^t)-\finf + \frac{\gamma n}{4p}\sqn{g^{t}-\nabla f(x^t)}+\frac{\gamma}{4np}\sum_{i=1}^n\sqn{g_i^t-\nabla f_{i}(x^{t})}.$$ 

Let us denote the Lyapunov function coefficients as $C_1 \eqdef \frac{\gamma n}{4p}$ and $C_2 \eqdef \frac{\gamma}{4np}$. Taking the conditional expectation of $\Psi^{t+1}$ and substituting the descent lemma, the global gradient recursion, and the component-wise gradient recursion, we have:
\begin{align*}
\Expc{\Psi^{t+1}} &= \Expc{f(x^{t+1})} - \finf + C_1 \Expc{\sqn{g^{t+1}-\nabla f(x^{t+1})}} + C_2 \Expc{\sum_{i=1}^n \sqn{g_i^{t+1}-\nabla f_i(x^{t+1})}} \\
&\leq \underbrace{f(x^t) - \finf +\frac{\gamma}{2}\sqn{g^{t}-\nabla f(x^t)}-\frac{\gamma}{2}\sqn{\nabla f(x^t)} + \left(\frac{L}{2}-\frac{1}{2\gamma}\right) \sqn{x^{t+1}-x^t}}_{\text{Descent Lemma}} \\
&\quad + C_1 \left[ \left(1-\frac{2p}{n}\right)\sqn{g^t-\nabla f(x^t)}+\frac{p}{n^3}\sum_{i=1}^n\sqn{g_{i}^t-\nabla f_{i}(x^{t})} + \frac{n-p}{n^2}\delta_2^2\sqn{x^{t+1}-x^t} \right] \\
&\quad + C_2 \left[ \left(1-\frac{p}{n}\right)\sum_{i=1}^n \sqn{g_i^{t}-\nabla f_i(x^{t})} + n\left(1-\frac{p}{n}\right)\delta_2^2\sqn{x^{t+1}-x^t} \right].
\end{align*}
We now group the terms by the error components $\sqn{g^t-\nabla f(x^t)}$, $\sum \sqn{g_i^t-\nabla f_i(x^t)}$, and the stepsize $\sqn{x^{t+1}-x^t}$:
\begin{align*}
\Expc{\Psi^{t+1}} &\leq f(x^t) - \finf -\frac{\gamma}{2}\sqn{\nabla f(x^t)} \\
&\quad + \left( \frac{\gamma}{2} + C_1\left(1-\frac{2p}{n}\right) \right) \sqn{g^{t}-\nabla f(x^t)} \\
&\quad + \left( C_1\frac{p}{n^3} + C_2\left(1-\frac{p}{n}\right) \right) \sum_{i=1}^n\sqn{g_i^t-\nabla f_{i}(x^{t})} \\
&\quad + \left( \frac{L}{2}-\frac{1}{2\gamma} + C_1\frac{n-p}{n^2}\delta_2^2 + C_2 n\left(\frac{n-p}{n}\right)\delta_2^2 \right) \sqn{x^{t+1}-x^t}.
\end{align*}
Substituting $C_1$ and $C_2$ back into the coefficients reveals the simplifications:
\begin{align*}
\Expc{\Psi^{t+1}} &\leq f(x^t) -\finf -\frac{\gamma}{2}\sqn{\nabla f(x^t)} \\
&\quad + \left( \frac{\gamma}{2} + \frac{\gamma n}{4p} - \frac{\gamma n}{4p}\cdot\frac{2p}{n} \right) \sqn{g^{t}-\nabla f(x^t)}\\
&\quad + \left( \frac{\gamma n}{4p}\cdot\frac{p}{n^3} + \frac{\gamma}{4np}\left(1-\frac{p}{n}\right) \right)\sum_{i=1}^n\sqn{g_i^t-\nabla f_{i}(x^{t})}\\
&\quad + \left(\frac{L}{2}-\frac{1}{2\gamma}+ \frac{\gamma n}{4p}\cdot\frac{n-p}{n^2}\delta_2^2 + \frac{\gamma}{4np}\cdot(n-p)\delta_2^2 \right) \sqn{x^{t+1}-x^t}\\
&= f(x^t) - \finf -\frac{\gamma}{2}\sqn{\nabla f(x^t)} \\
&\quad + \underbrace{\left( \frac{\gamma}{2} + \frac{\gamma n}{4p} - \frac{\gamma}{2} \right)}_{= \frac{\gamma n}{4p}} \sqn{g^{t}-\nabla f(x^t)}\\
&\quad + \underbrace{\left( \frac{\gamma}{4n^2} + \frac{\gamma}{4np} - \frac{\gamma}{4n^2} \right)}_{= \frac{\gamma}{4np}}\sum_{i=1}^n\sqn{g_i^t-\nabla f_{i}(x^{t})}\\
&\quad + \left(\frac{L}{2}-\frac{1}{2\gamma}+ \frac{\gamma(n-p)}{4np}\delta_2^2 + \frac{\gamma(n-p)}{4np}\delta_2^2 \right) \sqn{x^{t+1}-x^t}.
\end{align*}
Recognizing that the coefficients for the gradient error terms match the definition of $\Psi^t$, we obtain:
\begin{align*}
\Expc{\Psi^{t+1}} &\leq \Psi^t -\frac{\gamma}{2}\sqn{\nabla f(x^t)} +\left(\frac{L}{2}-\frac{1}{2\gamma}+\frac{\gamma(n-p)}{2np}\delta_2^2 \right) \sqn{x^{t+1}-x^t}.
\end{align*}

Hence, the last term is nonpositive whenever
\begin{align*}
L-\frac{1}{\gamma}+\gamma\,\delta_2^2\frac{n-p}{np}\leq 0,
\end{align*}
that is, whenever
\begin{align*}
\delta_2^2\frac{n-p}{np}\gamma^2+L\gamma\leq 1.
\end{align*}
By Lemma~\ref{lem:quadratic_step_bound}, applied with
\begin{align*}
a &\eqdef \delta_2^2\frac{n-p}{np},
\qquad
b \eqdef L,
\end{align*}
it therefore suffices to impose
\begin{align*}
\gamma\leq \frac{1}{L+\delta_2\sqrt{\frac{n-p}{np}}},
\end{align*}
which is exactly the stepsize condition in the theorem.

Assuming that this condition holds, and unrolling the recursion, we obtain, for every $T\geq 0$, by defining $\tilde{x}^T$ as $x^t$ for $t$ chosen uniformly at random in $\{0,\ldots,T\}$,
\begin{align*}
\Exp{\sqn{\nabla f(\tilde{x}^T)}} =\frac{1}{T+1}\sum_{t=0}^{T} \Exp{\sqn{\nabla f(x^t)}} &\leq \frac{2\Psi^0}{\gamma(T+1)}.
\end{align*}

\end{proof}

The following corollary turns Theorem~\ref{theo1} into a component-gradient complexity bound when the initial estimators are arbitrary.
It then specializes the bound to the choice $p=n/m$, which gives the stated complexity in the regime $m\ge n$.

\silageMgeNArbitraryInit*

\begin{proof}[Proof of Corollary~\ref{cor:silage_mge_n_arbitrary_init}]
Theorem~\ref{theo1} gives
\begin{align*}
\Exp{\sqn{\nf{\tilde{x}^T}}}
&=
\frac{1}{T+1}\sum_{t=0}^{T}\Exp{\sqn{\nf{x^t}}}
\le
\frac{2\Psi^0}{\gamma(T+1)}.
\end{align*}
Hence, if $T+1\ge 2\Psi^0/(\gamma\epsilon)$, then
\begin{align*}
\Exp{\sqn{\nf{\tilde{x}^T}}}\le \epsilon.
\end{align*}
Since
\begin{align*}
\frac1\gamma
&=
L+\delta_2\sqrt{\frac{n-p}{np}}
\le
L+\frac{\delta_2}{\sqrt p},
\end{align*}
the required number of iterations satisfies
\begin{align*}
T
=
\cO\rb{
\rb{L+\frac{\delta_2}{\sqrt p}}
\frac{\Psi^0}{\epsilon}
}.
\end{align*}

It remains to count component-gradient evaluations. At one iteration of
Algorithm~\ref{algsilage1}, there are two cases. If $\theta^t=0$, which happens
with probability $1-p$, the algorithm updates every group estimator using one
sample $j_i^t\in[m]$ per group $i\in[n]$. For each such group it computes
\begin{align*}
\nabla f_{i,j_i^t}(x^{t+1})
\qquad\text{and}\qquad
\nabla f_{i,j_i^t}(x^t),
\end{align*}
so this case costs $2n$ component-gradient evaluations.

If $\theta^t=1$, which happens with probability $p$, the algorithm samples one
group $i^t\in[n]$ and computes the full group gradient
\begin{align*}
\nabla f_{i^t}(x^{t+1})
=
\frac1m\sum_{j=1}^m \nabla f_{i^t,j}(x^{t+1}),
\end{align*}
which costs $m$ component-gradient evaluations. For the remaining $n-1$ groups,
the algorithm performs the same two-point component-gradient update as above,
costing $2(n-1)$ additional component-gradient evaluations. Therefore, this
case costs $m+2(n-1)$ component-gradient evaluations.

Thus the expected number of component-gradient evaluations per iteration is
\begin{align*}
(1-p)\,2n+p\rb{m+2(n-1)}
&=
2n+pm-2p
=
\cO\rb{n+pm}.
\end{align*}
Multiplying this expected per-iteration cost by the iteration bound yields
\begin{align*}
\cO\rb{
(n+pm)
\rb{L+\frac{\delta_2}{\sqrt p}}
\frac{\Psi^0}{\epsilon}
}.
\end{align*}

Finally, because $m\ge n$, the choice $p=n/m$ belongs to $(0,1]$. Substituting
$p=n/m$ gives
\begin{align*}
n+pm
&=
n+\frac{n}{m}m
=
2n,
\\
L+\frac{\delta_2}{\sqrt p}
&=
L+\delta_2\sqrt{\frac{m}{n}}.
\end{align*}
Hence
\begin{align*}
(n+pm)\rb{L+\frac{\delta_2}{\sqrt p}}
&=
2n\rb{L+\delta_2\sqrt{\frac{m}{n}}}
\\
&=
2nL+2\sqrt{nm}\,\delta_2,
\end{align*}
which gives
\begin{align*}
\cO\rb{
\rb{nL+\sqrt{nm}\,\delta_2}
\frac{\Psi_{n/m}^0}{\epsilon}
}.
\end{align*}

We now look at the case of exact initialization.


Since $m\ge n$, the choice $p=n/m$ belongs to $(0,1]$. For this value of $p$, the stepsize condition in Theorem~\ref{theo1} gives
\begin{align*}
\gamma
&=
\rb{
L+\delta_2\sqrt{\frac{n-p}{np}}
}^{-1}
=
\rb{
L+\delta_2\sqrt{
\frac{n-\frac{n}{m}}{n\frac{n}{m}}
}
}^{-1}
\\
&=
\rb{
L+\delta_2\sqrt{\frac{m-1}{n}}
}^{-1}.
\end{align*}

Next, the exact initialization $g_i^0=\nfi{x^0}$ for every $i\in[n]$ implies
\begin{align*}
g^0
&=
\frac1n\sum_{i=1}^n g_i^0
=
\frac1n\sum_{i=1}^n \nfi{x^0}
=
\nf{x^0}.
\end{align*}
Therefore, both initialization-error terms in the Lyapunov quantity from Theorem~\ref{theo1} vanish:
\begin{align*}
\sqn{g^0-\nf{x^0}}
&=0,
&
\frac1n\sum_{i=1}^n\sqn{g_i^0-\nfi{x^0}}
&=0.
\end{align*}
Hence, for $p=n/m$,
\begin{align*}
\Psi_{n/m}^0
&=
f(x^0)-\finf
=
\Delta_0.
\end{align*}

Thus, setting $p=n/m$ and using $\Psi_{n/m}^0=\Delta_0$ gives the post-initialization iteration complexity
\begin{align*}
T
&=
\cO\rb{
\rb{
L+\frac{\delta_2}{\sqrt{n/m}}
}
\frac{\Delta_0}{\epsilon}
}
=
\cO\rb{
\rb{
L+\delta_2\sqrt{\frac{m}{n}}
}
\frac{\Delta_0}{\epsilon}
}.
\end{align*}
The same corollary gives the expected post-initialization component-gradient complexity
\begin{align*}
\cO\rb{
\rb{
nL+\sqrt{nm}\,\delta_2
}
\frac{\Delta_0}{\epsilon}
}.
\end{align*}

It remains to add the initialization cost. Computing $g_i^0=\nfi{x^0}$ requires evaluating the full group gradient
\begin{align*}
\nfi{x^0}
&=
\frac1m\sum_{j=1}^m \nabla f_{i,j}(x^0),
\end{align*}
which costs $m$ component-gradient evaluations for a fixed group $i$. Repeating this for all $i\in[n]$ costs
\begin{align*}
\sum_{i=1}^n m
&=
mn
\end{align*}
component-gradient evaluations. Therefore, the total expected component-gradient complexity, including initialization, is
\begin{align*}
\cO\rb{
nm +
\rb{
nL+\sqrt{nm}\,\delta_2
}
\frac{\Delta_0}{\epsilon}
}.
\end{align*}
\end{proof}

\subsection{Analysis under the Polyak--\L{}ojasiewicz condition}\label{sec:app-pl-mge-n}

We next record the P\L{} specialization for the case $m\ge n$.
The algorithm and the estimator recursion are unchanged, but the P\L{} inequality turns the descent estimate into a contraction for the objective gap.
Thus the accuracy dependence becomes $O\left(\cdot\log\frac{1}{\epsilon}\right)$ for reaching $\Exp{f(x^T)-\fstar}\le \epsilon$, instead of the $O\left(\cdot\frac{1}{\epsilon}\right)$ dependence in the preceding general nonconvex stationarity bounds.
The theorem is stated under a global P\L{} condition.
However, the proof uses this condition only at the iterates, so the same argument applies on any P\L{} region that contains the trajectory.

We state the P\L{} condition once here because it is shared by the P\L{} analyses in both regimes.

\begin{assumption}[Polyak--\L{}ojasiewicz (P\L) condition]\label{ass:pl_condition}
The function $f$ satisfies the P\L-condition with parameter $\mu>0$ if
\begin{align}
\label{eq:pl_condition}
\sqn{\nabla f(x)} \ge 2\mu\big(f(x)-\fstar\big),\qquad \forall x\in\R^d,
\end{align}
where $\fstar\eqdef \inf_{x\in\R^d} f(x)>-\infty$.
\end{assumption}

The following lemma is a generic P\L{} recursion.
It lets us separate the algorithm-specific Lyapunov calculation from the final step that converts such a recursion into linear convergence.

\begin{lemma}[\citep{dasha23}]
    \label{lemma:good_recursion_pl}
    Suppose that Assumption~\ref{ass:pl_condition} holds (implying $\fstar > -\infty$). Let $\{\Psi^t\}_{t\ge 0}$ be a sequence of nonnegative real numbers. Suppose that for constants $C \geq 0$, $\mu > 0$, and stepsize $\gamma \in (0, 1 / \mu)$, the following inequality holds for all $t \geq 0$:
    \begin{align*}
        \Exp{f(x^{t+1})} + \gamma \Psi^{t+1} \leq \Exp{f(x^t)} - \frac{\gamma}{2}\Exp{\sqn{\nabla f(x^t)}} + (1 - \gamma \mu)\gamma \Psi^{t} + \gamma C.
    \end{align*}
    Then, for all $T \geq 0$,
    \begin{align}
        \label{eq:good_recursion_pl}
        \Exp{f(x^{T}) - \fstar} \leq (1 - \gamma \mu)^{T}\left((f(x^0) - \fstar) + \gamma \Psi^{0}\right) + \frac{C}{\mu}.
    \end{align}
\end{lemma}

\begin{proof}
    Subtracting $\fstar$ from both sides of the assumed inequality and applying the P\L-condition~\eqref{eq:pl_condition} (specifically, $-\frac{1}{2}\sqn{\nabla f(x^t)} \leq -\mu(f(x^t) - \fstar)$), we obtain
    \begin{align*}
        \Exp{f(x^{t+1}) - \fstar} + \gamma \Psi^{t+1} 
        &\leq \Exp{f(x^t) - \fstar} - \gamma\mu\Exp{f(x^t)-\fstar} + (1 - \gamma \mu)\gamma \Psi^{t} + \gamma C \\
        &= (1 - \gamma \mu)\left(\Exp{f(x^t) - \fstar} + \gamma \Psi^{t}\right) + \gamma C.
    \end{align*}
    Unrolling this recursion from $t=T-1$ down to $0$, we have
    \begin{align*}
        \Exp{f(x^{T}) - \fstar} + \gamma \Psi^{T} 
        &\leq (1 - \gamma \mu)^{T}\left((f(x^0) - \fstar) + \gamma \Psi^{0}\right) + \gamma C \sum_{k = 0}^{T-1} (1 - \gamma \mu)^k \\
        &\leq (1 - \gamma \mu)^{T}\left((f(x^0) - \fstar) + \gamma \Psi^{0}\right) + \frac{C}{\mu},
    \end{align*}
    where in the last step we bounded the finite geometric series by the infinite sum $\sum_{k=0}^\infty (1-\gamma\mu)^k = \frac{1}{\gamma\mu}$.
    Finally, using the nonnegativity of $\Psi^{T}$ and $\gamma$, we have $\Exp{f(x^{T}) - \fstar} \leq \Exp{f(x^{T}) - \fstar} + \gamma \Psi^{T}$, which yields the desired result.
\end{proof}

The next theorem is the P\L{} counterpart of Theorem~\ref{theo1}.
It applies Lemma~\ref{lemma:good_recursion_pl} to Algorithm~\ref{algsilage1} after choosing Lyapunov weights that match the estimator recursions from Lemma~\ref{lem:silage_variance_bounds}.

\begin{restatable}[Linear convergence under P\L, case $m\ge n$]{theorem}{silageMgeNPL}
\label{theo:silage_pl}
Let Assumptions~\ref{ass:lower-bounded}, \ref{ass:f-smooth}, \ref{ass:delta2}, and~\ref{ass:pl_condition} hold,
where Assumption~\ref{ass:pl_condition} holds with parameter $\mu>0$.
Run Algorithm~\ref{algsilage1} with probability $p\in(0,1]$, arbitrary initial estimators
$g_1^0,\ldots,g_n^0\in\R^d$, and $g^0\eqdef \frac1n\sum_{i=1}^n g_i^0$.
Let the stepsize be set as
$0<\gamma\le \min\left\{\rb{L+\delta_2\sqrt{\frac{2(n-p)}{np}}}^{-1},\,\frac{p}{2\mu n}\right\}$.
Then, for every $T\ge 0$,
\begin{align}
\Exp{f(x^T)-\fstar}
\le
(1-\gamma\mu)^T\rb{\Psi_{\mathrm{PL}}^0},
\label{eq:silage_pl_rate}
\end{align}
where
\begin{equation*}
\Psi_{\mathrm{PL}}^0
\eqdef
f(x^0)-\fstar
+\frac{\gamma n}{3p}\sqn{g^0-\nf{x^0}}
+\frac{2\gamma}{3p}\frac1n\sum_{i=1}^n\sqn{g_i^0-\nfi{x^0}}.
\end{equation*}
\end{restatable}

\begin{proof}[Proof of Theorem~\ref{theo:silage_pl}]
Let us fix constants $\nu,\rho\in[0,\infty)$ that we will define later.
We start from the descent lemma for $L$-smooth functions, applied with $x^{t+1}=x^t-\gamma g^t$:
\begin{align}
f(x^{t+1})
&\leq f(x^t) + \frac{\gamma}{2}\sqn{g^{t}-\nabla f(x^t)}-\frac{\gamma}{2}\sqn{\nabla f(x^t)} \notag\\
&\quad + \left(\frac{L}{2}-\frac{1}{2\gamma}\right)\sqn{x^{t+1}-x^t}.
\label{eq:descent_pl_silage}
\end{align}
Taking expectation and adding $\nu\,\Exp{\sqn{g^{t+1}-\nabla f(x^{t+1})}}$ and
$\rho\,\Exp{\frac{1}{n}\sum_{i=1}^n\sqn{g_i^{t+1}-\nabla f_i(x^{t+1})}}$ to both sides, and then using Lemma~\ref{lem:silage_variance_bounds}, we obtain
\begin{align*}
&\Exp{f(x^{t+1})} + \nu\,\Exp{\sqn{g^{t+1}-\nabla f(x^{t+1})}}
+ \rho\,\Exp{\frac{1}{n}\sum_{i=1}^n\sqn{g_i^{t+1}-\nabla f_i(x^{t+1})}}\\
&\leq \Exp{f(x^t)} -\frac{\gamma}{2}\Exp{\sqn{\nabla f(x^t)}}
+ \frac{\gamma}{2}\Exp{\sqn{g^{t}-\nabla f(x^t)}}
+ \left(\frac{L}{2}-\frac{1}{2\gamma}\right)\Exp{\sqn{x^{t+1}-x^t}}\\
&\quad+\nu\,\Exp{\left(1-\frac{2p}{n}\right)\sqn{g^t-\nabla f(x^t)}
+\frac{p}{n^3}\sum_{i=1}^n\sqn{g_{i}^t-\nabla f_{i}(x^{t})}
+\frac{n-p}{n^2}\delta_2^2\sqn{x^{t+1}-x^t}}\\
&\quad+\rho\,\Exp{\left(1-\frac{p}{n}\right)\frac{1}{n}\sum_{i=1}^n\sqn{g_i^{t}-\nabla f_i(x^{t})}
+\left(1-\frac{p}{n}\right)\delta_2^2\sqn{x^{t+1}-x^t}}.
\end{align*}
Rearranging the terms yields
\begin{align}
&\Exp{f(x^{t+1})} + \nu\,\Exp{\sqn{g^{t+1}-\nabla f(x^{t+1})}}
+ \rho\,\Exp{\frac{1}{n}\sum_{i=1}^n\sqn{g_i^{t+1}-\nabla f_i(x^{t+1})}}\notag\\
&\leq \Exp{f(x^t)} -\frac{\gamma}{2}\Exp{\sqn{\nabla f(x^t)}} \notag\\
&\quad+\left(\frac{\gamma}{2}+\nu\left(1-\frac{2p}{n}\right)\right)\Exp{\sqn{g^{t}-\nabla f(x^t)}}\notag\\
&\quad+\left(\rho\left(1-\frac{p}{n}\right)+\nu\frac{p}{n^2}\right)
\Exp{\frac{1}{n}\sum_{i=1}^n\sqn{g_{i}^t-\nabla f_{i}(x^{t})}}\notag\\
&\quad+\left(\frac{L}{2}-\frac{1}{2\gamma}
+\nu\frac{n-p}{n^2}\delta_2^2
+\rho\left(1-\frac{p}{n}\right)\delta_2^2\right)\Exp{\sqn{x^{t+1}-x^t}}.
\label{eq:pl_silage_prechoice}
\end{align}

We now choose
\begin{align}
\label{eq:nu_rho_choice}
\nu \eqdef \frac{\gamma n}{3p},
\qquad
\rho \eqdef \frac{2\gamma}{3p}.
\end{align}
With this choice, the coefficient of $\Exp{\sqn{x^{t+1}-x^t}}$ in~\eqref{eq:pl_silage_prechoice} becomes
\begin{align*}
\frac{L}{2}-\frac{1}{2\gamma}
+\nu\frac{n-p}{n^2}\delta_2^2
+\rho\left(1-\frac{p}{n}\right)\delta_2^2
&=\frac{L}{2}-\frac{1}{2\gamma}
+\frac{\gamma(n-p)}{pn}\delta_2^2.
\end{align*}

Thus, the term is nonpositive whenever
\begin{align*}
L-\frac{1}{\gamma}+2\gamma\,\delta_2^2\frac{n-p}{pn}\leq 0,
\end{align*}
that is, whenever
\begin{align*}
2\delta_2^2\frac{n-p}{pn}\gamma^2+L\gamma\leq 1.
\end{align*}
By Lemma~\ref{lem:quadratic_step_bound}, applied with
\begin{align*}
a &\eqdef 2\delta_2^2\frac{n-p}{pn},
\qquad
b \eqdef L,
\end{align*}
it therefore suffices to impose
\begin{align*}
\gamma \leq \left(L+\delta_2\sqrt{\frac{2(n-p)}{np}}\right)^{-1},
\end{align*}
which is exactly the first constraint in the stepsize condition of Theorem~\ref{theo:silage_pl}.

Hence, we can drop the $\Exp{\sqn{x^{t+1}-x^t}}$ term in~\eqref{eq:pl_silage_prechoice}.

Next, with the choice~\eqref{eq:nu_rho_choice}, we claim that if $\gamma \leq \frac{p}{2\mu n}$, then
\begin{align}
\frac{\gamma}{2}+\nu\left(1-\frac{2p}{n}\right) &\leq (1-\gamma\mu)\nu,
\label{eq:pl_coeff1}\\
\rho\left(1-\frac{p}{n}\right)+\nu\frac{p}{n^2} &\leq (1-\gamma\mu)\rho.
\label{eq:pl_coeff2}
\end{align}
Indeed, for~\eqref{eq:pl_coeff1},
\begin{align*}
\frac{\gamma}{2}+\nu\left(1-\frac{2p}{n}\right)
&=\frac{\gamma}{2}+\frac{\gamma n}{3p}-\frac{2\gamma}{3}
=\frac{\gamma n}{3p}-\frac{\gamma}{6} \notag\\
&\leq \frac{\gamma n}{3p}-\gamma^2\mu\frac{n}{3p}
=(1-\gamma\mu)\nu.
\end{align*}
For~\eqref{eq:pl_coeff2}, we have
\begin{align*}
\rho\left(1-\frac{p}{n}\right)+\nu\frac{p}{n^2}
&=\frac{2\gamma}{3p}\left(1-\frac{p}{n}\right)+\frac{\gamma n}{3p}\cdot \frac{p}{n^2}
=\frac{2\gamma}{3p}\left(1-\frac{p}{n}\right)+\frac{\gamma}{3n}\\
&=\frac{2\gamma}{3p}\left(1-\frac{p}{2n}\right)
\leq \frac{2\gamma}{3p}(1-\gamma\mu)
=(1-\gamma\mu)\rho.
\end{align*}

Substituting~\eqref{eq:nu_rho_choice},~\eqref{eq:pl_coeff1}, and~\eqref{eq:pl_coeff2} into~\eqref{eq:pl_silage_prechoice}, we obtain
\begin{align}
&\Exp{f(x^{t+1})}
+ \frac{\gamma n}{3p}\Exp{\sqn{g^{t+1}-\nabla f(x^{t+1})}}
+ \frac{2\gamma}{3p}\Exp{\frac{1}{n}\sum_{i=1}^n\sqn{g_i^{t+1}-\nabla f_i(x^{t+1})}}\notag\\
&\leq \Exp{f(x^{t})}
-\frac{\gamma}{2}\Exp{\sqn{\nabla f(x^t)}} \notag\\
&\quad + (1-\gamma\mu)\left(
\frac{\gamma n}{3p}\Exp{\sqn{g^{t}-\nabla f(x^{t})}}
+ \frac{2\gamma}{3p}\Exp{\frac{1}{n}\sum_{i=1}^n\sqn{g_i^{t}-\nabla f_i(x^{t})}}
\right).
\label{eq:pl_silage_recursion}
\end{align}

Finally, define the nonnegative sequence
\begin{align}
\label{eq:psi_pl_silage}
\Psi^t \eqdef \frac{n}{3p}\Exp{\sqn{g^{t}-\nabla f(x^{t})}}
+ \frac{2}{3p}\Exp{\frac{1}{n}\sum_{i=1}^n\sqn{g_i^{t}-\nabla f_i(x^{t})}}.
\end{align}
Then~\eqref{eq:pl_silage_recursion} reads
\begin{align*}
\Exp{f(x^{t+1})} + \gamma \Psi^{t+1}
\leq \Exp{f(x^{t})} -\frac{\gamma}{2}\Exp{\sqn{\nabla f(x^t)}} + (1-\gamma\mu)\gamma \Psi^{t}.
\end{align*}
In view of Lemma~\ref{lemma:good_recursion_pl} (with $C=0$), we conclude that for every $T\geq 0$,
\begin{align*}
\Exp{f(x^{T})-\fstar} \leq (1-\gamma\mu)^{T}\left((f(x^0)-\fstar)+\gamma\Psi^0\right).
\end{align*}
Substituting~\eqref{eq:psi_pl_silage} at $t=0$ yields~\eqref{eq:silage_pl_rate}.
\end{proof}

\newpage
\section{\algname{SILAGE}, case $n>m$}
\subsection{Implementation-efficient form and equivalence}\label{sec:equiv-n>m}

The algorithm below is an implementation-efficient form of Algorithm~\ref{algsilage2a}.
It avoids explicitly updating all table entries by storing a shared shift, while keeping the same iterates as the analysis-friendly version.

\begin{algorithm}[t]
\small
\caption{\algname{SILAGE} (\textbf{new}), case $n>m$, implementation-efficient form}
\label{algsilage2}
\begin{algorithmic}[1]
    \STATE \textbf{input:} stepsize $\gamma$, batch size $b_{\mathrm{grp}}\in[n]$
    \STATE \textbf{init:} $x^0\in\R^d$, $h_1^0,\ldots,h_n^0\in\R^d$, $q^0\eqdef 0$
    \STATE $h^0\eqdef \frac1n\sum_{i=1}^n h_i^0$
    \FOR{$t=0,1,\ldots$}
        \STATE $x^{t+1}\eqdef x^t-\gamma\rb{h^t+q^t}$
        \STATE pick $i^t\in[n]$ uniformly at random
        \STATE pick $j_{i^t}^t\in[m]$ uniformly at random
        \STATE compute $\nabla f_{i^t}(x^{t+1})$ and retain $\nabla f_{i^t,j_{i^t}^t}(x^{t+1})$
        \STATE pick $\Omega^t\subset[n]\setminus\{i^t\}$ uniformly, $|\Omega^t|=b_{\mathrm{grp}}-1$
        \STATE $\tilde{\Omega}^t\eqdef \Omega^t\cup\{i^t\}$
        \FOR{$i\in\Omega^t$}
            \STATE pick $j_i^t\in[m]$ uniformly at random
            \STATE $\Delta_i^t\eqdef \nabla f_{i,j_i^t}(x^{t+1})-\nabla f_{i,j_i^t}(x^t)$
        \ENDFOR
        \STATE $\Delta_{i^t}^t\eqdef \nabla f_{i^t,j_{i^t}^t}(x^{t+1})-\nabla f_{i^t,j_{i^t}^t}(x^t)$
        \STATE $d^t\eqdef \frac1b_{\mathrm{grp}}\sum_{i\in\tilde{\Omega}^t}\Delta_i^t$
        \STATE $q^{t+1}\eqdef q^t+d^t$
        \STATE $h^{t+1}\eqdef h^t+\frac1n\rb{\nabla f_{i^t}(x^{t+1})-h_{i^t}^t-q^t-\Delta_{i^t}^t}$
        \STATE $h_{i^t}^{t+1}\eqdef \nabla f_{i^t}(x^{t+1})-q^{t+1}$
        \FOR{$i\in\Omega^t$}
            \STATE $h_i^{t+1}\eqdef h_i^t+\Delta_i^t-d^t$
        \ENDFOR
        \STATE implicitly keep $h_i^{t+1}\eqdef h_i^t$ for all $i\in[n]\setminus\tilde{\Omega}^t$
    \ENDFOR
\end{algorithmic}
\end{algorithm}

Algorithm~\ref{algsilage2a} is the form used in the analysis, while Algorithm~\ref{algsilage2} is its implementation-efficient form.
They are linked by the reparameterization
\begin{align}
\label{eq:reparam-ghq}
g_i^t \eqdef h_i^t+q^t,\qquad i\in[n],
\qquad
g^t\eqdef \frac1n\sum_{i=1}^n g_i^t=h^t+q^t,
\end{align}
where $h^t\eqdef \frac1n\sum_{i=1}^n h_i^t$.
The equivalence lemma below justifies analyzing Algorithm~\ref{algsilage2a} while implementing Algorithm~\ref{algsilage2}.
It shows that the change of variables in~\eqref{eq:reparam-ghq} is exact, not an approximation.

\paragraph*{Coupling of the randomness.}
To compare the two procedures pathwise, we run them with the same random draws at every iteration:
the same index $i^t\in[n]$, the same subset $\Omega^t\subset[n]\setminus\{i^t\}$ of size $b_{\mathrm{grp}}-1$, and the same component indices $\{j_i^t\}_{i\in\tilde{\Omega}^t}$, where $\tilde{\Omega}^t\eqdef \Omega^t\cup\{i^t\}$.
Once these random variables are fixed, both algorithms are deterministic.

\begin{lemma}[Exact equivalence via the change of variables~\eqref{eq:reparam-ghq}]
\label{lem:equiv-alg2a-alg2}
Assume the initial conditions are consistent with~\eqref{eq:reparam-ghq}, i.e., $q^0=0$ and $g_i^0=h_i^0$ for all $i\in[n]$.
Run Algorithm~\ref{algsilage2a} and Algorithm~\ref{algsilage2} with the same random draws at every iteration.
Define the implicit variables produced by Algorithm~\ref{algsilage2} as
\begin{align}
\label{eq:implicit-g-from-hq}
\widehat g_i^t\eqdef h_i^t+q^t,\qquad
\widehat g^t\eqdef \frac1n\sum_{i=1}^n \widehat g_i^t=h^t+q^t.
\end{align}
Then, for all $t\ge 0$ and all $i\in[n]$, $\widehat g_i^t=g_i^t$.
Consequently, Algorithms~\ref{algsilage2a} and~\ref{algsilage2} generate the same iterate sequence $\{x^t\}_{t\ge 0}$.
\end{lemma}

\begin{proof}
We prove by induction on $t$ that $\widehat g_i^t=g_i^t$ for all $i\in[n]$ and that $h^t=\frac1n\sum_{i=1}^n h_i^t$.

\emph{Base case.}
Since $q^0=0$ and $g_i^0=h_i^0$ for all $i\in[n]$, we have
\begin{align*}
\widehat g_i^0
&=
h_i^0+q^0
=
g_i^0.
\end{align*}
Moreover, Algorithm~\ref{algsilage2} initializes $h^0=\frac1n\sum_{i=1}^n h_i^0$.
Thus $\widehat g^0=h^0+q^0=g^0$.

\emph{Induction step.}
Assume that, at the beginning of iteration $t$, we have $\widehat g_i^t=g_i^t$ for all $i\in[n]$ and $h^t=\frac1n\sum_{i=1}^n h_i^t$.
Then $\widehat g^t=g^t$, and Algorithm~\ref{algsilage2} updates
\begin{align*}
x^{t+1}
&=
x^t-\gamma\rb{h^t+q^t}
=
x^t-\gamma \widehat g^t
=
x^t-\gamma g^t,
\end{align*}
which is exactly the update of Algorithm~\ref{algsilage2a}.
Hence both algorithms produce the same $x^{t+1}$.

For each $i\in\tilde{\Omega}^t$, define
\begin{align*}
\Delta_i^t
&\eqdef
\nabla f_{i,j_i^t}(x^{t+1})-\nabla f_{i,j_i^t}(x^t),
&
d^t
&\eqdef
\frac1b_{\mathrm{grp}}\sum_{i\in\tilde{\Omega}^t}\Delta_i^t .
\end{align*}
Algorithm~\ref{algsilage2} sets $q^{t+1}=q^t+d^t$.
We now verify that the implicit variables $\widehat g_i^{t+1}=h_i^{t+1}+q^{t+1}$ match the table updates of Algorithm~\ref{algsilage2a}.

\smallskip
\noindent\textbf{Case 1: $i=i^t$.}
Algorithm~\ref{algsilage2} sets $h_{i^t}^{t+1}=\nabla f_{i^t}(x^{t+1})-q^{t+1}$.
Therefore,
\begin{align*}
\widehat g_{i^t}^{t+1}
&=
h_{i^t}^{t+1}+q^{t+1}
=
\nabla f_{i^t}(x^{t+1}),
\end{align*}
which coincides with the update $g_{i^t}^{t+1}=\nabla f_{i^t}(x^{t+1})$ in Algorithm~\ref{algsilage2a}.

\smallskip
\noindent\textbf{Case 2: $i\in\Omega^t$.}
Algorithm~\ref{algsilage2} sets $h_i^{t+1}=h_i^t+\Delta_i^t-d^t$.
Thus
\begin{align*}
\widehat g_i^{t+1}
&=
h_i^{t+1}+q^{t+1}
=
h_i^t+\Delta_i^t-d^t+q^t+d^t
\\
&=
\widehat g_i^t+\Delta_i^t
=
g_i^t+\Delta_i^t,
\end{align*}
where the last equality uses the induction hypothesis.
This coincides with the update of Algorithm~\ref{algsilage2a}.

\smallskip
\noindent\textbf{Case 3: $i\in[n]\setminus\tilde{\Omega}^t$.}
Algorithm~\ref{algsilage2} leaves $h_i$ unchanged.
Hence
\begin{align*}
\widehat g_i^{t+1}
&=
h_i^{t+1}+q^{t+1}
=
h_i^t+q^t+d^t
\\
&=
\widehat g_i^t+d^t
=
g_i^t+d^t,
\end{align*}
again by the induction hypothesis.
This is exactly the update of Algorithm~\ref{algsilage2a}.

Therefore $\widehat g_i^{t+1}=g_i^{t+1}$ for all $i\in[n]$.

It remains to check that the aggregate $h^{t+1}$ maintained in Algorithm~\ref{algsilage2} is equal to $\frac1n\sum_{i=1}^n h_i^{t+1}$.
Since only entries in $\tilde{\Omega}^t$ change, we have
\begin{align*}
\frac1n\sum_{i=1}^n h_i^{t+1}-h^t
&=
\frac1n\sum_{i\in\tilde{\Omega}^t}\rb{h_i^{t+1}-h_i^t}
\\
&=
\frac1n\rb{
\nabla f_{i^t}(x^{t+1})-q^{t+1}-h_{i^t}^t
+
\sum_{i\in\Omega^t}\rb{\Delta_i^t-d^t}
}
\\
&=
\frac1n\rb{
\nabla f_{i^t}(x^{t+1})-h_{i^t}^t-q^t-d^t
+
\sum_{i\in\Omega^t}\Delta_i^t
-
(b_{\mathrm{grp}}-1)d^t
}
\\
&=
\frac1n\rb{
\nabla f_{i^t}(x^{t+1})-h_{i^t}^t-q^t
+
\sum_{i\in\Omega^t}\Delta_i^t
-
bd^t
}
\\
&=
\frac1n\rb{
\nabla f_{i^t}(x^{t+1})-h_{i^t}^t-q^t-\Delta_{i^t}^t
},
\end{align*}
where the last equality uses $b_{\mathrm{grp}}d^t=\sum_{i\in\tilde{\Omega}^t}\Delta_i^t$.
This is exactly the aggregate update in Algorithm~\ref{algsilage2}.
Thus $h^{t+1}=\frac1n\sum_{i=1}^n h_i^{t+1}$, completing the induction.

Since both algorithms use the same gradient estimator, $g^t=\widehat g^t=h^t+q^t$, in the update of $x^{t+1}$, the iterate sequences coincide for all $t\ge 0$.
\end{proof}

\paragraph*{Why Algorithm~\ref{algsilage2} is algorithmically more efficient.}
Algorithm~\ref{algsilage2a} contains two dense bookkeeping operations.
First, it adds the same vector $d^t$ to all $n-b_{\mathrm{grp}}$ table entries outside $\tilde{\Omega}^t$.
Second, it recomputes the average $g^{t+1}=\frac1n\sum_{i=1}^n g_i^{t+1}$ by passing over all $n$ stored vectors.
These operations cost $\Theta(nd)$ arithmetic and memory traffic per iteration.

Algorithm~\ref{algsilage2} eliminates both dense passes.
The common update $g_i^{t+1}=g_i^t+d^t$ for all $i\notin\tilde{\Omega}^t$ is represented by the single shift update $q^{t+1}=q^t+d^t$, while the centered variables $h_i^t$ are left unchanged for those indices.
Hence only the $b_{\mathrm{grp}}$ entries in $\tilde{\Omega}^t$ are explicitly touched.

The average $h^t=\frac1n\sum_{i=1}^n h_i^t$ is also maintained incrementally.
Since only entries in $\tilde{\Omega}^t$ change, the proof of Lemma~\ref{lem:equiv-alg2a-alg2} shows that
\begin{align*}
h^{t+1}
&=
h^t+\frac1n\sum_{i\in\tilde{\Omega}^t}\rb{h_i^{t+1}-h_i^t}
\\
&=
h^t+\frac1n\rb{
\nabla f_{i^t}(x^{t+1})-h_{i^t}^t-q^t-\Delta_{i^t}^t
}.
\end{align*}
This update uses only quantities already computed at iteration $t$ and costs $\Theta(d)$ additional arithmetic after the sampled gradients have been evaluated.

Consequently, Algorithms~\ref{algsilage2a} and~\ref{algsilage2} use the same component-gradient evaluations and produce the same iterates, but Algorithm~\ref{algsilage2} avoids the $\Theta(nd)$ dense bookkeeping cost of Algorithm~\ref{algsilage2a}.
Its explicit table work scales with $\Theta(b_{\mathrm{grp}}d)$ per iteration instead of $\Theta(nd)$.

For $b_{\mathrm{grp}}<n$, the component-gradient count per iteration is $m+1+2(b_{\mathrm{grp}}-1)$ if the component $\nabla f_{i^t,j_{i^t}^t}(x^{t+1})$ is retained while computing the full group gradient $\nabla f_{i^t}(x^{t+1})$.
Indeed, computing $\nabla f_{i^t}(x^{t+1})$ costs $m$ component gradients, computing $\Delta_{i^t}^t$ requires one additional gradient at $x^t$, and computing $\Delta_i^t$ for each $i\in\Omega^t$ costs two component gradients.
When $b_{\mathrm{grp}}=m$, this is $\Theta(m)$ component-gradient evaluations per iteration.
When $b_{\mathrm{grp}}=n$, one may optionally omit the computation of $\Delta_{i^t}^t$ and the shared-shift update because $[n]\setminus\tilde{\Omega}^t=\emptyset$; this minor implementation shortcut does not affect the analysis.

\subsection{Proofs for general nonconvex functions}

The one-step analysis throughout this subsection uses the random variables generated by Algorithm~\ref{algsilage2a} at iteration $t$. For ease of auditing, we collect them here; conditioned on the filtration $\mathcal{F}^t\eqdef\sigma\rb{x^0,\{g_i^0\}_{i=1}^n,\dots,x^t,\{g_i^t\}_{i=1}^n}$, the iterate $x^{t+1}=x^t-\gamma g^t$ is deterministic, while the remaining randomness is drawn hierarchically: $i^t\sim\mathcal{U}([n])$; conditional on $i^t$, the subset $\Omega^t$ is uniform over the size-$(b_{\mathrm{grp}}-1)$ subsets of $[n]\setminus\{i^t\}$; and, conditional on $(i^t,\Omega^t)$, the component indices $\{j_i^t\}_{i\in\tilde{\Omega}^t}$ are independent and uniform on $[m]$.

\begin{center}
\small
\renewcommand{\arraystretch}{1.25}
\begin{tabular}{@{}c@{\quad}p{0.80\linewidth}@{}}
\toprule
Symbol & Definition (case $n>m$, Algorithm~\ref{algsilage2a}, iteration $t$) \\
\midrule
$i^t$ & Anchor index, $i^t\sim\mathcal{U}([n])$; the anchored estimator is reset exactly, $g_{i^t}^{t+1}=\nabla f_{i^t}(x^{t+1})$. \\
$\Omega^t$ & Recursive-update subset, sampled uniformly from $[n]\setminus\{i^t\}$ with $|\Omega^t|=b_{\mathrm{grp}}-1$. \\
$\tilde{\Omega}^t$ & Active set including the anchor, $\tilde{\Omega}^t\eqdef\Omega^t\cup\{i^t\}$, so $|\tilde{\Omega}^t|=b_{\mathrm{grp}}$. \\
$j_i^t$ & Component index for $i\in\tilde{\Omega}^t$, sampled $j_i^t\sim\mathcal{U}([m])$ (including $j_{i^t}^t$). \\
$\Delta_i^t$ & Stochastic difference for $i\in\tilde{\Omega}^t$, $\Delta_i^t\eqdef\nabla f_{i,j_i^t}(x^{t+1})-\nabla f_{i,j_i^t}(x^t)$. \\
$d^t$ & Shared drift, $d^t\eqdef\frac{1}{b_{\mathrm{grp}}}\sum_{i\in\tilde{\Omega}^t}\Delta_i^t$, added to every group outside $\tilde{\Omega}^t$. \\
$g_i^t$ & Group-$i$ gradient estimator (table entry), $\mathcal{F}^t$-measurable. \\
$g^t$ & Aggregated estimator, $g^t\eqdef\frac{1}{n}\sum_{i=1}^n g_i^t$, used in the step $x^{t+1}=x^t-\gamma g^t$. \\
\bottomrule
\end{tabular}
\end{center}

At each iteration, every entry $g_i^{t+1}$ follows exactly one of three mutually exclusive rules: \emph{anchor reset} if $i=i^t$ (so $g_i^{t+1}=\nabla f_i(x^{t+1})$), \emph{recursive update} if $i\in\Omega^t$ (so $g_i^{t+1}=g_i^t+\Delta_i^t$), and \emph{shared drift} otherwise (so $g_i^{t+1}=g_i^t+d^t$).

\subsubsection*{Key variance lemma}

As in the case $m\ge n$, the convergence proof starts with one-step estimator bounds.
The next lemma is the key technical input for Theorem~\ref{theo2}: it controls the table error and the aggregated estimator error for arbitrary batch size $b_{\mathrm{grp}}$.

\begin{lemma}\label{lem:silage_variance_bounds_ngtm}
Consider Algorithm~\ref{algsilage2a} (case $n>m$) with batch size $b_{\mathrm{grp}}\in[n]$. Let $\mathcal{F}^t$ be the filtration generated by the algorithm up to iteration $t$. Assume that the similarity conditions hold with parameters $\delta_1$ and $\delta_2$. Then the following bounds hold for the gradient estimators:

1. \textbf{Average error.}
\begin{align*}
&\Expc{\frac{1}{n}\sum_{i=1}^n \sqn{g_i^{t+1}-\nabla f_i(x^{t+1})}}\\
&\leq\left(1-\frac{1}{2n}\right)\frac{1}{n}\sum_{i=1}^n \sqn{g_i^t-\nabla f_i(x^{t})}
+\frac{2(n-b_{\mathrm{grp}})(n+1)}{n-1}\delta_1^2\sqn{x^{t+1}-x^t}
+\frac{n-2b_{\mathrm{grp}}+b_{\mathrm{grp}}^2}{b_{\mathrm{grp}}n}\delta_2^2\sqn{x^{t+1}-x^t}.
\end{align*}

2. \textbf{Global estimator error.}
\begin{align*}
&\Expc{\sqn{g^{t+1}-\nabla f(x^{t+1})}}\\
&\leq\frac{n-2}{n}\sqn{g^t-\nabla f(x^{t})}
+\frac{2}{n^3}\sum_{i=1}^n \sqn{g_i^t-\nabla f_i(x^{t})}
+\frac{2(n-b_{\mathrm{grp}})}{b_{\mathrm{grp}}(n-1)}\delta_1^2\sqn{x^{t+1}-x^t}
+\frac{n^2-2n+b_{\mathrm{grp}}}{b_{\mathrm{grp}}n^2}\delta_2^2\sqn{x^{t+1}-x^t},
\end{align*}
where $g^t\eqdef \frac{1}{n}\sum_{i=1}^n g_i^t$.
\end{lemma}
\begin{proof}

Let $\mathcal{F}^t \eqdef \sigma\rb{x^0, \{g_i^0\}_{i=1}^n, \dots, x^t, \{g_i^t\}_{i=1}^n}$ be the $\sigma$-algebra generated by the iterates and the individual gradient estimators (table entries) up to iteration $t$; the averaged estimator $g^t = \frac{1}{n}\sum_{i=1}^n g_i^t$ is therefore $\mathcal{F}^t$-measurable. Conditional on $\mathcal{F}^t$, the update
\begin{align*}
x^{t+1}=x^t-\gamma g^t
\end{align*}
is deterministic. On the other hand, the construction of the table $(g_i^{t+1})_{i=1}^n$ depends on the randomness generated at iteration $t$, drawn hierarchically: the index $i^t\sim\mathcal{U}([n])$; conditional on $i^t$, the subset $\Omega^t$ uniform over the size-$(b_{\mathrm{grp}}-1)$ subsets of $[n]\setminus\{i^t\}$; and, conditional on $(i^t,\Omega^t)$, the indices $j_i^t\sim\mathcal{U}([m])$ for $i\in\tilde{\Omega}^t\eqdef \Omega^t\cup\{i^t\}$ (including $j_{i^t}^t$), independent across $i$. Recall that $d^t$ is defined by
\begin{align*}
d^t\eqdef \frac{1}{b_{\mathrm{grp}}}\sum_{i\in \tilde{\Omega}^t}\Big(\nabla f_{i,j_i^t}(x^{t+1})-\nabla f_{i,j_i^t}(x^{t})\Big).
\end{align*}

Fix an index $i\in[n]$. To analyze the conditional error $\Expc{\sqn{g_i^{t+1}-\nabla f_i(x^{t+1})}}$, we apply the \textbf{law of total expectation} by conditioning on which update rule is applied to the $i$-th table entry at iteration $t$. From Algorithm~\ref{algsilage2a}, there are three mutually exclusive cases:

\begin{enumerate}
\item \textbf{Anchor reset ($i=i^t$).} This event has probability $\Prob(i=i^t)=\frac{1}{n}$. In this case, $g_i^{t+1}=\nabla f_i(x^{t+1})$, hence $\sqn{g_i^{t+1}-\nabla f_i(x^{t+1})}=0$.

\item \textbf{Recursive update ($i\in\Omega^t$).} Since $\Omega^t$ is sampled uniformly from $[n]\setminus\{i^t\}$, we have
\begin{align*}
\Prob(i\in\Omega^t)=\Prob(i^t\neq i)\Prob(i\in\Omega^t\mid i^t\neq i)=\frac{n-1}{n}\cdot \frac{b_{\mathrm{grp}}-1}{n-1}=\frac{b_{\mathrm{grp}}-1}{n}.
\end{align*}
In this case,
\begin{align*}
g_i^{t+1}=g_i^t+\nabla f_{i,j_i^t}(x^{t+1})-\nabla f_{i,j_i^t}(x^{t}),
\end{align*}
and by the same variance decomposition argument as in the case $m\geq n$ (using that $j_i^t$ is uniform on $[m]$ so the centered increment has zero conditional mean), we obtain
\begin{align*}
&\Expb{\sqn{g_i^{t+1}-\nabla f_i(x^{t+1})}}{\mathcal{F}^t,\, i\in\Omega^t}\\
&=\sqn{g_i^t-\nabla f_i(x^{t})}+\frac{1}{m}\sum_{j=1}^m \sqn{\nabla f_{i,j}(x^{t+1})-\nabla f_i(x^{t+1})-\nabla f_{i,j}(x^{t})+\nabla f_i(x^{t})}.
\end{align*}

\item \textbf{Drift update ($i\notin\tilde{\Omega}^t$).} This event has probability
\begin{align*}
\Prob(i\notin\tilde{\Omega}^t)=1-\Prob(i=i^t)-\Prob(i\in\Omega^t)=1-\frac{1}{n}-\frac{b_{\mathrm{grp}}-1}{n}=\frac{n-b_{\mathrm{grp}}}{n}.
\end{align*}
In this case, $g_i^{t+1}=g_i^t+d^t$, and the conditional error is $\sqn{g_i^{t}+d^t-\nabla f_i(x^{t+1})}$.
\end{enumerate}

Combining these cases via the law of total expectation and omitting the anchor case (which contributes $\frac{1}{n}\cdot 0$), we obtain
\begin{align}
&\Expc{\sqn{g_i^{t+1}-\nabla f_i(x^{t+1})}} \notag\\
&=\frac{b_{\mathrm{grp}}-1}{n} \Expb{\sqn{g_i^{t+1}-\nabla f_i(x^{t+1})}}{\mathcal{F}^t, i\in\Omega^t}\notag\\
&\quad +\frac{n-b_{\mathrm{grp}}}{n} \Expb{\sqn{g_i^{t+1}-\nabla f_i(x^{t+1})}}{\mathcal{F}^t, i\notin\tilde{\Omega}^t}\notag\\
&=\frac{b_{\mathrm{grp}}-1}{n} \sqn{g_i^t - \nabla f_i(x^{t})}\notag\\
&\quad+\frac{b_{\mathrm{grp}}-1}{nm}\sum_{j=1}^m \sqn{\nabla f_{i,j}(x^{t+1})-\nabla f_i(x^{t+1})-\nabla f_{i,j}(x^{t})+\nabla f_i(x^{t})}\notag\\
&\quad+\frac{n-b_{\mathrm{grp}}}{n} \Expb{\sqn{g_i^{t}+d^t-\nabla f_i(x^{t+1})}}{\mathcal{F}^t, i\notin\tilde{\Omega}^t}.
\label{eqpp1}
\end{align}

We now focus on the last conditional expectation in~\eqref{eqpp1}. If $b_{\mathrm{grp}}=n$, then $\tilde{\Omega}^t=[n]$ and the event $\{i\notin\tilde{\Omega}^t\}$ is empty, hence this term is zero. Thus, in what follows we assume $b_{\mathrm{grp}}\leq n-1$. Using the definition of $d^t$, we have
\begin{align}
&\Expb{\sqn{g_i^{t}+d^t-\nabla f_i(x^{t+1})}}{\mathcal{F}^t, i\notin\tilde{\Omega}^t}\notag\\
&=\Expb{\sqn{g_i^{t}-\nabla f_i(x^{t+1})+\frac{1}{b_{\mathrm{grp}}}\sum_{i'\in \tilde{\Omega}^t}\Big(\nabla f_{i',j_{i'}^t}(x^{t+1})-\nabla f_{i',j_{i'}^t}(x^{t})\Big)}}{\mathcal{F}^t, i\notin\tilde{\Omega}^t}\notag\\
&=\Expb{\sqn{g_i^{t}-\nabla f_i(x^{t+1})+\frac{1}{b_{\mathrm{grp}}}\sum_{i'\in \tilde{\Omega}^t}\Big(\nabla f_{i'}(x^{t+1})-\nabla f_{i'}(x^{t})\Big)}}{\mathcal{F}^t, i\notin\tilde{\Omega}^t}\notag\\
&\quad +\frac{1}{b_{\mathrm{grp}}^2} \Expb{\sqn{\sum_{i'\in \tilde{\Omega}^t} \Big(\nabla f_{i',j_{i'}^t}(x^{t+1})- \nabla f_{i'}(x^{t+1})-\nabla f_{i',j_{i'}^t}(x^{t})+\nabla f_{i'}(x^{t})\Big)}}{\mathcal{F}^t, i\notin\tilde{\Omega}^t}.
\label{eqp1_firstsplit}
\end{align}
In~\eqref{eqp1_firstsplit}, we decomposed each stochastic difference into its mean part
$\nabla f_{i'}(x^{t+1})-\nabla f_{i'}(x^{t})$ and its centered noise part
$\nabla f_{i',j_{i'}^t}(x^{t+1})-\nabla f_{i'}(x^{t+1})-\nabla f_{i',j_{i'}^t}(x^{t})+\nabla f_{i'}(x^{t})$,
and used that the latter has zero conditional mean given $\mathcal{F}^t$. Next, define the population mean drift (excluding $i$)
\begin{align*}
a_i^{t}\eqdef \frac{1}{n-1}\sum_{i'\in [n]\backslash\{i\}} \left(\nabla f_{i'}(x^{t+1})-\nabla f_{i'}(x^{t})\right).
\end{align*}




We justify that, despite being generated in two stages, the set $\tilde{\Omega}^t$ is uniformly distributed.

Recall the sampling procedure in Algorithm~\ref{algsilage2a}: first, we draw $i^t\sim\mathcal{U}([n])$. Then, conditional on $i^t$, we draw $\Omega^t$ uniformly at random among all subsets of $[n]\setminus\{i^t\}$ of cardinality $b_{\mathrm{grp}}-1$, and we set
\begin{align*}
\tilde{\Omega}^t \eqdef \Omega^t \cup \{i^t\},
\qquad |\tilde{\Omega}^t|=b_{\mathrm{grp}}.
\end{align*}
Although this is a two-stage procedure, the resulting set $\tilde{\Omega}^t$ is uniform over all subsets of $[n]$ of size $b_{\mathrm{grp}}$.

Indeed, fix an arbitrary subset $S\subset[n]$ with $|S|=b_{\mathrm{grp}}$. The event $\{\tilde{\Omega}^t=S\}$ occurs if and only if the anchor index satisfies $i^t=s$ for some $s\in S$ and, simultaneously, the second-stage sample satisfies $\Omega^t=S\setminus\{s\}$. These events are mutually exclusive over different choices of $s\in S$, hence by the law of total probability,
\begin{align*}
\Prob(\tilde{\Omega}^t=S)
&=\sum_{s\in S}\Prob\big(i^t=s,\ \Omega^t=S\setminus\{s\}\big)\\
&=\sum_{s\in S}\Prob(i^t=s)\,\Prob\big(\Omega^t=S\setminus\{s\}\mid i^t=s\big).
\end{align*}
Since $i^t$ is uniform on $[n]$, we have $\Prob(i^t=s)=\frac{1}{n}$. Moreover, conditional on $i^t=s$, the set $\Omega^t$ is uniform among all $\binom{n-1}{b_{\mathrm{grp}}-1}$ subsets of $[n]\setminus\{s\}$ of size $b_{\mathrm{grp}}-1$, and $S\setminus\{s\}$ is one such subset. Therefore,
\begin{align*}
\Prob\big(\Omega^t=S\setminus\{s\}\mid i^t=s\big)=\frac{1}{\binom{n-1}{b_{\mathrm{grp}}-1}}.
\end{align*}
Substituting these probabilities back gives
\begin{align*}
\Prob(\tilde{\Omega}^t=S)
&=\sum_{s\in S}\frac{1}{n}\cdot\frac{1}{\binom{n-1}{b_{\mathrm{grp}}-1}}
=\frac{|S|}{n\binom{n-1}{b_{\mathrm{grp}}-1}}
=\frac{b_{\mathrm{grp}}}{n\binom{n-1}{b_{\mathrm{grp}}-1}}.
\end{align*}
This expression does not depend on the particular choice of $S$, hence $\tilde{\Omega}^t$ is uniformly distributed over all subsets of $[n]$ of size $b_{\mathrm{grp}}$, i.e., over $\binom{[n]}{b_{\mathrm{grp}}}$.

We now derive the conditional uniformity used in the proof. Fix $i\in[n]$. Since $\tilde{\Omega}^t$ is uniform among all size-$b_{\mathrm{grp}}$ subsets of $[n]$, we have
\begin{align*}
\Prob(i\notin\tilde{\Omega}^t)
=\frac{\#\{S\subset[n]\setminus\{i\}:\ |S|=b_{\mathrm{grp}}\}}{\#\{S\subset[n]:\ |S|=b_{\mathrm{grp}}\}}
=\frac{\binom{n-1}{b_{\mathrm{grp}}}}{\binom{n}{b_{\mathrm{grp}}}}
=\frac{n-b_{\mathrm{grp}}}{n}.
\end{align*}
Now take any subset $S\subset[n]\setminus\{i\}$ with $|S|=b_{\mathrm{grp}}$. Then $\{ \tilde{\Omega}^t=S\}\subseteq \{i\notin\tilde{\Omega}^t\}$, and thus
\begin{align*}
\Prob(\tilde{\Omega}^t=S \mid i\notin\tilde{\Omega}^t)
&=\frac{\Prob(\tilde{\Omega}^t=S)}{\Prob(i\notin\tilde{\Omega}^t)}\\
&=\frac{\frac{b_{\mathrm{grp}}}{n\binom{n-1}{b_{\mathrm{grp}}-1}}}{\frac{n-b_{\mathrm{grp}}}{n}}
=\frac{b_{\mathrm{grp}}}{(n-b_{\mathrm{grp}})\binom{n-1}{b_{\mathrm{grp}}-1}}
=\frac{1}{\binom{n-1}{b_{\mathrm{grp}}}},
\end{align*}
where in the last step we used the identity $\binom{n-1}{b_{\mathrm{grp}}}=\frac{n-b_{\mathrm{grp}}}{b_{\mathrm{grp}}}\binom{n-1}{b_{\mathrm{grp}}-1}$.
Hence the conditional probability above is constant over all $S\subset[n]\setminus\{i\}$ with $|S|=b_{\mathrm{grp}}$. Therefore, conditioned on the event $\{i\notin\tilde{\Omega}^t\}$, the set $\tilde{\Omega}^t$ is uniformly distributed over all subsets of $[n]\setminus\{i\}$ of size $b_{\mathrm{grp}}$.

Finally, conditioning on $\mathcal{F}^t$ does not affect this conclusion because $\tilde{\Omega}^t$ is sampled at iteration $t$ independently of $\mathcal{F}^t$. Consequently, conditioned on $\{\mathcal{F}^t,\, i\notin\tilde{\Omega}^t\}$, the set $\tilde{\Omega}^t$ is a uniformly random subset of $[n]\setminus\{i\}$ of size $b_{\mathrm{grp}}$.

\begin{align*}
\Expb{\frac{1}{b_{\mathrm{grp}}}\sum_{i'\in \tilde{\Omega}^t}\left(\nabla f_{i'}(x^{t+1})-\nabla f_{i'}(x^{t})\right)}{\mathcal{F}^t, i\notin\tilde{\Omega}^t}
=a_i^t.
\end{align*}
Therefore, applying the same variance decomposition as before (cross term vanishes because the sample mean is unbiased for $a_i^t$ under the conditioning), we obtain
\begin{align}
&\Expb{\sqn{g_i^{t}-\nabla f_i(x^{t+1})+\frac{1}{b_{\mathrm{grp}}}\sum_{i'\in \tilde{\Omega}^t}\left(\nabla f_{i'}(x^{t+1})-\nabla f_{i'}(x^{t})\right)}}{\mathcal{F}^t, i\notin\tilde{\Omega}^t}\notag\\
&=\sqn{g_i^{t}-\nabla f_i(x^{t+1})+a_i^t} \notag\\
&\quad +\Expb{\sqn{\frac{1}{b_{\mathrm{grp}}}\sum_{i'\in \tilde{\Omega}^t}\left( \nabla f_{i'}(x^{t+1})-\nabla f_{i'}(x^{t})\right)-a_i^t}}{\mathcal{F}^t, i\notin\tilde{\Omega}^t}.
\label{eqp1_secondsplit}
\end{align}
Finally, in the noise term of~\eqref{eqp1_firstsplit}, we use the independence of the indices $(j_{i'}^t)_{i'\in\tilde{\Omega}^t}$ (conditional on $\mathcal{F}^t$ and on the sampled set $\tilde{\Omega}^t$) to replace the squared norm of the sum by the sum of squared norms (cross terms vanish because each centered noise has zero conditional mean). Putting~\eqref{eqp1_firstsplit} and~\eqref{eqp1_secondsplit} together yields
\begin{align}
&\Expb{\sqn{g_i^{t}+d^t-\nabla f_i(x^{t+1})}}{\mathcal{F}^t, i\notin\tilde{\Omega}^t}\notag\\
&=\sqn{g_i^{t}-\nabla f_i(x^{t+1})+a_i^t}\notag\\
&\quad+\Expb{\sqn{\frac{1}{b_{\mathrm{grp}}}\sum_{i'\in \tilde{\Omega}^t}\left( \nabla f_{i'}(x^{t+1})-\nabla f_{i'}(x^{t})\right)-a_i^t}}{\mathcal{F}^t, i\notin\tilde{\Omega}^t}\notag\\
&\quad +\frac{1}{b_{\mathrm{grp}}^2} \Expb{\sum_{i'\in \tilde{\Omega}^t} \sqn{\nabla f_{i',j_{i'}^t}(x^{t+1})- \nabla f_{i'}(x^{t+1})-\nabla f_{i',j_{i'}^t}(x^{t})+\nabla f_{i'}(x^{t})}}{\mathcal{F}^t, i\notin\tilde{\Omega}^t},
\label{eqp1}
\end{align}
which is the desired three-term expansion. In~\eqref{eqp1}, the second term captures the sampling variance of the mean drift over the random set $\tilde{\Omega}^t$, while the third term captures the variance due to the inner sampling $(j_{i'}^t)$.

Let us further expand the three terms in~\eqref{eqp1}. Throughout, we work under the conditioning
$\{\mathcal{F}^t,\, i\notin\tilde{\Omega}^t\}$, so that $\tilde{\Omega}^t$ is a uniformly random subset of $[n]\setminus\{i\}$ of size $b_{\mathrm{grp}}$ (as shown above), and all quantities depending on $(x^t,x^{t+1},(g_i^t)_{i=1}^n)$ are $\mathcal{F}^t$-measurable.

\textbf{First term.}
Recall that
\begin{align*}
a_i^{t}\eqdef \frac{1}{n-1}\sum_{i'\in [n]\backslash\{i\}} \left(\nabla f_{i'}(x^{t+1})-\nabla f_{i'}(x^{t})\right).
\end{align*}
We rewrite $a_i^t$ by adding and subtracting the missing $i$-th component inside the sum:
\begin{align*}
a_i^t
&=\frac{1}{n-1}\sum_{i'=1}^n\left(\nabla f_{i'}(x^{t+1})-\nabla f_{i'}(x^{t})\right)
-\frac{1}{n-1}\left(\nabla f_{i}(x^{t+1})-\nabla f_{i}(x^{t})\right).
\end{align*}
Using $\sum_{i'=1}^n \nabla f_{i'}(x)=n\nabla f(x)$, we obtain
\begin{align*}
a_i^t
&=\frac{n}{n-1}\left(\nabla f(x^{t+1})-\nabla f(x^{t})\right)
-\frac{1}{n-1}\left(\nabla f_{i}(x^{t+1})-\nabla f_{i}(x^{t})\right).
\end{align*}
Substituting this expression into the first term of~\eqref{eqp1} yields
\begin{align*}
\sqn{g_i^{t}-\nabla f_i(x^{t+1})+a_i^t}
&=\left\| g_i^{t}-\nabla f_i(x^{t+1})
+\frac{n}{n-1}\left(\nabla f(x^{t+1})-\nabla f(x^{t})\right)\right.\\
&\quad\left.{}-\frac{1}{n-1}\left(\nabla f_{i}(x^{t+1})-\nabla f_{i}(x^{t})\right)\right\|^2.
\end{align*}
Finally, we regroup the terms by adding and subtracting $\nabla f_i(x^t)$:
\begin{align*}
&g_i^{t}-\nabla f_i(x^{t+1})
-\frac{1}{n-1}\left(\nabla f_{i}(x^{t+1})-\nabla f_{i}(x^{t})\right)\\
&\qquad = g_i^{t}-\nabla f_i(x^{t})
-\frac{n}{n-1}\left(\nabla f_{i}(x^{t+1})-\nabla f_{i}(x^{t})\right),
\end{align*}
and therefore
\begin{align*}
\sqn{g_i^{t}-\nabla f_i(x^{t+1})+a_i^t}
&=\left\| g_i^{t}-\nabla f_i(x^{t})
-\frac{n}{n-1}\left(\nabla f_{i}(x^{t+1})-\nabla f_{i}(x^{t})\right.\right.\\
&\quad\left.\left.{}-\nabla f(x^{t+1})+\nabla f(x^{t})\right)\right\|^2.
\end{align*}

\textbf{Second term.}
We now bound the sampling variance of the mean drift over the random set $\tilde{\Omega}^t$.
Assume first that $n\geq 3$ (so that $n-1\geq 2$). Under $\{i\notin\tilde{\Omega}^t\}$, the set
$\tilde{\Omega}^t$ is a uniformly random subset of $[n]\setminus\{i\}$ of size $b_{\mathrm{grp}}$.
Apply Lemma~\ref{lem:prop1} with
\begin{align*}
N\eqdef n-1,\qquad B\eqdef b_{\mathrm{grp}},\qquad
v_{i'}\eqdef \nabla f_{i'}(x^{t+1})-\nabla f_{i'}(x^{t})\quad (i'\in[n]\setminus\{i\}),
\qquad v\eqdef a_i^t.
\end{align*}
Then we obtain
\begin{align*}
&\Expb{\sqn{\frac{1}{b_{\mathrm{grp}}}\sum_{i'\in \tilde{\Omega}^t}\left( \nabla f_{i'}(x^{t+1})-\nabla f_{i'}(x^{t})\right)-a_i^t}}{\mathcal{F}^t, i\notin\tilde{\Omega}^t}\\
&=\frac{(n-1)-b_{\mathrm{grp}}}{b_{\mathrm{grp}}(n-1)\big((n-1)-1\big)}\sum_{i'\in [n]\backslash\{i\}}
\sqn{\left(\nabla f_{i'}(x^{t+1})-\nabla f_{i'}(x^{t})\right)-a_i^t}\\
&=\frac{n-1-b_{\mathrm{grp}}}{b_{\mathrm{grp}}(n-2)}\frac{1}{n-1}\sum_{i'\in [n]\backslash\{i\}}
\sqn{\nabla f_{i'}(x^{t+1})-\nabla f_{i'}(x^{t})-a_i^t}.
\end{align*}
If $n=2$, then necessarily $b_{\mathrm{grp}}=1$ (since $b_{\mathrm{grp}}\leq n-1$ under $\{i\notin\tilde{\Omega}^t\}$), and in this case the second term equals zero (the sample equals the population, hence the variance vanishes).

\textbf{Third term.}
We now compute the expectation of the averaged inner-noise term. Define, for each $i'\in[n]$,
\begin{align*}
\widetilde{\Delta}_{i'}^t(j)\eqdef \nabla f_{i',j}(x^{t+1})-\nabla f_{i'}(x^{t+1})-\nabla f_{i',j}(x^{t})+\nabla f_{i'}(x^{t}).
\end{align*}
Conditioned on $\{\mathcal{F}^t,\, i\notin\tilde{\Omega}^t\}$, the set $\tilde{\Omega}^t$ is a uniformly random subset of $[n]\setminus\{i\}$ of size $b_{\mathrm{grp}}$, and for each $i'\in\tilde{\Omega}^t$ the index $j_{i'}^t$ is uniform on $[m]$ and independent of the set selection. Hence,
\begin{align*}
&\Expb{\frac{1}{b_{\mathrm{grp}}}\sum_{i'\in \tilde{\Omega}^t} \sqn{\widetilde{\Delta}_{i'}^t( j_{i'}^t)}}{\mathcal{F}^t, i\notin\tilde{\Omega}^t} \\
&=\Expb{\frac{1}{b_{\mathrm{grp}}}\sum_{i'\in \tilde{\Omega}^t}\Expb{\sqn{\widetilde{\Delta}_{i'}^t(j_{i'}^t)}}{\mathcal{F}^t,\tilde{\Omega}^t}}{\mathcal{F}^t, i\notin\tilde{\Omega}^t}\\
&=\Expb{\frac{1}{b_{\mathrm{grp}}}\sum_{i'\in \tilde{\Omega}^t}\left(\frac{1}{m}\sum_{j=1}^m \sqn{\widetilde{\Delta}_{i'}^t(j)}\right)}{\mathcal{F}^t, i\notin\tilde{\Omega}^t} \\
&=\frac{1}{n-1}\sum_{i'\in [n]\backslash\{i\}}\frac{1}{m}\sum_{j=1}^m \sqn{\widetilde{\Delta}_{i'}^t(j)}\\
&=\frac{1}{(n-1)m}\sum_{i'=1}^n\sum_{j=1}^m \sqn{\widetilde{\Delta}_{i'}^t(j)}
-\frac{1}{(n-1)m}\sum_{j=1}^m \sqn{\widetilde{\Delta}_{i}^t(j)}\\
&=\frac{1}{(n-1)m}\sum_{i'=1}^n\sum_{j=1}^m
\sqn{\nabla f_{i',j}(x^{t+1})- \nabla f_{i'}(x^{t+1})-\nabla f_{i',j}(x^{t})+\nabla f_{i'}(x^{t})}\\
&\quad-\frac{1}{(n-1)m}\sum_{j=1}^m
\sqn{\nabla f_{i,j}(x^{t+1})- \nabla f_{i}(x^{t+1})-\nabla f_{i,j}(x^{t})+\nabla f_{i}(x^{t})}.
\end{align*}

\textbf{Summing~\eqref{eqpp1} over $i$.}
Hence, taking the sum over $i=1,\ldots,n$ in~\eqref{eqpp1}, we obtain
\begin{align*}
&\Expc{\sum_{i=1}^n \sqn{g_i^{t+1}-\nabla f_i(x^{t+1})}}\\
&=\frac{b_{\mathrm{grp}}-1}{n} \sum_{i=1}^n\sqn{g_i^t - \nabla f_i(x^{t})}\\
&\quad+\frac{b_{\mathrm{grp}}-1}{nm}\sum_{i=1}^n\sum_{j=1}^m
\sqn{\nabla f_{i,j}(x^{t+1})-\nabla f_i(x^{t+1})-\nabla f_{i,j}(x^{t})+\nabla f_i(x^{t})}\\
&\quad+\frac{n-b_{\mathrm{grp}}}{n}\sum_{i=1}^n \Expb{\sqn{g_i^{t}+d^t-\nabla f_i(x^{t+1})}}{\mathcal{F}^t, i\notin\tilde{\Omega}^t},
\end{align*}
with
\begin{align}
&\sum_{i=1}^n\Expb{\sqn{g_i^{t}+d^t-\nabla f_i(x^{t+1})}}{\mathcal{F}^t, i\notin\tilde{\Omega}^t}\notag\\
&=\sum_{i=1}^n\left\| g_i^{t}-\nabla f_i(x^{t})- \frac{n}{n-1}\left(\nabla f_{i}(x^{t+1})-\nabla f_{i}(x^{t})-\nabla f(x^{t+1})+\nabla f(x^{t}) \right) \right\|^2\notag\\
&\quad+\frac{n-1-b_{\mathrm{grp}}}{b_{\mathrm{grp}}(n-2)}\frac{1}{n-1}\sum_{i=1}^n\sum_{i'\in [n]\backslash\{i\}}
\sqn{\nabla f_{i'}(x^{t+1})-\nabla f_{i'}(x^{t})-a_i^t}\label{eqp2}\\
&\quad+\frac{1}{b_{\mathrm{grp}}m}\sum_{i=1}^n\sum_{j=1}^m
\sqn{\nabla f_{i,j}(x^{t+1})- \nabla f_{i}(x^{t+1})-\nabla f_{i,j}(x^{t})+\nabla f_{i}(x^{t})},\notag
\end{align}
with the term~\eqref{eqp2} replaced by zero if $n=2$.

We now justify how the last term arises from summing the third term in~\eqref{eqp1} over $i$.
For $i'\in[n]$ and $j\in[m]$, recall
\begin{align*}
\widetilde{\Delta}_{i'}^t(j)\eqdef \nabla f_{i',j}(x^{t+1})-\nabla f_{i'}(x^{t+1})-\nabla f_{i',j}(x^{t})+\nabla f_{i'}(x^{t}).
\end{align*}
Let
\begin{align*}
S \eqdef \sum_{i'=1}^n\sum_{j=1}^m \sqn{\widetilde{\Delta}_{i'}^t(j)}.
\end{align*}
Then
\begin{align*}
&\frac{1}{b_{\mathrm{grp}}}\sum_{i=1}^n \left(
\frac{1}{(n-1)m}\sum_{i'=1}^n\sum_{j=1}^m \sqn{\widetilde{\Delta}_{i'}^t(j)}
-\frac{1}{(n-1)m}\sum_{j=1}^m \sqn{\widetilde{\Delta}_{i}^t(j)}
\right)\\
&\eqtext{(a)}
\frac{1}{b_{\mathrm{grp}}}\left(
\sum_{i=1}^n \frac{1}{(n-1)m}\sum_{i'=1}^n\sum_{j=1}^m \sqn{\widetilde{\Delta}_{i'}^t(j)}
-\frac{1}{(n-1)m}\sum_{i=1}^n\sum_{j=1}^m \sqn{\widetilde{\Delta}_{i}^t(j)}
\right)\\
&\eqtext{(b)}
\frac{1}{b_{\mathrm{grp}}}\left(
\frac{n}{(n-1)m}\sum_{i'=1}^n\sum_{j=1}^m \sqn{\widetilde{\Delta}_{i'}^t(j)}
-\frac{1}{(n-1)m}\sum_{i=1}^n\sum_{j=1}^m \sqn{\widetilde{\Delta}_{i}^t(j)}
\right)\\
&\eqtext{(c)}
\frac{1}{b_{\mathrm{grp}}}\left(\frac{n}{(n-1)m}S-\frac{1}{(n-1)m}S\right)
=\frac{1}{b_{\mathrm{grp}}}\cdot \frac{1}{m}S\\
&\eqtext{(d)}
\frac{1}{b_{\mathrm{grp}}m}\sum_{i=1}^n\sum_{j=1}^m \sqn{\widetilde{\Delta}_{i}^t(j)}\\
&\eqtext{(e)}
\frac{1}{b_{\mathrm{grp}}m}\sum_{i=1}^n\sum_{j=1}^m
\sqn{\nabla f_{i,j}(x^{t+1})- \nabla f_{i}(x^{t+1})-\nabla f_{i,j}(x^{t})+\nabla f_{i}(x^{t})}.
\end{align*}
Here, (a) is linearity of summation, (b) uses $\sum_{i=1}^n 1=n$, (c) is the definition of $S$,
(d) simplifies $\frac{n-1}{(n-1)m}=\frac{1}{m}$, and (e) expands $\widetilde{\Delta}_i^t(j)$.

Moreover, assume that $n\geq 3$. We further simplify the double-sum term appearing in~\eqref{eqp2}.
Recall that
\begin{align*}
a_i^{t}\eqdef \frac{1}{n-1}\sum_{i'\in [n]\backslash\{i\}} \left(\nabla f_{i'}(x^{t+1})-\nabla f_{i'}(x^{t})\right).
\end{align*}
For each fixed $i\in[n]$, since $a_i^t$ is the average of the $(n-1)$ vectors
$\{\nabla f_{i'}(x^{t+1})-\nabla f_{i'}(x^{t})\}_{i'\neq i}$, we have
\begin{align*}
\sum_{i'\in [n]\backslash\{i\}}\left(\nabla f_{i'}(x^{t+1})-\nabla f_{i'}(x^{t})-a_i^t\right)=0.
\end{align*}
Expanding the square and using the previous identity to cancel the cross term yields
\begin{align*}
&\frac{1}{n-1}\sum_{i'\in [n]\backslash\{i\}} \sqn{\nabla f_{i'}(x^{t+1})-\nabla f_{i'}(x^{t})-a_i^t}\\
&=\frac{1}{n-1}\sum_{i'\in [n]\backslash\{i\}} \sqn{\nabla f_{i'}(x^{t+1})-\nabla f_{i'}(x^{t})}
-\sqn{a_i^t}.
\end{align*}
Summing this identity over $i=1,\ldots,n$ gives
\begin{align*}
&\frac{1}{n-1}\sum_{i=1}^n\sum_{i'\in [n]\backslash\{i\}}
\sqn{\nabla f_{i'}(x^{t+1})-\nabla f_{i'}(x^{t})-a_i^t}\\
&=\frac{1}{n-1}\sum_{i=1}^n\sum_{i'\in [n]\backslash\{i\}}
\sqn{\nabla f_{i'}(x^{t+1})-\nabla f_{i'}(x^{t})}
-\sum_{i=1}^n \sqn{a_i^t}.
\end{align*}
In the first term on the right-hand side, each index $i'\in[n]$ appears exactly $n-1$ times in the inner sum
(since it is excluded only when $i=i'$). Therefore,
\begin{align*}
\frac{1}{n-1}\sum_{i=1}^n\sum_{i'\in [n]\backslash\{i\}}
\sqn{\nabla f_{i'}(x^{t+1})-\nabla f_{i'}(x^{t})}
=\sum_{i=1}^n \sqn{\nabla f_{i}(x^{t+1})-\nabla f_{i}(x^{t})},
\end{align*}
and hence
\begin{align}
&\frac{1}{n-1}\sum_{i=1}^n\sum_{i'\in [n]\backslash\{i\}}
\sqn{\nabla f_{i'}(x^{t+1})-\nabla f_{i'}(x^{t})-a_i^t}\notag\\
&=\sum_{i=1}^n \sqn{\nabla f_{i}(x^{t+1})-\nabla f_{i}(x^{t})}
-\sum_{i=1}^n \sqn{a_i^t}.
\label{eq:prob_reduce_1}
\end{align}

Next, define the full-average drift
\begin{align*}
a^{t}\eqdef \nabla f(x^{t+1})-\nabla f(x^{t})
=\frac{1}{n}\sum_{i=1}^n \left(\nabla f_i(x^{t+1})-\nabla f_i(x^{t})\right).
\end{align*}
We first use the standard identity
\begin{align*}
\frac{1}{n}\sum_{i=1}^n \sqn{a_i^t}
=\sqn{a^t}+\frac{1}{n}\sum_{i=1}^n \sqn{a_i^t-a^t},
\end{align*}
which holds because $a^t=\frac{1}{n}\sum_{i=1}^n a_i^t$.
To compute the second term, we invoke Lemma~\ref{lem:prop1} with $N=n$, $B=n-1$ and vectors
\begin{align*}
v_i \eqdef \nabla f_i(x^{t+1})-\nabla f_i(x^{t})\qquad (i\in[n]),
\end{align*}
so that $v=a^t$. For each $i\in[n]$, the quantity $a_i^t$ is the average of the $n-1$ vectors $\{v_{i'}\}_{i'\neq i}$,
and the family of subsets $[n]\backslash\{i\}$ is uniformly distributed over all $(n-1)$-subsets of $[n]$.
Therefore Lemma~\ref{lem:prop1} yields
\begin{align*}
\frac{1}{n}\sum_{i=1}^n \sqn{a_i^t-a^t}
=\frac{1}{n(n-1)^2}\sum_{i=1}^n \sqn{\nabla f_{i}(x^{t+1})-\nabla f_{i}(x^{t})-a^t}.
\end{align*}
Combining the last two displays gives
\begin{align}
\frac{1}{n}\sum_{i=1}^n \sqn{a_i^t}
=\sqn{a^t}+\frac{1}{n(n-1)^2}\sum_{i=1}^n \sqn{\nabla f_{i}(x^{t+1})-\nabla f_{i}(x^{t})-a^t}.
\label{eq:avg_ai_sq}
\end{align}
Multiplying~\eqref{eq:avg_ai_sq} by $n$ and substituting into~\eqref{eq:prob_reduce_1}, we obtain
\begin{align*}
&\frac{1}{n-1}\sum_{i=1}^n\sum_{i'\in [n]\backslash\{i\}}
\sqn{\nabla f_{i'}(x^{t+1})-\nabla f_{i'}(x^{t})-a_i^t}\\
&=\sum_{i=1}^n\sqn{\nabla f_{i}(x^{t+1})-\nabla f_{i}(x^{t})}-n\sqn{a^t} \\
&\quad-\frac{1}{(n-1)^2}\sum_{i=1}^n \sqn{\nabla f_{i}(x^{t+1})-\nabla f_{i}(x^{t})-a^t}.
\end{align*}
Finally, since $a^t$ is the average of $\nabla f_i(x^{t+1})-\nabla f_i(x^{t})$, we have the identity
\begin{align*}
&\sum_{i=1}^n\sqn{\nabla f_{i}(x^{t+1})-\nabla f_{i}(x^{t})}-n\sqn{a^t} \\
&=\sum_{i=1}^n\sqn{\nabla f_{i}(x^{t+1})-\nabla f_{i}(x^{t})-\nabla f(x^{t+1})+\nabla f(x^{t})}.
\end{align*}
Therefore,
\begin{align}
&\frac{1}{n-1}\sum_{i=1}^n\sum_{i'\in [n]\backslash\{i\}}
\sqn{\nabla f_{i'}(x^{t+1})-\nabla f_{i'}(x^{t})-a_i^t}\notag\\
&=\left(1-\frac{1}{(n-1)^2}\right)
\sum_{i=1}^n\sqn{\nabla f_{i}(x^{t+1})-\nabla f_{i}(x^{t})-\nabla f(x^{t+1})+\nabla f(x^{t})}\notag\\
&=\frac{n(n-2)}{(n-1)^2}\sum_{i=1}^n\sqn{\nabla f_{i}(x^{t+1})-\nabla f_{i}(x^{t})-\nabla f(x^{t+1})+\nabla f(x^{t})}.
\label{eq:prob_reduce_final}
\end{align}

Substituting~\eqref{eq:prob_reduce_final} into~\eqref{eqp2} (and recalling that the term~\eqref{eqp2} is replaced by zero if $n=2$),
and combining with the remaining terms derived above, we obtain that for any $n\geq 2$ and any $b_{\mathrm{grp}}\in[n]$,
\begin{align*}
&\Expc{\sum_{i=1}^n \sqn{g_i^{t+1}-\nabla f_i(x^{t+1})}}\\
&=\frac{b_{\mathrm{grp}}-1}{n} \sum_{i=1}^n\sqn{g_i^t - \nabla f_i(x^{t})}\\
&\quad+\frac{b_{\mathrm{grp}}-1}{nm}\sum_{i=1}^n\sum_{j=1}^m
\sqn{\nabla f_{i,j}(x^{t+1})-\nabla f_i(x^{t+1})-\nabla f_{i,j}(x^{t})+\nabla f_i(x^{t})}\\
&\quad+\frac{n-b_{\mathrm{grp}}}{n}
\sum_{i=1}^n\left\| g_i^{t}-\nabla f_i(x^{t})- \frac{n}{n-1}\left(\nabla f_{i}(x^{t+1})-\nabla f_{i}(x^{t})-\nabla f(x^{t+1})+\nabla f(x^{t}) \right) \right\|^2\\
&\quad+\frac{n-b_{\mathrm{grp}}}{(n-1)^2}\frac{n-1-b_{\mathrm{grp}}}{b_{\mathrm{grp}}}\sum_{i=1}^n
\sqn{\nabla f_{i}(x^{t+1})-\nabla f_{i}(x^{t})-\nabla f(x^{t+1})+\nabla f(x^{t})}\\
&\quad+\frac{n-b_{\mathrm{grp}}}{bnm}\sum_{i=1}^n\sum_{j=1}^m
\sqn{\nabla f_{i,j}(x^{t+1})- \nabla f_{i}(x^{t+1})-\nabla f_{i,j}(x^{t})+\nabla f_{i}(x^{t})}.\notag
\end{align*}

There remains to simplify the quadratic term
\begin{align*}
&\sum_{i=1}^n\left\| g_i^{t}-\nabla f_i(x^{t})- \frac{n}{n-1}\left(\nabla f_{i}(x^{t+1})-\nabla f_{i}(x^{t})-\nabla f(x^{t+1})+\nabla f(x^{t}) \right) \right\|^2.
\end{align*}
Expanding the squared norms and collecting terms yields
\begin{align*}
&\sum_{i=1}^n\left\| g_i^{t}-\nabla f_i(x^{t})- \frac{n}{n-1}\left(\nabla f_{i}(x^{t+1})-\nabla f_{i}(x^{t})-\nabla f(x^{t+1})+\nabla f(x^{t}) \right) \right\|^2\\
&=\sum_{i=1}^n\sqn{g_i^{t}-\nabla f_i(x^{t})}
+\frac{n^2}{(n-1)^2}\sum_{i=1}^n\sqn{\nabla f_{i}(x^{t+1})-\nabla f_{i}(x^{t})-\nabla f(x^{t+1})+\nabla f(x^{t})}\\
&\quad-\frac{2n}{n-1}\sum_{i=1}^n\left\langle g_i^{t}-\nabla f_i(x^{t}),
\nabla f_{i}(x^{t+1})-\nabla f_{i}(x^{t})-\nabla f(x^{t+1})+\nabla f(x^{t})\right\rangle.
\end{align*}
Substituting this expansion back and regrouping the coefficients of identical terms, we arrive at
\begin{align*}
&\Expc{\sum_{i=1}^n \sqn{g_i^{t+1}-\nabla f_i(x^{t+1})}}\\
&=\frac{n-1}{n} \sum_{i=1}^n\sqn{g_i^t - \nabla f_i(x^{t})}\\
&\quad-\frac{2(n-b_{\mathrm{grp}})}{n-1}\sum_{i=1}^n\left\langle g_i^{t}-\nabla f_i(x^{t}),
\nabla f_{i}(x^{t+1})-\nabla f_{i}(x^{t})-\nabla f(x^{t+1})+\nabla f(x^{t})\right\rangle\\
&\quad+\frac{(n-b_{\mathrm{grp}})(b_{\mathrm{grp}}+1)}{b_{\mathrm{grp}}(n-1)}\sum_{i=1}^n
\sqn{\nabla f_{i}(x^{t+1})-\nabla f_{i}(x^{t})-\nabla f(x^{t+1})+\nabla f(x^{t})}\\
&\quad+\frac{n-2b_{\mathrm{grp}}+b_{\mathrm{grp}}^2}{bnm}\sum_{i=1}^n\sum_{j=1}^m
\sqn{\nabla f_{i,j}(x^{t+1})- \nabla f_{i}(x^{t+1})-\nabla f_{i,j}(x^{t})+\nabla f_{i}(x^{t})}.\notag
\end{align*}

We now control the cross term in the previous expression. Recall that for any vectors $u,v\in\R^d$ and any $\alpha>0$,
Young's inequality gives
\begin{align*}
-2\langle u,v\rangle \leq \frac{1}{\alpha}\sqn{u}+\alpha \sqn{v}.
\end{align*}
Choosing $\alpha=2n$ and applying this with
\begin{align*}
u \eqdef g_i^{t}-\nabla f_i(x^{t}),
\qquad
v \eqdef \nabla f_{i}(x^{t+1})-\nabla f_{i}(x^{t})-\nabla f(x^{t+1})+\nabla f(x^{t}),
\end{align*}
we obtain, for every $i\in[n]$,
\begin{align*}
&-2\left\langle g_i^{t}-\nabla f_i(x^{t}),\nabla f_{i}(x^{t+1})-\nabla f_{i}(x^{t})-\nabla f(x^{t+1})+\nabla f(x^{t})\right\rangle\\
&\leq \frac{1}{2n}\sqn{g_i^{t}-\nabla f_i(x^{t})}
+2n\sqn{\nabla f_{i}(x^{t+1})-\nabla f_{i}(x^{t})-\nabla f(x^{t+1})+\nabla f(x^{t})}.
\end{align*}
Substituting this bound into the identity obtained above for
$\Expc{\sum_{i=1}^n \sqn{g_i^{t+1}-\nabla f_i(x^{t+1})}}$, we get
\begin{align*}
&\Expc{\sum_{i=1}^n \sqn{g_i^{t+1}-\nabla f_i(x^{t+1})}}\\
&\leq \left(\frac{n-1}{n} +\frac{n-b_{\mathrm{grp}}}{2n(n-1)}\right)\sum_{i=1}^n\sqn{g_i^t - \nabla f_i(x^{t})}\\
&\quad+\left(\frac{(n-b_{\mathrm{grp}})(b_{\mathrm{grp}}+1)}{b_{\mathrm{grp}}(n-1)}+\frac{2n(n-b_{\mathrm{grp}})}{n-1}\right)
\sum_{i=1}^n\sqn{\nabla f_{i}(x^{t+1})-\nabla f_{i}(x^{t}) -\nabla f(x^{t+1})+\nabla f(x^{t})}\\
&\quad+\frac{n-2b_{\mathrm{grp}}+b_{\mathrm{grp}}^2}{bnm}\sum_{i=1}^n\sum_{j=1}^m
\sqn{\nabla f_{i,j}(x^{t+1})- \nabla f_{i}(x^{t+1})-\nabla f_{i,j}(x^{t})+\nabla f_{i}(x^{t})}.
\end{align*}
We now simplify each term. First, since $b_{\mathrm{grp}}\geq 1$, we have
\begin{align*}
\frac{n-1}{n} +\frac{n-b_{\mathrm{grp}}}{2n(n-1)} \leq \frac{n-1}{n}+\frac{1}{2n}=1-\frac{1}{2n}.
\end{align*}
Second, using $\frac{b_{\mathrm{grp}}+1}{b_{\mathrm{grp}}}\leq 2$ and factoring out $\frac{n-b_{\mathrm{grp}}}{n-1}$, we obtain
\begin{align*}
\frac{(n-b_{\mathrm{grp}})(b_{\mathrm{grp}}+1)}{b_{\mathrm{grp}}(n-1)}+\frac{2n(n-b_{\mathrm{grp}})}{n-1}
=\frac{n-b_{\mathrm{grp}}}{n-1}\left(\frac{b_{\mathrm{grp}}+1}{b_{\mathrm{grp}}}+2n\right)
\leq \frac{n-b_{\mathrm{grp}}}{n-1}\left(2+2n\right)
=\frac{2(n-b_{\mathrm{grp}})(n+1)}{n-1}.
\end{align*}
Finally, by the similarity assumptions (with parameters $\delta_1$ and $\delta_2$), we have
\begin{align*}
\sum_{i=1}^n\sqn{\nabla f_{i}(x^{t+1})-\nabla f_{i}(x^{t})-\nabla f(x^{t+1})+\nabla f(x^{t})}
&\leq n\,\delta_1^2\sqn{x^{t+1}-x^t},\\
\frac{1}{nm}\sum_{i=1}^n\sum_{j=1}^m
\sqn{\nabla f_{i,j}(x^{t+1})- \nabla f_{i}(x^{t+1})-\nabla f_{i,j}(x^{t})+\nabla f_{i}(x^{t})}
&\leq \delta_2^2\sqn{x^{t+1}-x^t}.
\end{align*}
Combining the previous inequalities yields
\begin{align*}
&\Expc{\sum_{i=1}^n \sqn{g_i^{t+1}-\nabla f_i(x^{t+1})}}\\
&\leq\left(1-\frac{1}{2n}\right)\sum_{i=1}^n\sqn{g_i^t - \nabla f_i(x^{t})}
+\frac{2n(n-b_{\mathrm{grp}})(n+1)}{n-1}\delta_1^2\sqn{x^{t+1}-x^t}
+\frac{n-2b_{\mathrm{grp}}+b_{\mathrm{grp}}^2}{b_{\mathrm{grp}}}\delta_2^2\sqn{x^{t+1}-x^t}.
\end{align*}

\medskip

We now study the averaged estimator $g^{t+1}\eqdef \frac{1}{n}\sum_{i=1}^n g_i^{t+1}$.
Using the update rules of Algorithm~\ref{algsilage2a} (anchor reset on $i^t$, recursive updates on $\Omega^t$,
and drift updates on $[n]\setminus\tilde{\Omega}^t$), we can write
\begin{align*}
g^{t+1}
&=\frac{1}{n}\nabla f_{i^t}(x^{t+1})
+\frac{1}{n}\sum_{i\in\Omega^t}\left(g_{i}^t+ \nabla f_{i,j_i^t}(x^{t+1})-\nabla f_{i,j_i^t}(x^{t})\right)
+\frac{1}{n}\sum_{i\in [n]\backslash\tilde{\Omega}^t}\left(g_{i}^t+ d^t\right)\\
&=g^t + \frac{1}{n}\left(\nabla f_{i^t}(x^{t+1})-g_{i^t}^t\right)
+\frac{1}{n}\sum_{i\in\Omega^t}\left( \nabla f_{i,j_i^t}(x^{t+1})-\nabla f_{i,j_i^t}(x^{t})\right)
+\frac{n-b_{\mathrm{grp}}}{n} d^t.
\end{align*}

Next, we incorporate the anchor index into the sum carefully. Recall that $\tilde{\Omega}^t=\Omega^t\cup\{i^t\}$ and define, for each $i\in\tilde{\Omega}^t$,
\begin{align*}
\Delta_i^t
&\eqdef
\nabla f_{i,j_i^t}(x^{t+1})-\nabla f_{i,j_i^t}(x^{t}).
\end{align*}
Since $d^t=\frac{1}{b_{\mathrm{grp}}}\sum_{i\in\tilde{\Omega}^t}\Delta_i^t$, we have
\begin{align*}
\frac{1}{n}\sum_{i\in\Omega^t}\Delta_i^t+\frac{n-b_{\mathrm{grp}}}{n}\,d^t
&=
\frac{1}{n}\left(\sum_{i\in\tilde{\Omega}^t}\Delta_i^t-\Delta_{i^t}^t\right)
+\frac{n-b_{\mathrm{grp}}}{n}\,d^t
\\
&=
\frac{b_{\mathrm{grp}}}{n}\,d^t-\frac{1}{n}\Delta_{i^t}^t+\frac{n-b_{\mathrm{grp}}}{n}\,d^t
\\
&=
d^t-\frac{1}{n}\Delta_{i^t}^t.
\end{align*}
Therefore, the exact pathwise identity is
\begin{align*}
g^{t+1}
&=
g^t
+\frac{1}{n}\left(\nabla f_{i^t}(x^{t+1})-g_{i^t}^t\right)
+d^t
-\frac{1}{n}\Delta_{i^t}^t
\\
&=
g^t
+\frac{1}{n}\left(\nabla f_{i^t}(x^{t})-g_{i^t}^t\right)
+d^t+\eta^t,
\end{align*}
where
\begin{align*}
\eta^t
&\eqdef
\frac{1}{n}\left[
\nabla f_{i^t}(x^{t+1})-\nabla f_{i^t}(x^t)-\Delta_{i^t}^t
\right].
\end{align*}
Conditionally on $\mathcal{F}^t$ and $i^t$, the index $j_{i^t}^t$ is uniform on $[m]$, and hence
\begin{align*}
\Expb{\Delta_{i^t}^t}{\mathcal{F}^t,i^t}
&=
\nabla f_{i^t}(x^{t+1})-\nabla f_{i^t}(x^t).
\end{align*}
Thus,
\begin{align*}
\Expb{\eta^t}{\mathcal{F}^t,i^t}
&=0,
&
\Expc{\eta^t}
&=0.
\end{align*}
Taking conditional expectation in the corrected pathwise identity gives
\begin{align*}
\Expc{g^{t+1}-\nabla f(x^{t+1})}
&=
g^t-\nabla f(x^{t+1})
+\frac{1}{n}\Expb{\nabla f_{i^t}(x^t)-g_{i^t}^t}{\mathcal{F}^t}
+\Expc{d^t}
+\Expc{\eta^t}.
\end{align*}
Since $i^t$ is uniform on $[n]$ and independent of $\mathcal{F}^t$, we have
\begin{align*}
\Expb{\nabla f_{i^t}(x^t)}{\mathcal{F}^t}
&=
\nabla f(x^t),
&
\Expb{g_{i^t}^t}{\mathcal{F}^t}
&=
g^t.
\end{align*}
Moreover, each $\Delta_i^t$ is an unbiased estimator of $\nabla f_i(x^{t+1})-\nabla f_i(x^t)$, and $\tilde{\Omega}^t$ is uniform over size-$b_{\mathrm{grp}}$ subsets of $[n]$, hence
\begin{align*}
\Expc{d^t}
&=
\nabla f(x^{t+1})-\nabla f(x^t).
\end{align*}
Using also $\Expc{\eta^t}=0$, we obtain
\begin{align*}
\Expc{g^{t+1}-\nabla f(x^{t+1})}
&=
g^t-\nabla f(x^{t+1})
+\frac{1}{n}\left(\nabla f(x^t)-g^t\right)
+\nabla f(x^{t+1})-\nabla f(x^t)
\\
&=
\frac{n-1}{n}\left(g^t-\nabla f(x^t)\right).
\end{align*}

We now turn to the conditional second moment. Using the variance decomposition
$\Expc{\sqn{Z}}=\sqn{\Expc{Z}}+\Expc{\sqn{Z-\Expc{Z}}}$ with
$Z\eqdef g^{t+1}-\nabla f(x^{t+1})$, we obtain
\begin{align*}
\Expc{ \sqn{g^{t+1}-\nabla f(x^{t+1})}}
&=\frac{(n-1)^2}{n^2} \sqn{g^t-\nabla f(x^{t})}
+\Expc{\sqn{g^{t+1}-\nabla f(x^{t+1})-\frac{n-1}{n}\left(g^t-\nabla f(x^{t})\right)}}.
\end{align*}

We now simplify the expansion of the centered term 
\begin{align*}
g^{t+1}
&=g^t+\frac{1}{n}\left(\nabla f_{i^t}(x^{t+1})-g_{i^t}^t\right)
+\frac{1}{n}\sum_{i\in\Omega^t}\left(\nabla f_{i,j_i^t}(x^{t+1})-\nabla f_{i,j_i^t}(x^{t})\right)
+\frac{n-b_{\mathrm{grp}}}{n}\,d^t,
\end{align*}
where
\begin{align*}
d^t=\frac{1}{b_{\mathrm{grp}}}\sum_{i\in\tilde{\Omega}^t}\Big(\nabla f_{i,j_i^t}(x^{t+1})-\nabla f_{i,j_i^t}(x^{t})\Big),
\qquad \tilde{\Omega}^t\eqdef \Omega^t\cup\{i^t\}.
\end{align*}
Starting from
\begin{align*}
g^{t+1}-\nabla f(x^{t+1})-\frac{n-1}{n}\left(g^t-\nabla f(x^{t})\right),
\end{align*}
we substitute the above expression for $g^{t+1}$ and regroup terms:
\begin{align*}
&g^{t+1}-\nabla f(x^{t+1})-\frac{n-1}{n}\left(g^t-\nabla f(x^{t})\right)\\
&=\left(g^t-\nabla f(x^{t+1})\right)-\frac{n-1}{n}g^t+\frac{n-1}{n}\nabla f(x^{t})
+\frac{1}{n}\left(\nabla f_{i^t}(x^{t+1})-g_{i^t}^t\right)\\
&\quad+\frac{1}{n}\sum_{i\in\Omega^t}\left(\nabla f_{i,j_i^t}(x^{t+1})-\nabla f_{i,j_i^t}(x^{t})\right)
+\frac{n-b_{\mathrm{grp}}}{n}\,d^t\\
&=\frac{1}{n}\left(\nabla f_{i^t}(x^{t+1})-g_{i^t}^t+g^t-\nabla f(x^{t+1})\right)\\
&\quad+\frac{1}{n}\sum_{i\in\Omega^t}\left(\nabla f_{i,j_i^t}(x^{t+1})-\nabla f_{i,j_i^t}(x^{t})\right)
+\frac{n-b_{\mathrm{grp}}}{n}\,d^t
+\frac{n-1}{n}\left(\nabla f(x^{t})-\nabla f(x^{t+1})\right).
\end{align*}
We now distribute the deterministic drift term $\frac{n-1}{n}\left(\nabla f(x^{t})-\nabla f(x^{t+1})\right)$ across the
$\Omega^t$-sum and the drift part. Using $|\Omega^t|=b_{\mathrm{grp}}-1$ and $|\tilde{\Omega}^t|=b_{\mathrm{grp}}$, we have the identity
\begin{align*}
\frac{n-1}{n}\left(\nabla f(x^{t})-\nabla f(x^{t+1})\right)
=\frac{1}{n}\sum_{i\in\Omega^t}\left(\nabla f(x^{t})-\nabla f(x^{t+1})\right)
+\frac{n-b_{\mathrm{grp}}}{n}\left(\nabla f(x^{t})-\nabla f(x^{t+1})\right).
\end{align*}
Substituting this into the previous display yields
\begin{align*}
&g^{t+1}-\nabla f(x^{t+1})-\frac{n-1}{n}\left(g^t-\nabla f(x^{t})\right)\\
&=\frac{1}{n}\left(\nabla f_{i^t}(x^{t+1})-\nabla f(x^{t+1})-g_{i^t}^t+g^t\right)\\
&\quad+\frac{1}{n}\sum_{i\in\Omega^t}\left(\nabla f_{i,j_i^t}(x^{t+1})-\nabla f(x^{t+1})-\nabla f_{i,j_i^t}(x^{t})+\nabla f(x^{t})\right)\\
&\quad+\frac{n-b_{\mathrm{grp}}}{n}\left(d^t-\big(\nabla f(x^{t+1})-\nabla f(x^{t})\big)\right).
\end{align*}
Finally, we expand the last difference using the definition of $d^t$:
\begin{align*}
d^t-\big(\nabla f(x^{t+1})-\nabla f(x^{t})\big)
&=\frac{1}{b_{\mathrm{grp}}}\sum_{i\in\tilde{\Omega}^t}\Big(\nabla f_{i,j_i^t}(x^{t+1})-\nabla f_{i,j_i^t}(x^{t})\Big)
-\big(\nabla f(x^{t+1})-\nabla f(x^{t})\big)\\
&=\frac{1}{b_{\mathrm{grp}}}\sum_{i\in\tilde{\Omega}^t}\Big(\nabla f_{i,j_i^t}(x^{t+1})-\nabla f(x^{t+1})-\nabla f_{i,j_i^t}(x^{t})+\nabla f(x^{t})\Big).
\end{align*}
Substituting this back gives the desired decomposition:
\begin{align*}
&g^{t+1}-\nabla f(x^{t+1})-\frac{n-1}{n}\left(g^t-\nabla f(x^{t})\right)\\
&=\frac{1}{n}\left(\nabla f_{i^t}(x^{t+1})-\nabla f(x^{t+1})-g_{i^t}^t+g^t\right)\\
&\quad+\frac{1}{n}\sum_{i\in\Omega^t}\left(\nabla f_{i,j_i^t}(x^{t+1})-\nabla f(x^{t+1})-\nabla f_{i,j_i^t}(x^{t})+\nabla f(x^{t})\right)\\
&\quad+\frac{n-b_{\mathrm{grp}}}{b_{\mathrm{grp}}n}\sum_{i\in \tilde{\Omega}^t}\left( \nabla f_{i,j_i^t}(x^{t+1})-\nabla f(x^{t+1})-\nabla f_{i,j_i^t}(x^{t})+\nabla f(x^{t})\right).
\end{align*}
Consequently,
\begin{align*}
&\sqn{g^{t+1}-\nabla f(x^{t+1})-\frac{n-1}{n}\left(g^t-\nabla f(x^{t})\right)}\\
&=\left\|\frac{1}{n}\left(\nabla f_{i^t}(x^{t+1})-\nabla f(x^{t+1})-g_{i^t}^t+g^t\right)\right.\\
&\quad+\frac{1}{n}\sum_{i\in\Omega^t}\left( \nabla f_{i,j_i^t}(x^{t+1})-\nabla f(x^{t+1})-\nabla f_{i,j_i^t}(x^{t})+\nabla f(x^{t})\right)\\
&\quad\left.{}+\frac{n-b_{\mathrm{grp}}}{b_{\mathrm{grp}}n} \sum_{i\in \tilde{\Omega}^t}\left( \nabla f_{i,j_i^t}(x^{t+1})-\nabla f(x^{t+1})-\nabla f_{i,j_i^t}(x^{t})+\nabla f(x^{t})\right)\right\|^2.
\end{align*}

Therefore, we arrive at
\begin{align}\label{eq:internal}
&\Expc{ \sqn{g^{t+1}-\nabla f(x^{t+1})}} \notag\\
&=\frac{(n-1)^2}{n^2} \sqn{g^t-\nabla f(x^{t})}
+\mathbb{E}\!\left[\left\|\frac{1}{n}\left(\nabla f_{i^t}(x^{t+1})-\nabla f(x^{t+1})-g_{i^t}^t+g^t\right)\right.\right. \notag\\
&\quad+\frac{1}{n}\sum_{i\in\Omega^t}\left( \nabla f_{i,j_i^t}(x^{t+1})-\nabla f(x^{t+1})-\nabla f_{i,j_i^t}(x^{t})+\nabla f(x^{t})\right) \notag\\
&\quad\left.\left.\left.\!\!\!{}+\frac{n-b_{\mathrm{grp}}}{b_{\mathrm{grp}}n} \sum_{i\in \tilde{\Omega}^t}\left( \nabla f_{i,j_i^t}(x^{t+1})-\nabla f(x^{t+1})-\nabla f_{i,j_i^t}(x^{t})+\nabla f(x^{t})\right)\right\|^2\,\right|\mathcal{F}^t\right].
\end{align}

We now show explicitly how to pass from \eqref{eq:internal}
to the representation where the ``outer'' deviations are grouped inside one squared norm and the ``inner'' centered noises are separated.

For each $i\in[n]$ and $j\in[m]$, define the centered inner-noise increment
\begin{align*}
\widetilde{\Delta}_i^t(j)\eqdef \nabla f_{i,j}(x^{t+1})-\nabla f_i(x^{t+1})-\nabla f_{i,j}(x^{t})+\nabla f_i(x^{t}),
\end{align*}
and define the outer deviation increment
\begin{align*}
\bar{\Delta}_i^t \eqdef \nabla f_i(x^{t+1})-\nabla f(x^{t+1})-\nabla f_i(x^{t})+\nabla f(x^{t}).
\end{align*}
Then, for every $i\in[n]$ and every realization of $j_i^t$, we have the identity
\begin{align}
\label{eq:split_inner_outer}
\nabla f_{i,j_i^t}(x^{t+1})-\nabla f(x^{t+1})-\nabla f_{i,j_i^t}(x^{t})+\nabla f(x^{t})
=\bar{\Delta}_i^t+\widetilde{\Delta}_i^t(j_i^t).
\end{align}

We now rewrite the vector inside the conditional squared norm by using~\eqref{eq:split_inner_outer} term-by-term.
Set
\begin{align*}
A^t &\eqdef \frac{1}{n}\left(\nabla f_{i^t}(x^{t+1})-\nabla f(x^{t+1})-g_{i^t}^t+g^t\right),\\
B^t &\eqdef \frac{1}{n}\sum_{i\in\Omega^t}\left( \nabla f_{i,j_i^t}(x^{t+1})-\nabla f(x^{t+1})-\nabla f_{i,j_i^t}(x^{t})+\nabla f(x^{t})\right),\\
C^t &\eqdef \frac{n-b_{\mathrm{grp}}}{b_{\mathrm{grp}}n} \sum_{i\in \tilde{\Omega}^t}\left( \nabla f_{i,j_i^t}(x^{t+1})-\nabla f(x^{t+1})-\nabla f_{i,j_i^t}(x^{t})+\nabla f(x^{t})\right).
\end{align*}
Then the conditional squared norm above is $\sqn{A^t+B^t+C^t}$.
Using~\eqref{eq:split_inner_outer} in $B^t$ and $C^t$ yields
\begin{align}
B^t
&=\frac{1}{n}\sum_{i\in\Omega^t}\bar{\Delta}_i^t
+\frac{1}{n}\sum_{i\in\Omega^t}\widetilde{\Delta}_i^t(j_i^t),
\label{eq:B-split}\\
C^t
&=\frac{n-b_{\mathrm{grp}}}{b_{\mathrm{grp}}n}\sum_{i\in\tilde{\Omega}^t}\bar{\Delta}_i^t
+\frac{n-b_{\mathrm{grp}}}{b_{\mathrm{grp}}n}\sum_{i\in\tilde{\Omega}^t}\widetilde{\Delta}_i^t(j_i^t).
\label{eq:C-split}
\end{align}
Moreover, in $A^t$ we add and subtract $\nabla f_{i^t}(x^t)-\nabla f(x^t)$:
\begin{align}
A^t
&=\frac{1}{n}\left(\nabla f_{i^t}(x^{t})-\nabla f(x^{t})-g_{i^t}^t+g^t\right)
+\frac{1}{n}\left(\nabla f_{i^t}(x^{t+1})-\nabla f_{i^t}(x^{t})-\nabla f(x^{t+1})+\nabla f(x^{t})\right)\notag\\
&=\frac{1}{n}\left(\nabla f_{i^t}(x^{t})-\nabla f(x^{t})-g_{i^t}^t+g^t\right)
+\frac{1}{n}\bar{\Delta}_{i^t}^t.
\label{eq:A-split}
\end{align}
Now combine~\eqref{eq:B-split},~\eqref{eq:C-split}, and~\eqref{eq:A-split}:
\begin{align*}
A^t+B^t+C^t
&=\frac{1}{n}\left(\nabla f_{i^t}(x^{t})-\nabla f(x^{t})-g_{i^t}^t+g^t\right)
+\left(\frac{1}{n}\bar{\Delta}_{i^t}^t+\frac{1}{n}\sum_{i\in\Omega^t}\bar{\Delta}_i^t+\frac{n-b_{\mathrm{grp}}}{b_{\mathrm{grp}}n}\sum_{i\in\tilde{\Omega}^t}\bar{\Delta}_i^t\right)\\
&\quad+\frac{1}{n}\sum_{i\in\Omega^t}\widetilde{\Delta}_i^t(j_i^t)
+\frac{n-b_{\mathrm{grp}}}{b_{\mathrm{grp}}n}\sum_{i\in\tilde{\Omega}^t}\widetilde{\Delta}_i^t(j_i^t).
\end{align*}
Since $\tilde{\Omega}^t=\Omega^t\cup\{i^t\}$ and $|\tilde{\Omega}^t|=b_{\mathrm{grp}}$, we have
\begin{align*}
\frac{1}{n}\bar{\Delta}_{i^t}^t+\frac{1}{n}\sum_{i\in\Omega^t}\bar{\Delta}_i^t
=\frac{1}{n}\sum_{i\in\tilde{\Omega}^t}\bar{\Delta}_i^t
=\frac{b_{\mathrm{grp}}}{n}\cdot \frac{1}{b_{\mathrm{grp}}}\sum_{i\in\tilde{\Omega}^t}\bar{\Delta}_i^t,
\end{align*}
and therefore
\begin{align}
\frac{1}{n}\bar{\Delta}_{i^t}^t+\frac{1}{n}\sum_{i\in\Omega^t}\bar{\Delta}_i^t+\frac{n-b_{\mathrm{grp}}}{b_{\mathrm{grp}}n}\sum_{i\in\tilde{\Omega}^t}\bar{\Delta}_i^t
&=\left(\frac{b_{\mathrm{grp}}}{n}+\frac{n-b_{\mathrm{grp}}}{n}\right)\frac{1}{b_{\mathrm{grp}}}\sum_{i\in\tilde{\Omega}^t}\bar{\Delta}_i^t \notag\\
&=\frac{1}{b_{\mathrm{grp}}}\sum_{i\in\tilde{\Omega}^t}\bar{\Delta}_i^t.
\label{eq:outer-combine}
\end{align}
Similarly, using again $\tilde{\Omega}^t=\Omega^t\cup\{i^t\}$, we can split the second noise sum as
\begin{align}
\label{eq:noise-split}
\sum_{i\in\tilde{\Omega}^t}\widetilde{\Delta}_i^t(j_i^t)
=\sum_{i\in\Omega^t}\widetilde{\Delta}_i^t(j_i^t)+\widetilde{\Delta}_{i^t}^t(j_{i^t}^t),
\end{align}
and hence
\begin{align}
\frac{1}{n}\sum_{i\in\Omega^t}\widetilde{\Delta}_i^t(j_i^t)+\frac{n-b_{\mathrm{grp}}}{b_{\mathrm{grp}}n}\sum_{i\in\tilde{\Omega}^t}\widetilde{\Delta}_i^t(j_i^t)
&=\left(\frac{1}{n}+\frac{n-b_{\mathrm{grp}}}{b_{\mathrm{grp}}n}\right)\sum_{i\in\Omega^t}\widetilde{\Delta}_i^t(j_i^t)
+\frac{n-b_{\mathrm{grp}}}{b_{\mathrm{grp}}n}\widetilde{\Delta}_{i^t}^t(j_{i^t}^t)\notag\\
&=\frac{1}{b_{\mathrm{grp}}}\sum_{i\in\Omega^t}\widetilde{\Delta}_i^t(j_i^t)
+\frac{n-b_{\mathrm{grp}}}{b_{\mathrm{grp}}n}\widetilde{\Delta}_{i^t}^t(j_{i^t}^t).
\label{eq:noise-combine}
\end{align}

Substituting~\eqref{eq:outer-combine} and~\eqref{eq:noise-combine} back, we obtain the exact decomposition
\begin{align}
\label{eq:ABC-decomp}
A^t+B^t+C^t
&=\underbrace{\frac{1}{n}\left(\nabla f_{i^t}(x^{t})-\nabla f(x^{t})-g_{i^t}^t+g^t\right)
+\frac{1}{b_{\mathrm{grp}}}\sum_{i\in\tilde{\Omega}^t}\bar{\Delta}_i^t}_{\text{outer part}} \notag\\
&\quad +\underbrace{\frac{1}{b_{\mathrm{grp}}}\sum_{i\in\Omega^t}\widetilde{\Delta}_i^t(j_i^t)
+\frac{n-b_{\mathrm{grp}}}{b_{\mathrm{grp}}n}\widetilde{\Delta}_{i^t}^t(j_{i^t}^t)}_{\text{inner noise}}.
\end{align}

We now expand the conditional second moment with all cross terms and show explicitly why they vanish.
Define the ``outer part'' and the ``inner noise'' as in~\eqref{eq:ABC-decomp}:
\begin{align*}
R^t &\eqdef \frac{1}{n}\left(\nabla f_{i^t}(x^{t})-\nabla f(x^{t})-g_{i^t}^t+g^t\right)
+\frac{1}{b_{\mathrm{grp}}}\sum_{i\in\tilde{\Omega}^t}\bar{\Delta}_i^t,\\
N^t &\eqdef \frac{1}{b_{\mathrm{grp}}}\sum_{i\in\Omega^t}\widetilde{\Delta}_i^t(j_i^t)
+\frac{n-b_{\mathrm{grp}}}{b_{\mathrm{grp}}n}\widetilde{\Delta}_{i^t}^t(j_{i^t}^t),
\end{align*}
so that $A^t+B^t+C^t = R^t+N^t$.
Recall that, conditional on $\mathcal{F}^t$ and on the sampled sets $(i^t,\Omega^t)$ (equivalently $\tilde{\Omega}^t$),
the vectors $R^t$ and $\bar{\Delta}_i^t$ are deterministic, while the randomness in $N^t$ comes only from the indices
$\{j_i^t\}_{i\in\Omega^t}$ and $j_{i^t}^t$.

\textbf{Step 1: expand $\|R^t+N^t\|^2$ and drop the outer--inner cross term.}

We first expand the square:
\begin{align*}
\Expc{\sqn{A^t+B^t+C^t}}
&=\Expc{\sqn{R^t+N^t}} \\
&=\Expc{\sqn{R^t}} + 2\,\Expc{\inner{R^t}{N^t}} + \Expc{\sqn{N^t}}.
\end{align*}
We show that the cross term is zero. Using the tower property with conditioning on $(\mathcal{F}^t,\tilde{\Omega}^t)$,
\begin{align*}
\Expc{\inner{R^t}{N^t}}
&=\Expb{\Expb{\inner{R^t}{N^t}}{\mathcal{F}^t,\tilde{\Omega}^t}}{\mathcal{F}^t}.
\end{align*}
Conditional on $(\mathcal{F}^t,\tilde{\Omega}^t)$, the vector $R^t$ is deterministic and thus can be pulled out:
\begin{align*}
\Expb{\inner{R^t}{N^t}}{\mathcal{F}^t,\tilde{\Omega}^t}
&=\inner{R^t}{\Expb{N^t}{\mathcal{F}^t,\tilde{\Omega}^t}}.
\end{align*}
Moreover, for each $i\in\tilde{\Omega}^t$, the index $j_i^t$ is uniform on $[m]$, hence
\begin{align*}
\Expb{\widetilde{\Delta}_i^t(j_i^t)}{\mathcal{F}^t,\tilde{\Omega}^t}=0.
\end{align*}
Therefore,
\begin{align*}
\Expb{N^t}{\mathcal{F}^t,\tilde{\Omega}^t}
&=\frac{1}{b_{\mathrm{grp}}}\sum_{i\in\Omega^t}\Expb{\widetilde{\Delta}_i^t(j_i^t)}{\mathcal{F}^t,\tilde{\Omega}^t}
+\frac{n-b_{\mathrm{grp}}}{b_{\mathrm{grp}}n}\Expb{\widetilde{\Delta}_{i^t}^t(j_{i^t}^t)}{\mathcal{F}^t,\tilde{\Omega}^t}
=0,
\end{align*}
and hence
\begin{align*}
\underbrace{\Expb{\inner{R^t}{N^t}}{\mathcal{F}^t,\tilde{\Omega}^t}}_{=\,0}=0
\qquad\Longrightarrow\qquad
\underbrace{\Expc{\inner{R^t}{N^t}}}_{=\,0}=0.
\end{align*}
Plugging this into the square expansion yields
\begin{align*}
\Expc{\sqn{R^t+N^t}}=\Expc{\sqn{R^t}}+\Expc{\sqn{N^t}}.
\end{align*}

\textbf{Step 2: expand $\|N^t\|^2$ and drop all inner-noise cross terms.}

Write $N^t$ as the sum of two parts,
\begin{align*}
N^t = N_1^t + N_2^t,
\qquad
N_1^t \eqdef \frac{1}{b_{\mathrm{grp}}}\sum_{i\in\Omega^t}\widetilde{\Delta}_i^t(j_i^t),
\qquad
N_2^t \eqdef \frac{n-b_{\mathrm{grp}}}{b_{\mathrm{grp}}n}\widetilde{\Delta}_{i^t}^t(j_{i^t}^t).
\end{align*}
Then
\begin{align*}
\Expc{\sqn{N^t}}
&=\Expc{\sqn{N_1^t+N_2^t}} \\
&=\Expc{\sqn{N_1^t}} + 2\,\Expc{\inner{N_1^t}{N_2^t}} + \Expc{\sqn{N_2^t}}.
\end{align*}
We show that $\Expc{\inner{N_1^t}{N_2^t}}=0$. Conditional on $(\mathcal{F}^t,\tilde{\Omega}^t)$,
the random variables $\{j_i^t\}_{i\in\Omega^t}$ are independent of $j_{i^t}^t$, and the vectors
$\widetilde{\Delta}_i^t(j_i^t)$ have zero conditional mean. Hence,
\begin{align*}
\Expb{\inner{N_1^t}{N_2^t}}{\mathcal{F}^t,\tilde{\Omega}^t}
&=\frac{n-b_{\mathrm{grp}}}{b_{\mathrm{grp}}^2n}\sum_{i\in\Omega^t}
\Expb{\inner{\widetilde{\Delta}_i^t(j_i^t)}{\widetilde{\Delta}_{i^t}^t(j_{i^t}^t)}}{\mathcal{F}^t,\tilde{\Omega}^t}\\
&=\frac{n-b_{\mathrm{grp}}}{b_{\mathrm{grp}}^2n}\sum_{i\in\Omega^t}
\inner{\Expb{\widetilde{\Delta}_i^t(j_i^t)}{\mathcal{F}^t,\tilde{\Omega}^t}}{\Expb{\widetilde{\Delta}_{i^t}^t(j_{i^t}^t)}{\mathcal{F}^t,\tilde{\Omega}^t}}\\
&=\frac{n-b_{\mathrm{grp}}}{b_{\mathrm{grp}}^2n}\sum_{i\in\Omega^t}\inner{0}{0}=0,
\end{align*}
and therefore
\begin{align*}
\underbrace{\Expc{\inner{N_1^t}{N_2^t}}}_{=\,0}=0.
\end{align*}
It remains to compute $\Expc{\sqn{N_1^t}}$. Expanding the square gives
\begin{align*}
\sqn{N_1^t}
&=\sqn{\frac{1}{b_{\mathrm{grp}}}\sum_{i\in\Omega^t}\widetilde{\Delta}_i^t(j_i^t)} \\
&=\frac{1}{b_{\mathrm{grp}}^2}\sum_{i\in\Omega^t}\sqn{\widetilde{\Delta}_i^t(j_i^t)}
+\frac{2}{b_{\mathrm{grp}}^2}\sum_{\substack{i,k\in\Omega^t\\ i<k}}
\inner{\widetilde{\Delta}_i^t(j_i^t)}{\widetilde{\Delta}_k^t(j_k^t)}.
\end{align*}
Conditional on $(\mathcal{F}^t,\tilde{\Omega}^t)$, the random indices $(j_i^t)_{i\in\Omega^t}$ are independent,
and each $\widetilde{\Delta}_i^t(j_i^t)$ has zero conditional mean. Hence for $i\neq k$,
\begin{align*}
\Expb{\inner{\widetilde{\Delta}_i^t(j_i^t)}{\widetilde{\Delta}_k^t(j_k^t)}}{\mathcal{F}^t,\tilde{\Omega}^t}
&=\inner{\Expb{\widetilde{\Delta}_i^t(j_i^t)}{\mathcal{F}^t,\tilde{\Omega}^t}}
{\Expb{\widetilde{\Delta}_k^t(j_k^t)}{\mathcal{F}^t,\tilde{\Omega}^t}}
=\inner{0}{0}=0.
\end{align*}
Therefore all cross terms vanish in expectation:
\begin{align*}
\underbrace{\Expc{\sum_{\substack{i,k\in\Omega^t\\ i<k}}
\inner{\widetilde{\Delta}_i^t(j_i^t)}{\widetilde{\Delta}_k^t(j_k^t)}}}_{=\,0}=0,
\end{align*}
and we obtain
\begin{align*}
\Expc{\sqn{N_1^t}}
=\frac{1}{b_{\mathrm{grp}}^2}\Expc{\sum_{i\in\Omega^t}\sqn{\widetilde{\Delta}_i^t(j_i^t)}}.
\end{align*}
Finally, $\sqn{N_2^t}=\frac{(n-b_{\mathrm{grp}})^2}{b_{\mathrm{grp}}^2n^2}\sqn{\widetilde{\Delta}_{i^t}^t(j_{i^t}^t)}$, so
\begin{align*}
\Expc{\sqn{N_2^t}}
=\frac{(n-b_{\mathrm{grp}})^2}{b_{\mathrm{grp}}^2n^2}\Expc{\sqn{\widetilde{\Delta}_{i^t}^t(j_{i^t}^t)}}.
\end{align*}
Combining the last displays yields
\begin{align*}
\Expc{\sqn{N^t}}
=\frac{1}{b_{\mathrm{grp}}^2}\Expc{\sum_{i\in\Omega^t}\sqn{\widetilde{\Delta}_i^t(j_i^t)}}
+\frac{(n-b_{\mathrm{grp}})^2}{b_{\mathrm{grp}}^2n^2}\Expc{\sqn{\widetilde{\Delta}_{i^t}^t(j_{i^t}^t)}}.
\end{align*}

\textbf{Conclusion.}

Putting Step~1 and Step~2 together, we obtain
\begin{align*}
\Expc{\sqn{A^t+B^t+C^t}}
&=\Expc{\sqn{R^t}}
+\frac{1}{b_{\mathrm{grp}}^2}\Expc{\sum_{i\in\Omega^t}\sqn{\widetilde{\Delta}_i^t(j_i^t)}}\\
&\quad+\frac{(n-b_{\mathrm{grp}})^2}{b_{\mathrm{grp}}^2n^2}\Expc{\sqn{\widetilde{\Delta}_{i^t}^t(j_{i^t}^t)}},
\end{align*}
i.e.,
\begin{align*}
&\Expc{\sqn{A^t+B^t+C^t}} \\
&=\Expc{\left\|\frac{1}{n}\left(\nabla f_{i^t}(x^{t})-\nabla f(x^{t})-g_{i^t}^t+g^t\right)
+\frac{1}{b_{\mathrm{grp}}}\sum_{i\in\tilde{\Omega}^t}\bar{\Delta}_i^t\right\|^2}\\
&\quad+\frac{1}{b_{\mathrm{grp}}^2}\Expc{\sum_{i\in\Omega^t}\sqn{\widetilde{\Delta}_i^t(j_i^t)}}
+\frac{(n-b_{\mathrm{grp}})^2}{b_{\mathrm{grp}}^2n^2}\Expc{\sqn{\widetilde{\Delta}_{i^t}^t(j_{i^t}^t)}}.
\end{align*}

Finally, substituting back $\bar{\Delta}_i^t$ and $\widetilde{\Delta}_i^t(j)$ gives exactly
\begin{align*}
&\Expc{ \sqn{g^{t+1}-\nabla f(x^{t+1})}}\\
&=\frac{(n-1)^2}{n^2} \sqn{g^t-\nabla f(x^{t})}
+\mathbb{E}\!\left[\left\|\frac{1}{n}\left(\nabla f_{i^t}(x^{t})-\nabla f(x^{t})-g_{i^t}^t+g^t\right)\right.\right.\\
&\quad\left.\left.\left.\!\!\!{}+\frac{1}{b_{\mathrm{grp}}} \sum_{i\in \tilde{\Omega}^t}\left( \nabla f_{i}(x^{t+1})-\nabla f(x^{t+1})-\nabla f_{i}(x^{t})+\nabla f(x^{t})\right)\right\|^2\,\right|\mathcal{F}^t\right]\\
&\quad+\frac{1}{b_{\mathrm{grp}}^2}\Expc{ \sum_{i\in \Omega^t}\sqn{ \nabla f_{i,j_i^t}(x^{t+1})-\nabla f_i(x^{t+1})-\nabla f_{i,j_i^t}(x^{t})+\nabla f_i(x^t)}}\\
&\quad+\frac{(n-b_{\mathrm{grp}})^2}{b_{\mathrm{grp}}^2n^2}\Expc{ \sqn{ \nabla f_{i^t,j_{i^t}^t}(x^{t+1})-\nabla f_{i^t}(x^{t+1})-\nabla f_{i^t,j_{i^t}^t}(x^{t})+\nabla f_{i^t}(x^t)}}.
\end{align*}

We now bound the last two noise terms.
For $i\in[n]$ and $j\in[m]$, recall the shorthand
\begin{align*}
\widetilde{\Delta}_i^t(j)\eqdef \nabla f_{i,j}(x^{t+1})-\nabla f_i(x^{t+1})-\nabla f_{i,j}(x^{t})+\nabla f_i(x^{t}).
\end{align*}

\textbf{Noise term over $\Omega^t$.}

Using the tower property and conditioning on $(i^t,\Omega^t)$, we obtain
\begin{align*}
&\frac{1}{b_{\mathrm{grp}}^2}\Expc{ \sum_{i\in \Omega^t}\sqn{ \nabla f_{i,j_i^t}(x^{t+1})-\nabla f_i(x^{t+1})-\nabla f_{i,j_i^t}(x^{t})+\nabla f_i(x^t)}}\\
&=\frac{1}{b_{\mathrm{grp}}^2}\Expc{\sum_{i\in\Omega^t}\sqn{\widetilde{\Delta}_i^t(j_i^t)}}\\
&\eqtext{(a)}\frac{1}{b_{\mathrm{grp}}^2}\Expb{\Expb{\sum_{i\in\Omega^t}\sqn{\widetilde{\Delta}_i^t(j_i^t)}}{\mathcal{F}^t,i^t,\Omega^t}}{\mathcal{F}^t}\\
&\eqtext{(b)}\frac{1}{b_{\mathrm{grp}}^2}\Expb{\sum_{i\in\Omega^t}\Expb{\sqn{\widetilde{\Delta}_i^t(j_i^t)}}{\mathcal{F}^t,i^t,\Omega^t}}{\mathcal{F}^t}\\
&\eqtext{(c)}\frac{1}{b_{\mathrm{grp}}^2}\Expb{\sum_{i\in\Omega^t}\frac{1}{m}\sum_{j=1}^m \sqn{\widetilde{\Delta}_i^t(j)}}{\mathcal{F}^t}\\
&\eqtext{(d)}\frac{1}{b_{\mathrm{grp}}^2}\sum_{i=1}^n \Prob(i\in\Omega^t)\cdot \frac{1}{m}\sum_{j=1}^m \sqn{\widetilde{\Delta}_i^t(j)}\\
&\eqtext{(e)}\frac{b_{\mathrm{grp}}-1}{b_{\mathrm{grp}}^2nm}\sum_{i=1}^n\sum_{j=1}^m \sqn{\widetilde{\Delta}_i^t(j)}\\
&=\frac{b_{\mathrm{grp}}-1}{b_{\mathrm{grp}}^2nm}\sum_{i=1}^n\sum_{j=1}^m
\sqn{\nabla f_{i,j}(x^{t+1})-\nabla f_i(x^{t+1})-\nabla f_{i,j}(x^{t})+\nabla f_i(x^{t})}.
\end{align*}
Here, (a)--(b) use the tower property and linearity of conditional expectation, (c) uses that conditional on
$(\mathcal{F}^t,i^t,\Omega^t)$ each $j_i^t$ is uniform on $[m]$, (d) uses
$\Expb{\sum_{i\in\Omega^t}a_i}{\mathcal{F}^t}=\sum_{i=1}^n a_i\,\Prob(i\in\Omega^t)$ for $\mathcal{F}^t$-measurable $(a_i)$,
and (e) uses $\Prob(i\in\Omega^t)=\frac{b_{\mathrm{grp}}-1}{n}$.

\textbf{Noise term for the anchor index $i^t$.}

Similarly, by conditioning first on $i^t$ and then on $j_{i^t}^t$, we obtain
\begin{align*}
&\frac{(n-b_{\mathrm{grp}})^2}{b_{\mathrm{grp}}^2n^2}\Expc{ \sqn{ \nabla f_{i^t,j_{i^t}^t}(x^{t+1})-\nabla f_{i^t}(x^{t+1})-\nabla f_{i^t,j_{i^t}^t}(x^{t})+\nabla f_{i^t}(x^t)}}\\
&=\frac{(n-b_{\mathrm{grp}})^2}{b_{\mathrm{grp}}^2n^2}\Expc{\sqn{\widetilde{\Delta}_{i^t}^t(j_{i^t}^t)}}\\
&\eqtext{(a)}\frac{(n-b_{\mathrm{grp}})^2}{b_{\mathrm{grp}}^2n^2}\Expb{\Expb{\sqn{\widetilde{\Delta}_{i^t}^t(j_{i^t}^t)}}{\mathcal{F}^t,i^t}}{\mathcal{F}^t}\\
&\eqtext{(b)}\frac{(n-b_{\mathrm{grp}})^2}{b_{\mathrm{grp}}^2n^2}\Expb{\frac{1}{m}\sum_{j=1}^m \sqn{\widetilde{\Delta}_{i^t}^t(j)}}{\mathcal{F}^t}\\
&\eqtext{(c)}\frac{(n-b_{\mathrm{grp}})^2}{b_{\mathrm{grp}}^2n^2}\cdot \frac{1}{n}\sum_{i=1}^n \frac{1}{m}\sum_{j=1}^m \sqn{\widetilde{\Delta}_i^t(j)}\\
&=\frac{(n-b_{\mathrm{grp}})^2}{b_{\mathrm{grp}}^2n^3m}\sum_{i=1}^n\sum_{j=1}^m \sqn{\widetilde{\Delta}_i^t(j)}\\
&=\frac{(n-b_{\mathrm{grp}})^2}{b_{\mathrm{grp}}^2n^3m}\sum_{i=1}^n\sum_{j=1}^m
\sqn{\nabla f_{i,j}(x^{t+1})-\nabla f_i(x^{t+1})-\nabla f_{i,j}(x^{t})+\nabla f_i(x^{t})}.
\end{align*}
Here, (a) is the tower property, (b) uses that conditional on $(\mathcal{F}^t,i^t)$ the index $j_{i^t}^t$ is uniform on $[m]$,
and (c) uses that $i^t$ is uniform on $[n]$ and independent of $\mathcal{F}^t$.

Combining these two identities yields
\begin{align*}
&\frac{1}{b_{\mathrm{grp}}^2}\Expc{ \sum_{i\in \Omega^t}\sqn{ \nabla f_{i,j_i^t}(x^{t+1})-\nabla f_i(x^{t+1})-\nabla f_{i,j_i^t}(x^{t})+\nabla f_i(x^t)}} \\
&\quad+\frac{(n-b_{\mathrm{grp}})^2}{b_{\mathrm{grp}}^2n^2}\Expc{ \sqn{ \nabla f_{i^t,j_{i^t}^t}(x^{t+1})-\nabla f_{i^t}(x^{t+1})-\nabla f_{i^t,j_{i^t}^t}(x^{t})+\nabla f_{i^t}(x^t)}}\\
&=\frac{b_{\mathrm{grp}}-1}{b_{\mathrm{grp}}^2nm}\sum_{i=1}^n\sum_{j=1}^m\sqn{\nabla f_{i,j}(x^{t+1})-\nabla f_i(x^{t+1})-\nabla f_{i,j}(x^{t})+\nabla f_i(x^{t})} \\
&\quad+\frac{(n-b_{\mathrm{grp}})^2}{b_{\mathrm{grp}}^2n^3m}\sum_{i=1}^n\sum_{j=1}^m\sqn{\nabla f_{i,j}(x^{t+1})-\nabla f_i(x^{t+1})-\nabla f_{i,j}(x^{t})+\nabla f_i(x^{t})}\\
&=\frac{n^2-2n+b_{\mathrm{grp}}}{b_{\mathrm{grp}}n^3m}\sum_{i=1}^n\sum_{j=1}^m\sqn{\nabla f_{i,j}(x^{t+1})-\nabla f_i(x^{t+1})-\nabla f_{i,j}(x^{t})+\nabla f_i(x^{t})}.
\end{align*}
Finally, applying the internal similarity condition (with parameter $\delta_2$) to $(x^{t+1},x^t)$ gives
\begin{align*}
\frac{n^2-2n+b_{\mathrm{grp}}}{b_{\mathrm{grp}}n^3m}\sum_{i=1}^n\sum_{j=1}^m\sqn{\nabla f_{i,j}(x^{t+1})-\nabla f_i(x^{t+1})-\nabla f_{i,j}(x^{t})+\nabla f_i(x^{t})}
\leq\frac{n^2-2n+b_{\mathrm{grp}}}{b_{\mathrm{grp}}n^2}\delta_2^2\sqn{x^{t+1}-x^t}.
\end{align*}

Furthermore, we bound the remaining conditional second moment term by separating the two contributions inside the squared norm. Define
\begin{align*}
U^t \eqdef \frac{1}{n}\left(\nabla f_{i^t}(x^{t})-\nabla f(x^{t})-g_{i^t}^t+g^t\right),
\qquad
V^t \eqdef \frac{1}{b_{\mathrm{grp}}} \sum_{i\in \tilde{\Omega}^t}\left( \nabla f_{i}(x^{t+1})-\nabla f(x^{t+1})-\nabla f_{i}(x^{t})+\nabla f(x^{t})\right).
\end{align*}
Then, by the elementary inequality $\sqn{u+v}\leq 2\sqn{u}+2\sqn{v}$, we have
\begin{align*}
&\mathbb{E}\!\left[\left\|U^t+V^t\right\|^2\,\Big|\,\mathcal{F}^t\right]
\leq 2\,\Expc{\sqn{U^t}}+2\,\Expc{\sqn{V^t}}\\
&=\frac{2}{n^2}\Expc{\sqn{\nabla f_{i^t}(x^{t})-\nabla f(x^{t})-g_{i^t}^t+g^t}}\\
&\quad+2\Expc{\sqn{\frac{1}{b_{\mathrm{grp}}} \sum_{i\in \tilde{\Omega}^t}\left( \nabla f_{i}(x^{t+1})-\nabla f(x^{t+1})-\nabla f_{i}(x^{t})+\nabla f(x^{t})\right)}}.
\end{align*}

We now bound the second term $\Expc{\sqn{V^t}}$ using the variance formula for uniform sampling without replacement.
Apply Lemma~\ref{lem:prop1} with
\begin{align*}
N\eqdef n,\qquad B\eqdef b_{\mathrm{grp}},\qquad
v_i \eqdef \nabla f_{i}(x^{t+1})-\nabla f(x^{t+1})-\nabla f_{i}(x^{t})+\nabla f(x^{t}),
\qquad
v \eqdef \frac{1}{n}\sum_{i=1}^n v_i = 0.
\end{align*}
Since $\tilde{\Omega}^t$ is a uniformly random subset of $[n]$ of size $b_{\mathrm{grp}}$, Lemma~\ref{lem:prop1} yields
\begin{align*}
\Expc{\sqn{\frac{1}{b_{\mathrm{grp}}}\sum_{i\in\tilde{\Omega}^t} v_i}}
=\frac{n-b_{\mathrm{grp}}}{b_{\mathrm{grp}}n(n-1)}\sum_{i=1}^n \sqn{v_i}.
\end{align*}
Substituting back $v_i$ gives
\begin{align*}
&\Expc{\sqn{\frac{1}{b_{\mathrm{grp}}}\sum_{i\in\tilde{\Omega}^t}\left( \nabla f_{i}(x^{t+1})-\nabla f(x^{t+1})-\nabla f_{i}(x^{t})+\nabla f(x^{t})\right)}}\\
&= \frac{n-b_{\mathrm{grp}}}{b_{\mathrm{grp}}n(n-1)}\sum_{i=1}^n \sqn{ \nabla f_{i}(x^{t+1})-\nabla f(x^{t+1})-\nabla f_{i}(x^{t})+\nabla f(x^{t})}.
\end{align*}
Combining the previous two displays with the bound derived above for the inner-noise terms yields
\begin{align*}
\Expc{ \sqn{g^{t+1}-\nabla f(x^{t+1})}}
&\leq\frac{(n-1)^2}{n^2} \sqn{g^t-\nabla f(x^{t})}
+\frac{2}{n^2}\Expc{\sqn{\nabla f_{i^t}(x^{t})-\nabla f(x^{t})-g_{i^t}^t+g^t}}\\
&\quad+ \frac{2(n-b_{\mathrm{grp}})}{b_{\mathrm{grp}}n(n-1)}\sum_{i=1}^n \sqn{ \nabla f_{i}(x^{t+1})-\nabla f(x^{t+1})-\nabla f_{i}(x^{t})+\nabla f(x^{t})}\\
&\quad+\frac{n^2-2n+b_{\mathrm{grp}}}{b_{\mathrm{grp}}n^2}\delta_2^2\sqn{x^{t+1}-x^t}.
\end{align*}
Next, since $i^t$ is uniform on $[n]$ and independent of $\mathcal{F}^t$, we have
\begin{align*}
\Expc{\sqn{\nabla f_{i^t}(x^{t})-\nabla f(x^{t})-g_{i^t}^t+g^t}}
&=\frac{1}{n}\sum_{i=1}^n \sqn{\nabla f_{i}(x^{t})-\nabla f(x^{t})-g_{i}^t+g^t}\\
&=\frac{1}{n}\sum_{i=1}^n \sqn{g_i^t-\nabla f_{i}(x^{t})}
-\sqn{g^t-\nabla f(x^{t})}.
\end{align*}

Substituting this identity and using the similarity bound (with parameter $\delta_1$) gives
\begin{align*}
&\Expc{ \sqn{g^{t+1}-\nabla f(x^{t+1})}}\\
&\leq\frac{(n-1)^2}{n^2} \sqn{g^t-\nabla f(x^{t})}
+\frac{2}{n^3}\sum_{i=1}^n\sqn{g_i^t-\nabla f_{i}(x^{t})}
-\frac{2}{n^2}\sqn{g^t-\nabla f(x^{t})}\\
&\quad+ \frac{2(n-b_{\mathrm{grp}})}{b_{\mathrm{grp}}n(n-1)}\sum_{i=1}^n \sqn{ \nabla f_{i}(x^{t+1})-\nabla f(x^{t+1})-\nabla f_{i}(x^{t})+\nabla f(x^{t})}\\
&\quad+\frac{n^2-2n+b_{\mathrm{grp}}}{b_{\mathrm{grp}}n^2}\delta_2^2\sqn{x^{t+1}-x^t}\\
&\leq\frac{(n-1)^2}{n^2} \sqn{g^t-\nabla f(x^{t})}
+\frac{2}{n^3}\sum_{i=1}^n\sqn{g_i^t-\nabla f_{i}(x^{t})}
-\frac{2}{n^2}\sqn{g^t-\nabla f(x^{t})}\\
&\quad+ \frac{2(n-b_{\mathrm{grp}})}{b_{\mathrm{grp}}(n-1)}\delta_1^2\sqn{x^{t+1}-x^t}
+\frac{n^2-2n+b_{\mathrm{grp}}}{b_{\mathrm{grp}}n^2}\delta_2^2\sqn{x^{t+1}-x^t}\\
&\leq\frac{n-2}{n} \sqn{g^t-\nabla f(x^{t})}+\frac{2}{n^3}\sum_{i=1}^n\sqn{g_i^t-\nabla f_{i}(x^{t})}\\
&\quad+ \frac{2(n-b_{\mathrm{grp}})}{b_{\mathrm{grp}}(n-1)}\delta_1^2\sqn{x^{t+1}-x^t}+\frac{n^2-2n+b_{\mathrm{grp}}}{b_{\mathrm{grp}}n^2}\delta_2^2\sqn{x^{t+1}-x^t}.
\end{align*}
\end{proof}

\subsubsection*{Proof of Theorem~\ref{theo2}}

The following theorem is the general nonconvex convergence guarantee for Algorithm~\ref{algsilage2a}.
The proof uses Lemma~\ref{lem:silage_variance_bounds_ngtm} to choose a Lyapunov function whose estimator-error terms match the one-step descent recursion.

\silageNgtMNC*

\begin{proof}[Proof of Theorem~\ref{theo2}]
We rely on the variance bounds established in Lemma~\ref{lem:silage_variance_bounds_ngtm}.
Specifically, item~2 of the lemma provides the bound for the global estimator error:
\begin{align*}
&\Expc{\sqn{g^{t+1}-\nabla f(x^{t+1})}}\\
&\leq\frac{n-2}{n}\sqn{g^t-\nabla f(x^{t})}
+\frac{2}{n^3}\sum_{i=1}^n \sqn{g_i^t-\nabla f_i(x^{t})}\\
&\quad+\frac{2(n-b_{\mathrm{grp}})}{b_{\mathrm{grp}}(n-1)}\delta_1^2\sqn{x^{t+1}-x^t}
+\frac{n^2-2n+b_{\mathrm{grp}}}{b_{\mathrm{grp}}n^2}\delta_2^2\sqn{x^{t+1}-x^t},
\end{align*}
and item~1 provides the bound for the average error:
\begin{align*}
&\Expc{\frac1n\sum_{i=1}^n \sqn{g_i^{t+1}-\nfi{x^{t+1}}}}\\
&\leq\left(1-\frac{1}{2n}\right)\frac1n\sum_{i=1}^n \sqn{g_i^t-\nfi{x^{t}}}\\
&\quad+\frac{2(n-b_{\mathrm{grp}})(n+1)}{n-1}\delta_1^2\sqn{x^{t+1}-x^t}
+\frac{n-2b_{\mathrm{grp}}+b_{\mathrm{grp}}^2}{b_{\mathrm{grp}}n}\delta_2^2\sqn{x^{t+1}-x^t}.
\end{align*}

We also have the descent lemma for $L$-smooth functions:
\begin{align*}
f(x^{t+1})
&\leq f(x^t)
+\frac{\gamma}{2}\sqn{g^{t}-\nabla f(x^t)}
-\frac{\gamma}{2}\sqn{\nabla f(x^t)}\\
&\quad+\left(\frac{L}{2}-\frac{1}{2\gamma}\right)\sqn{x^{t+1}-x^t}.
\end{align*}
We define the Lyapunov function
\begin{align*}
\Psi^t\eqdef f(x^t)-\finf + \frac{\gamma n}{4}\sqn{g^{t}-\nabla f(x^t)}+\frac{\gamma}{n}\sum_{i=1}^n\sqn{g_i^t-\nabla f_{i}(x^{t})}.
\end{align*}

Let us denote the Lyapunov coefficients by $C_1 \eqdef \frac{\gamma n}{4}$ and $C_2\eqdef \frac{\gamma}{n}$.
Taking the conditional expectation of $\Psi^{t+1}$ and substituting the descent lemma together with the two variance bounds above, we obtain
\begin{align*}
&\Expc{\Psi^{t+1}} \\
&=\Expc{f(x^{t+1})}-\finf
+C_1\Expc{\sqn{g^{t+1}-\nabla f(x^{t+1})}}\\
&\quad+C_2\Expc{\sum_{i=1}^n\sqn{g_i^{t+1}-\nabla f_i(x^{t+1})}}\\
&\leq
\underbrace{\begin{aligned}[t]
&f(x^t)-\finf
+\frac{\gamma}{2}\sqn{g^{t}-\nabla f(x^t)}
-\frac{\gamma}{2}\sqn{\nabla f(x^t)}\\
&\quad+\left(\frac{L}{2}-\frac{1}{2\gamma}\right)\sqn{x^{t+1}-x^t}
\end{aligned}}_{\text{Descent lemma}}\\
&\quad +C_1\Bigg[\frac{n-2}{n}\sqn{g^t-\nabla f(x^{t})}
+\frac{2}{n^3}\sum_{i=1}^n \sqn{g_i^t-\nabla f_i(x^{t})}\\
&\qquad\qquad
+\frac{2(n-b_{\mathrm{grp}})}{b_{\mathrm{grp}}(n-1)}\delta_1^2\sqn{x^{t+1}-x^t}
+\frac{n^2-2n+b_{\mathrm{grp}}}{b_{\mathrm{grp}}n^2}\delta_2^2\sqn{x^{t+1}-x^t}\Bigg]\\
&\quad +C_2\Bigg[n\left(1-\frac{1}{2n}\right)
\frac{1}{n}\sum_{i=1}^n \sqn{g_i^t-\nabla f_i(x^{t})}\\
&\qquad\qquad
+n\Bigg(\frac{2(n-b_{\mathrm{grp}})(n+1)}{n-1}\delta_1^2\sqn{x^{t+1}-x^t}\\
&\qquad\qquad\qquad
+\frac{n-2b_{\mathrm{grp}}+b_{\mathrm{grp}}^2}{b_{\mathrm{grp}}n}\delta_2^2\sqn{x^{t+1}-x^t}\Bigg)\Bigg].
\end{align*}
We now group the terms by the error components $\sqn{g^t-\nabla f(x^t)}$, $\sum_{i=1}^n\sqn{g_i^t-\nabla f_i(x^t)}$, and the stepsize $\sqn{x^{t+1}-x^t}$:
\begin{align*}
\Expc{\Psi^{t+1}}
&\leq f(x^t)-\finf-\frac{\gamma}{2}\sqn{\nabla f(x^t)}\\
&\quad+\left(\frac{\gamma}{2}+C_1\frac{n-2}{n}\right)\sqn{g^t-\nabla f(x^t)}\\
&\quad+\left(C_1\frac{2}{n^3}+C_2\left(1-\frac{1}{2n}\right)\right)\\
&\qquad\cdot\sum_{i=1}^n\sqn{g_i^t-\nabla f_i(x^t)}\\
&\quad+\Bigg(\frac{L}{2}-\frac{1}{2\gamma}
+C_1\frac{2(n-b_{\mathrm{grp}})}{b_{\mathrm{grp}}(n-1)}\delta_1^2
+C_2n\frac{2(n-b_{\mathrm{grp}})(n+1)}{n-1}\delta_1^2\\
&\qquad\qquad
+C_1\frac{n^2-2n+b_{\mathrm{grp}}}{b_{\mathrm{grp}}n^2}\delta_2^2
+C_2n\frac{n-2b_{\mathrm{grp}}+b_{\mathrm{grp}}^2}{b_{\mathrm{grp}}n}\delta_2^2\Bigg)\sqn{x^{t+1}-x^t}.
\end{align*}
Substituting $C_1=\frac{\gamma n}{4}$ and $C_2=\frac{\gamma}{n}$ yields the simplifications
\begin{align*}
\frac{\gamma}{2}+C_1\frac{n-2}{n}
&=\frac{\gamma}{2}+\frac{\gamma n}{4}\cdot\frac{n-2}{n}
=\frac{\gamma}{2}+\frac{\gamma(n-2)}{4}
=\frac{\gamma n}{4}=C_1,\\
C_1\frac{2}{n^3}+C_2\left(1-\frac{1}{2n}\right)
&=\frac{\gamma n}{4}\cdot\frac{2}{n^3}+\frac{\gamma}{n}\left(1-\frac{1}{2n}\right)
=\frac{\gamma}{2n^2}+\frac{\gamma}{n}-\frac{\gamma}{2n^2}
=\frac{\gamma}{n}=C_2,
\end{align*}
and hence the coefficients of the error terms match exactly those in $\Psi^t$. Therefore,
\begin{align*}
\Expc{\Psi^{t+1}}
&\leq \Psi^t-\frac{\gamma}{2}\sqn{\nabla f(x^t)}\\
&\quad+\left(\frac{L}{2}-\frac{1}{2\gamma}
+\gamma\left(\frac{n(n-b_{\mathrm{grp}})}{2b_{\mathrm{grp}}(n-1)}
+\frac{2(n-b_{\mathrm{grp}})(n+1)}{n-1}\right)\delta_1^2\right.\\
&\qquad\qquad\left.
+\gamma\left(\frac{n^2-2n+b_{\mathrm{grp}}}{4b_{\mathrm{grp}}n}
+\frac{n-2b_{\mathrm{grp}}+b_{\mathrm{grp}}^2}{b_{\mathrm{grp}}n}\right)\delta_2^2\right)\sqn{x^{t+1}-x^t}.
\end{align*}
Using $b_{\mathrm{grp}}\geq 1$, we have $\frac{n}{2b_{\mathrm{grp}}}\leq \frac{n}{2}$, and thus
\begin{align*}
&\frac{n(n-b_{\mathrm{grp}})}{2b_{\mathrm{grp}}(n-1)}
+\frac{2(n-b_{\mathrm{grp}})(n+1)}{n-1}\\
&\quad=\frac{n-b_{\mathrm{grp}}}{n-1}\left(\frac{n}{2b_{\mathrm{grp}}}+2(n+1)\right)\\
&\quad\leq \frac{(n-b_{\mathrm{grp}})(5n+4)}{2(n-1)}.
\end{align*}
Moreover,
\begin{align*}
\frac{n^2-2n+b_{\mathrm{grp}}}{4b_{\mathrm{grp}}n}+\frac{n-2b_{\mathrm{grp}}+b_{\mathrm{grp}}^2}{b_{\mathrm{grp}}n}
=\frac{n^2+2n-7b_{\mathrm{grp}}+4b_{\mathrm{grp}}^2}{4b_{\mathrm{grp}}n}.
\end{align*}
Therefore,
\begin{align*}
\Expc{\Psi^{t+1}}
&\leq \Psi^t-\frac{\gamma}{2}\sqn{\nabla f(x^t)} \\
&\quad+\left(\frac{L}{2}-\frac{1}{2\gamma}
+\frac{\gamma}{2}\,\delta_1^2\frac{(n-b_{\mathrm{grp}})(5n+4)}{n-1}\right.\\
&\qquad\qquad\left.
+\frac{\gamma}{2}\,\delta_2^2\frac{n^2+2n-7b_{\mathrm{grp}}+4b_{\mathrm{grp}}^2}{2b_{\mathrm{grp}}n}\right)
\sqn{x^{t+1}-x^t}.
\end{align*}

Hence, the last term is nonpositive whenever
\begin{align*}
L-\frac{1}{\gamma}
+\gamma\left(
\delta_1^2\frac{(n-b_{\mathrm{grp}})(5n+4)}{n-1}
+\delta_2^2\frac{n^2+2n-7b_{\mathrm{grp}}+4b_{\mathrm{grp}}^2}{2b_{\mathrm{grp}}n}
\right)
&\leq 0,
\end{align*}
that is, whenever
\begin{align*}
\left(
\delta_1^2\frac{(n-b_{\mathrm{grp}})(5n+4)}{n-1}
+\delta_2^2\frac{n^2+2n-7b_{\mathrm{grp}}+4b_{\mathrm{grp}}^2}{2b_{\mathrm{grp}}n}
\right)\gamma^2
+L\gamma
\leq 1.
\end{align*}
By Lemma~\ref{lem:quadratic_step_bound}, applied with
\begin{align*}
a &\eqdef
\delta_1^2\frac{(n-b_{\mathrm{grp}})(5n+4)}{n-1}
+\delta_2^2\frac{n^2+2n-7b_{\mathrm{grp}}+4b_{\mathrm{grp}}^2}{2b_{\mathrm{grp}}n},
\qquad
b \eqdef L,
\end{align*}
it therefore suffices to impose
\begin{align*}
\gamma
\leq
\left(
L+\sqrt{
\delta_1^2\frac{(n-b_{\mathrm{grp}})(5n+4)}{n-1}
+\delta_2^2\frac{n^2+2n-7b_{\mathrm{grp}}+4b_{\mathrm{grp}}^2}{2b_{\mathrm{grp}}n}}
\right)^{-1},
\end{align*}
which is exactly the stepsize condition in the theorem.

Under this stepsize condition, we obtain the recursion
\begin{align*}
\Expc{\Psi^{t+1}}\leq \Psi^t-\frac{\gamma}{2}\sqn{\nabla f(x^t)}.
\end{align*}
Unrolling this inequality for $t=0,\ldots,T$ and using $\Psi^{t}\geq 0$ yields
\begin{align*}
\frac{\gamma}{2}\sum_{t=0}^{T}\Exp{\sqn{\nabla f(x^t)}}\leq \Psi^0.
\end{align*}
Finally, by defining $\tilde{x}^T$ as $x^t$ for $t$ chosen uniformly at random in $\{0,\ldots,T\}$, we have
\begin{align*}
\Exp{\sqn{\nabla f(\tilde{x}^T)}}=\frac{1}{T+1}\sum_{t=0}^{T}\Exp{\sqn{\nabla f(x^t)}}
\leq \frac{2\Psi^0}{\gamma(T+1)},
\end{align*}
which concludes the proof.
\end{proof}


The next corollary converts Theorem~\ref{theo2} into iteration and component-gradient complexity with arbitrary initialization.
It then specializes the bound to the natural choice $b_{\mathrm{grp}}=m$, which is admissible in the regime $n>m$.

\silageNgtMArbitraryInit*

\begin{proof}[Proof of Corollary~\ref{cor:silage_ngt_m_arbitrary_init}]
By Theorem~\ref{theo2}, for the chosen value of $b_{\mathrm{grp}}$ and stepsize
\begin{align*}
\gamma
&=
\rb{
L+\sqrt{
\delta_1^2\frac{(n-b_{\mathrm{grp}})(5n+4)}{n-1}
+\delta_2^2\frac{n^2+2n-7b_{\mathrm{grp}}+4b_{\mathrm{grp}}^2}{2b_{\mathrm{grp}}n}}
}^{-1},
\end{align*}
we have
\begin{align*}
\Exp{\sqn{\nf{\tilde{x}^T}}}
&=
\frac{1}{T+1}\sum_{t=0}^{T}\Exp{\sqn{\nf{x^t}}}
\le
\frac{2\Psi_{b_{\mathrm{grp}}}^0}{\gamma(T+1)}.
\end{align*}
Thus, if $T+1\ge 2\Psi_{b_{\mathrm{grp}}}^0/(\gamma\epsilon)$, then $\Exp{\sqn{\nf{\tilde{x}^T}}}\le \epsilon$. Equivalently, the required number of iterations satisfies
\begin{align*}
T
&=
\cO\rb{
\frac{\Psi_{b_{\mathrm{grp}}}^0}{\gamma\epsilon}
}
\\
&=
\cO\rb{
\rb{
L+\sqrt{
\delta_1^2\frac{(n-b_{\mathrm{grp}})(5n+4)}{n-1}
+\delta_2^2\frac{n^2+2n-7b_{\mathrm{grp}}+4b_{\mathrm{grp}}^2}{2b_{\mathrm{grp}}n}}
}
\frac{\Psi_{b_{\mathrm{grp}}}^0}{\epsilon}
}.
\end{align*}

It remains to count component-gradient evaluations. At each iteration of Algorithm~\ref{algsilage2a}, the algorithm computes one full group gradient $\nfi{x^{t+1}}$, which costs $m$ component-gradient evaluations. In addition, the update uses one extra component-gradient evaluation for the selected group and two component-gradient evaluations for each of the other $b_{\mathrm{grp}}-1$ groups in the batch. Hence, for $b_{\mathrm{grp}}<n$, the per-iteration cost is
\begin{align*}
m+1+2(b_{\mathrm{grp}}-1)
&=
m+2b_{\mathrm{grp}}-1
=
\cO\rb{m+b_{\mathrm{grp}}}.
\end{align*}
When $b_{\mathrm{grp}}=n$, the extra component gradient for the selected group is not needed, and the cost is $m+2(b_{\mathrm{grp}}-1)$, which is still $\cO\rb{m+b_{\mathrm{grp}}}$. Therefore, multiplying the iteration bound by the per-iteration cost gives
\begin{align*}
\cO\rb{
(m+b_{\mathrm{grp}})
\rb{
L+\sqrt{
\delta_1^2\frac{(n-b_{\mathrm{grp}})(5n+4)}{n-1}
+\delta_2^2\frac{n^2+2n-7b_{\mathrm{grp}}+4b_{\mathrm{grp}}^2}{2b_{\mathrm{grp}}n}}
}
\frac{\Psi_{b_{\mathrm{grp}}}^0}{\epsilon}
}
\end{align*}
component-gradient evaluations, excluding initialization.

Now specialize to $b_{\mathrm{grp}}=m$. Since $n>m$, this choice is admissible. Define
\begin{align*}
A_1(m)
&\eqdef
\frac{(n-m)(5n+4)}{n-1},
&
A_2(m)
&\eqdef
\frac{n^2+2n-7m+4m^2}{2mn}.
\end{align*}
Then the post-initialization component-gradient complexity is
\begin{align*}
\cO\rb{
m\rb{
L+\sqrt{\delta_1^2A_1(m)+\delta_2^2A_2(m)}
}
\frac{\Psi_m^0}{\epsilon}
}.
\end{align*}
Using $\sqrt{a+b_{\mathrm{grp}}}\le \sqrt a+\sqrt b_{\mathrm{grp}}$, we get
\begin{align*}
m\rb{
L+\sqrt{\delta_1^2A_1(m)+\delta_2^2A_2(m)}
}
&\le
mL+m\delta_1\sqrt{A_1(m)}+m\delta_2\sqrt{A_2(m)}.
\end{align*}
Since $n>m$, we have $n\ge 2$ and
\begin{align*}
\frac{5n+4}{n-1}
&=
5+\frac{9}{n-1}
\le
14.
\end{align*}
Therefore
\begin{align*}
A_1(m)
&\le
14(n-m),
&
m\delta_1\sqrt{A_1(m)}
&=
\cO\rb{m\sqrt{n-m}\,\delta_1}.
\end{align*}
Similarly, using $m\le n$ and $n\ge 2$,
\begin{align*}
n^2+2n-7m+4m^2
&\le
n^2+2n+4m^2
\le
7n^2,
\end{align*}
and hence
\begin{align*}
A_2(m)
&\le
\frac{7n}{2m},
&
m\delta_2\sqrt{A_2(m)}
&=
\cO\rb{\sqrt{nm}\,\delta_2}.
\end{align*}
Combining the last three displays gives
\begin{align*}
m\rb{
L+\sqrt{\delta_1^2A_1(m)+\delta_2^2A_2(m)}
}
&=
\cO\rb{
mL+m\sqrt{n-m}\,\delta_1+\sqrt{nm}\,\delta_2
}.
\end{align*}
Thus, with $b_{\mathrm{grp}}=m$, the expected component-gradient complexity is
\begin{align*}
\cO\rb{
\rb{
mL+m\sqrt{n-m}\,\delta_1+\sqrt{nm}\,\delta_2
}
\frac{\Psi_m^0}{\epsilon}
}.
\end{align*}

Now, we assume exact initialization.


For the choice $b_{\mathrm{grp}}=m$, the stepsize condition in Theorem~\ref{theo2} gives
\begin{align*}
\gamma
&=
\rb{
L+\sqrt{
\delta_1^2\frac{(n-b_{\mathrm{grp}})(5n+4)}{n-1}
+\delta_2^2\frac{n^2+2n-7b_{\mathrm{grp}}+4b_{\mathrm{grp}}^2}{2b_{\mathrm{grp}}n}
}
}^{-1}
\\
&=
\rb{
L+\sqrt{
\delta_1^2\frac{(n-m)(5n+4)}{n-1}
+\delta_2^2\frac{n^2+2n-7m+4m^2}{2mn}
}
}^{-1}.
\end{align*}

The exact initialization $g_i^0=\nfi{x^0}$ for every $i\in[n]$ implies
\begin{align*}
g^0
&=
\frac1n\sum_{i=1}^n g_i^0
=
\frac1n\sum_{i=1}^n \nfi{x^0}
=
\nf{x^0}.
\end{align*}
Therefore, both initialization-error terms in the Lyapunov quantity from Theorem~\ref{theo2} vanish:
\begin{align*}
\sqn{g^0-\nf{x^0}}
&=0,
&
\frac1n\sum_{i=1}^n\sqn{g_i^0-\nfi{x^0}}
&=0.
\end{align*}
Hence, for $b_{\mathrm{grp}}=m$,
\begin{align*}
\Psi_m^0
&=
f(x^0)-\finf
=
\Delta_0.
\end{align*}

Applying Corollary~\ref{cor:silage_ngt_m_arbitrary_init} with $b_{\mathrm{grp}}=m$ and $\Psi_m^0=\Delta_0$ gives the post-initialization iteration complexity
\begin{align*}
T
&=
\cO\rb{
\rb{
L+\sqrt{
\delta_1^2\frac{(n-m)(5n+4)}{n-1}
+\delta_2^2\frac{n^2+2n-7m+4m^2}{2mn}
}
}
\frac{\Delta_0}{\epsilon}
}.
\end{align*}
Using $\sqrt{a+b_{\mathrm{grp}}}\le \sqrt a+\sqrt b_{\mathrm{grp}}$, we obtain
\begin{align*}
T
&=
\cO\rb{
\rb{
L
+\delta_1\sqrt{\frac{(n-m)(5n+4)}{n-1}}
+\delta_2\sqrt{\frac{n^2+2n-7m+4m^2}{2mn}}
}
\frac{\Delta_0}{\epsilon}
}.
\end{align*}
Since $n>m$, we have $n\ge 2$ and
\begin{align*}
\frac{5n+4}{n-1}
&=
5+\frac{9}{n-1}
\le
14,
\end{align*}
and hence
\begin{align*}
\sqrt{\frac{(n-m)(5n+4)}{n-1}}
&=
\cO\rb{\sqrt{n-m}}.
\end{align*}
Moreover, using $m\le n$ and $n\ge 2$,
\begin{align*}
n^2+2n-7m+4m^2
&\le
n^2+2n+4m^2
\le
7n^2,
\end{align*}
so that
\begin{align*}
\sqrt{\frac{n^2+2n-7m+4m^2}{2mn}}
&=
\cO\rb{\sqrt{\frac{n}{m}}}.
\end{align*}
Combining the last displays yields
\begin{align*}
T
&=
\cO\rb{
\rb{
L+\sqrt{n-m}\,\delta_1+\sqrt{\frac{n}{m}}\,\delta_2
}
\frac{\Delta_0}{\epsilon}
}.
\end{align*}

The same Corollary~\ref{cor:silage_ngt_m_arbitrary_init} gives the expected post-initialization component-gradient complexity
\begin{align*}
\cO\rb{
\rb{
mL+m\sqrt{n-m}\,\delta_1+\sqrt{nm}\,\delta_2
}
\frac{\Delta_0}{\epsilon}
}.
\end{align*}

It remains to add the initialization cost. Computing $g_i^0=\nfi{x^0}$ requires evaluating the full group gradient
\begin{align*}
\nfi{x^0}
&=
\frac1m\sum_{j=1}^m \nabla f_{i,j}(x^0),
\end{align*}
which costs $m$ component-gradient evaluations for a fixed group $i$. Repeating this for all $i\in[n]$ costs
\begin{align*}
\sum_{i=1}^n m
&=
mn
\end{align*}
component-gradient evaluations. Therefore, the total expected component-gradient complexity, including initialization, is
\begin{align*}
\cO\rb{
nm +
\rb{
mL+m\sqrt{n-m}\,\delta_1+\sqrt{nm}\,\delta_2
}
\frac{\Delta_0}{\epsilon}
}.
\end{align*}

\end{proof}

\subsection{Analysis under the Polyak--\L{}ojasiewicz condition}\label{sec:app-pl-ngt-m}

The P\L{} argument for the regime $n>m$ is parallel.
The only change is in the constant controlling the Lyapunov recursion: both similarity parameters now enter through the batch-dependent quantity $\Xi_{n,b_{\mathrm{grp}}}$ below.
Under Assumption~\ref{ass:pl_condition}, this gives an $O\left(\cdot\log\frac{1}{\epsilon}\right)$ dependence for the objective gap, rather than the $O\left(\cdot\frac{1}{\epsilon}\right)$ dependence of the general nonconvex stationarity result.
As above, the theorem is stated globally, while the proof only invokes P\L{} along the iterates; local P\L{} regions are therefore sufficient whenever the trajectory stays inside them.

The next theorem is the P\L{} analogue of Theorem~\ref{theo2}.
The quantity $\Xi_{n,b_{\mathrm{grp}}}$ summarizes the batch-size-dependent similarity terms that enter the P\L{} Lyapunov recursion.

\begin{restatable}[Linear convergence under P\L, case $n>m$]{theorem}{silageNgtMPL}
\label{theo:silage_pl_ngtm}
Let Assumptions~\ref{ass:lower-bounded}, \ref{ass:f-smooth}, \ref{ass:delta1}, \ref{ass:delta2}, and~\ref{ass:pl_condition} hold,
where Assumption~\ref{ass:pl_condition} holds with parameter $\mu>0$.
Run Algorithm~\ref{algsilage2a} with batch size $b_{\mathrm{grp}}\in[n]$, arbitrary initial estimators
$g_1^0,\ldots,g_n^0\in\R^d$, and $g^0\eqdef \frac1n\sum_{i=1}^n g_i^0$.
Define
\begin{align}
\Xi_{n,b_{\mathrm{grp}}}
\eqdef
\frac{2(n-b_{\mathrm{grp}})}{n-1}\rb{\frac{n}{3b_{\mathrm{grp}}}+4(n+1)}\delta_1^2
+\frac{n^2+10n-23b_{\mathrm{grp}}+12b_{\mathrm{grp}}^2}{3b_{\mathrm{grp}}n}\delta_2^2.
\label{eq:Xi_nb_def}
\end{align}
Let the stepsize be set as
$0<\gamma\le \min\left\{\rb{L+\sqrt{2\Xi_{n,b_{\mathrm{grp}}}}}^{-1},\,\frac{1}{3\mu n}\right\}$.
Then, for every $T\ge 0$,
\begin{align}
\Exp{f(x^T)-\fstar}
\le
(1-\gamma\mu)^T\rb{\Psi_{\mathrm{PL}}^0},
\label{eq:silage_pl_rate_ngtm}
\end{align}
where
\begin{equation*}
\Psi_{\mathrm{PL}}^0
\eqdef
f(x^0)-\fstar
+\frac{\gamma n}{3}\sqn{g^0-\nf{x^0}}
+4\gamma\frac1n\sum_{i=1}^n\sqn{g_i^0-\nfi{x^0}}.
\end{equation*}
\end{restatable}

\begin{proof}[Proof of Theorem~\ref{theo:silage_pl_ngtm}]
Let us fix constants $\nu,\rho\in[0,\infty)$ that we will define later.
We start from the descent lemma for $L$-smooth functions, applied with $x^{t+1}=x^t-\gamma g^t$:
\begin{align}
f(x^{t+1})
&\leq f(x^t) + \frac{\gamma}{2}\sqn{g^{t}-\nabla f(x^t)}-\frac{\gamma}{2}\sqn{\nabla f(x^t)} \notag\\
&\quad + \left(\frac{L}{2}-\frac{1}{2\gamma}\right)\sqn{x^{t+1}-x^t}.
\label{eq:descent_pl_silage_ngtm}
\end{align}
Taking expectation and adding $\nu\,\Exp{\sqn{g^{t+1}-\nabla f(x^{t+1})}}$ and
$\rho\,\Exp{\frac{1}{n}\sum_{i=1}^n\sqn{g_i^{t+1}-\nabla f_i(x^{t+1})}}$ to both sides, and then using Lemma~\ref{lem:silage_variance_bounds_ngtm}, we obtain
\begin{align*}
&\Exp{f(x^{t+1})}
+\nu\,\Exp{\sqn{g^{t+1}-\nabla f(x^{t+1})}}\\
&\quad+\rho\,\Exp{\frac{1}{n}\sum_{i=1}^n\sqn{g_i^{t+1}-\nabla f_i(x^{t+1})}}\\
&\leq \Exp{f(x^t)}
-\frac{\gamma}{2}\Exp{\sqn{\nabla f(x^t)}}
+\frac{\gamma}{2}\Exp{\sqn{g^{t}-\nabla f(x^t)}}
+\left(\frac{L}{2}-\frac{1}{2\gamma}\right)\Exp{\sqn{x^{t+1}-x^t}}\\
&\quad+\nu\,\Exp{\begin{aligned}[t]
&\frac{n-2}{n}\sqn{g^t-\nabla f(x^{t})}
+\frac{2}{n^3}\sum_{i=1}^n \sqn{g_i^t-\nabla f_i(x^{t})}\\
&\quad+\frac{2(n-b_{\mathrm{grp}})}{b_{\mathrm{grp}}(n-1)}\delta_1^2\sqn{x^{t+1}-x^t}
+\frac{n^2-2n+b_{\mathrm{grp}}}{b_{\mathrm{grp}}n^2}\delta_2^2\sqn{x^{t+1}-x^t}
\end{aligned}}\\
&\quad+\rho\,\Exp{\begin{aligned}[t]
&\left(1-\frac{1}{2n}\right)\frac{1}{n}\sum_{i=1}^n
\sqn{g_i^t-\nabla f_i(x^{t})}\\
&\quad+\frac{2(n-b_{\mathrm{grp}})(n+1)}{n-1}\delta_1^2\sqn{x^{t+1}-x^t}
+\frac{n-2b_{\mathrm{grp}}+b_{\mathrm{grp}}^2}{b_{\mathrm{grp}}n}\delta_2^2\sqn{x^{t+1}-x^t}
\end{aligned}}.
\end{align*}
Rearranging the terms yields
\begin{align}
&\Exp{f(x^{t+1})}
+\nu\,\Exp{\sqn{g^{t+1}-\nabla f(x^{t+1})}}
\notag\\
&\quad+\rho\,\Exp{\frac{1}{n}\sum_{i=1}^n\sqn{g_i^{t+1}-\nabla f_i(x^{t+1})}}\notag\\
&\leq \Exp{f(x^t)}
-\frac{\gamma}{2}\Exp{\sqn{\nabla f(x^t)}} \notag\\
&\quad+\left(\frac{\gamma}{2}+\nu\frac{n-2}{n}\right)\Exp{\sqn{g^{t}-\nabla f(x^{t})}}\notag\\
&\quad+\left(\rho\left(1-\frac{1}{2n}\right)+\nu\frac{2}{n^2}\right)
\Exp{\frac{1}{n}\sum_{i=1}^n\sqn{g_i^{t}-\nabla f_i(x^{t})}}\notag\\
&\quad+\Bigg(\frac{L}{2}-\frac{1}{2\gamma}
+\nu\frac{2(n-b_{\mathrm{grp}})}{b_{\mathrm{grp}}(n-1)}\delta_1^2
+\nu\frac{n^2-2n+b_{\mathrm{grp}}}{b_{\mathrm{grp}}n^2}\delta_2^2
+\rho\frac{2(n-b_{\mathrm{grp}})(n+1)}{n-1}\delta_1^2
+\rho\frac{n-2b_{\mathrm{grp}}+b_{\mathrm{grp}}^2}{b_{\mathrm{grp}}n}\delta_2^2\Bigg)\notag\\
&\qquad\cdot\Exp{\sqn{x^{t+1}-x^t}}.
\label{eq:pl_silage_ngtm_prechoice}
\end{align}

We now choose
\begin{align}
\label{eq:nu_rho_choice_ngtm}
\nu \eqdef \frac{\gamma n}{3},
\qquad
\rho \eqdef 4\gamma.
\end{align}
With this choice, the coefficient of $\Exp{\sqn{x^{t+1}-x^t}}$ in~\eqref{eq:pl_silage_ngtm_prechoice} becomes
\begin{align*}
&\frac{L}{2}-\frac{1}{2\gamma}
+\nu\frac{2(n-b_{\mathrm{grp}})}{b_{\mathrm{grp}}(n-1)}\delta_1^2
+\nu\frac{n^2-2n+b_{\mathrm{grp}}}{b_{\mathrm{grp}}n^2}\delta_2^2
+\rho\frac{2(n-b_{\mathrm{grp}})(n+1)}{n-1}\delta_1^2
+\rho\frac{n-2b_{\mathrm{grp}}+b_{\mathrm{grp}}^2}{b_{\mathrm{grp}}n}\delta_2^2\\
&=\frac{L}{2}-\frac{1}{2\gamma}
+\gamma\Bigg(
\frac{2(n-b_{\mathrm{grp}})}{n-1}\left(\frac{n}{3b_{\mathrm{grp}}}+4(n+1)\right)\delta_1^2\\
&\qquad\qquad
+\frac{n^2+10n-23b_{\mathrm{grp}}+12b_{\mathrm{grp}}^2}{3b_{\mathrm{grp}}n}\delta_2^2
\Bigg)\\
&=\frac{L}{2}-\frac{1}{2\gamma}+\gamma\,\Xi_{n,b_{\mathrm{grp}}},
\end{align*}
where $\Xi_{n,b_{\mathrm{grp}}}$ is defined in~\eqref{eq:Xi_nb_def}.

Thus, the term is nonpositive whenever
\begin{align*}
L-\frac{1}{\gamma}+2\gamma\,\Xi_{n,b_{\mathrm{grp}}}\leq 0,
\end{align*}
that is, whenever
\begin{align*}
2\Xi_{n,b_{\mathrm{grp}}}\gamma^2+L\gamma\leq 1.
\end{align*}
By Lemma~\ref{lem:quadratic_step_bound}, applied with
\begin{align*}
a &\eqdef 2\Xi_{n,b_{\mathrm{grp}}},
\qquad
b \eqdef L,
\end{align*}
it therefore suffices to impose
\begin{align*}
\gamma \leq \left(L+\sqrt{2\Xi_{n,b_{\mathrm{grp}}}}\right)^{-1},
\end{align*}
which is exactly the first constraint in the stepsize condition of Theorem~\ref{theo:silage_pl_ngtm}.


Next, with the choice~\eqref{eq:nu_rho_choice_ngtm}, we claim that if $\gamma \leq \frac{1}{3\mu n}$, then
\begin{align}
\frac{\gamma}{2}+\nu\frac{n-2}{n} &\leq (1-\gamma\mu)\nu,
\label{eq:pl_coeff1_ngtm}\\
\rho\left(1-\frac{1}{2n}\right)+\nu\frac{2}{n^2} &\leq (1-\gamma\mu)\rho.
\label{eq:pl_coeff2_ngtm}
\end{align}
Indeed, for~\eqref{eq:pl_coeff1_ngtm},
\begin{align*}
\frac{\gamma}{2}+\nu\frac{n-2}{n}
&=\frac{\gamma}{2}+\nu-\frac{2\nu}{n}
=\frac{\gamma}{2}+\frac{\gamma n}{3}-\frac{2\gamma}{3}
=\frac{\gamma n}{3}-\frac{\gamma}{6}
=\nu-\frac{\gamma}{6}\\
&\leq \nu-\gamma^2\mu\,\frac{n}{3}
=(1-\gamma\mu)\nu,
\end{align*}
where the inequality holds whenever $\gamma\mu\leq \frac{1}{2n}$, which is implied by $\gamma\leq \frac{1}{3\mu n}$.
For~\eqref{eq:pl_coeff2_ngtm}, we have
\begin{align*}
\rho\left(1-\frac{1}{2n}\right)+\nu\frac{2}{n^2}
&=4\gamma\left(1-\frac{1}{2n}\right)+\frac{\gamma n}{3}\cdot\frac{2}{n^2}\\
&=4\gamma-\frac{2\gamma}{n}+\frac{2\gamma}{3n}
=4\gamma-\frac{4\gamma}{3n}
=\rho\left(1-\frac{1}{3n}\right)\\
&\leq \rho(1-\gamma\mu)
=(1-\gamma\mu)\rho,
\end{align*}
where the inequality is equivalent to $\gamma\mu\leq \frac{1}{3n}$.

Substituting~\eqref{eq:nu_rho_choice_ngtm},~\eqref{eq:pl_coeff1_ngtm}, and~\eqref{eq:pl_coeff2_ngtm}
into~\eqref{eq:pl_silage_ngtm_prechoice}, we obtain
\begin{align}
&\Exp{f(x^{t+1})}
+ \frac{\gamma n}{3}\Exp{\sqn{g^{t+1}-\nabla f(x^{t+1})}}
+ 4\gamma\,\Exp{\frac{1}{n}\sum_{i=1}^n\sqn{g_i^{t+1}-\nabla f_i(x^{t+1})}}\notag\\
&\leq \Exp{f(x^{t})}
-\frac{\gamma}{2}\Exp{\sqn{\nabla f(x^t)}} \notag\\
&\quad + (1-\gamma\mu)\left(
\frac{\gamma n}{3}\Exp{\sqn{g^{t}-\nabla f(x^{t})}}
+ 4\gamma\,\Exp{\frac{1}{n}\sum_{i=1}^n\sqn{g_i^{t}-\nabla f_i(x^{t})}}
\right).
\label{eq:pl_silage_recursion_ngtm}
\end{align}

Finally, define the nonnegative sequence
\begin{align}
\label{eq:psi_pl_silage_ngtm}
\Psi^t \eqdef \frac{n}{3}\Exp{\sqn{g^{t}-\nabla f(x^{t})}}
+ 4\,\Exp{\frac{1}{n}\sum_{i=1}^n\sqn{g_i^{t}-\nabla f_i(x^{t})}}.
\end{align}
Then~\eqref{eq:pl_silage_recursion_ngtm} reads
\begin{align*}
\Exp{f(x^{t+1})} + \gamma \Psi^{t+1}
\leq \Exp{f(x^{t})} -\frac{\gamma}{2}\Exp{\sqn{\nabla f(x^t)}} + (1-\gamma\mu)\gamma \Psi^{t}.
\end{align*}
In view of Lemma~\ref{lemma:good_recursion_pl} (with $C=0$), we conclude that for every $T\geq 0$,
\begin{align*}
\Exp{f(x^{T})-\fstar} \leq (1-\gamma\mu)^{T}\left((f(x^0)-\fstar)+\gamma\Psi^0\right).
\end{align*}
Substituting~\eqref{eq:psi_pl_silage_ngtm} at $t=0$ yields~\eqref{eq:silage_pl_rate_ngtm}.
\end{proof}


\subsection{Optimizing the batch size $b_{\mathrm{grp}}$}\label{sec:app-opt-bgrp}
The preceding convergence bound in the regime $n>m$ depends on the minibatch size $b_{\mathrm{grp}}$.
This raises a practical question: which value of $b_{\mathrm{grp}}$ should be used to minimize the total number of component-gradient evaluations?
In this section, we optimize the bound from Theorem~\ref{theo2} with respect to $b_{\mathrm{grp}}$.

We use the upper bound on the total number of component-gradient evaluations as the criterion for choosing $b_{\mathrm{grp}}$.
Assume that $g_i^0=\nabla f_i(x^0)$ for all $i\in[n]$, so that
\begin{align*}
\Psi^0=\Delta_0 \eqdef f(x^0)-\finf.
\end{align*}
Let $\mathrm{IFO}(b_{\mathrm{grp}})$ denote the total number of component-gradient evaluations $\nabla f_{i,j}$ needed to obtain an $\epsilon$-stationary point.

By Theorem~\ref{theo2}, with
\begin{align*}
A_1(b_{\mathrm{grp}}) &\eqdef \frac{(n-b_{\mathrm{grp}})(5n+4)}{n-1},
\qquad
A_2(b_{\mathrm{grp}}) \eqdef \frac{n^2+2n-7b_{\mathrm{grp}}+4b_{\mathrm{grp}}^2}{2b_{\mathrm{grp}}n},
\end{align*}
it is sufficient to take
\begin{align*}
T+1
=
\cO\rb{
\rb{
L+\sqrt{\delta_1^2 A_1(b_{\mathrm{grp}})+\delta_2^2 A_2(b_{\mathrm{grp}})}
}\frac{\Delta_0}{\epsilon}
}.
\end{align*}
Furthermore, each iteration computes $m+2b_{\mathrm{grp}}-1$ component gradients (and $m+2b_{\mathrm{grp}}-2$ when $b_{\mathrm{grp}}=n$, which changes the final complexity only by an additive constant). Since the initialization $g_i^0=\nabla f_i(x^0)$ costs $nm$ component gradients, we obtain
\begin{align}
\label{eq:ifo_ngtm_exact_b}
\mathrm{IFO}(b_{\mathrm{grp}})
=
\cO\rb{
nm +
\rb{m+2b_{\mathrm{grp}}-1}
\rb{
L+\sqrt{\delta_1^2 A_1(b_{\mathrm{grp}})+\delta_2^2 A_2(b_{\mathrm{grp}})}
}
\frac{\Delta_0}{\epsilon}
}.
\end{align}

We now simplify the dependence on $b_{\mathrm{grp}}$. Since $n\geq 2$,
\begin{align*}
5(n-b_{\mathrm{grp}})\leq A_1(b_{\mathrm{grp}})\leq 14(n-b_{\mathrm{grp}}),
\end{align*}
because $\frac{5n+4}{n-1}\in[5,14]$ for all $n\geq 2$. Also,
\begin{align*}
A_2(b_{\mathrm{grp}})
=
\frac{n}{2b_{\mathrm{grp}}}+\frac{1}{b_{\mathrm{grp}}}-\frac{7}{2n}+\frac{2b_{\mathrm{grp}}}{n}.
\end{align*}
For $1\leq b_{\mathrm{grp}}\leq n$, the last three terms are jointly nonnegative. Over a common denominator,
\begin{align*}
\frac{1}{b_{\mathrm{grp}}}-\frac{7}{2n}+\frac{2b_{\mathrm{grp}}}{n}
=
\frac{2b_{\mathrm{grp}}^2-\frac{7}{2}b_{\mathrm{grp}}+n}{b_{\mathrm{grp}}n}.
\end{align*}
As a function of $b_{\mathrm{grp}}$, the numerator $2b_{\mathrm{grp}}^2-\frac{7}{2}b_{\mathrm{grp}}+n$ is a convex parabola with vertex at $b_{\mathrm{grp}}=\frac{7}{8}<1$, hence nondecreasing on $[1,n]$; its minimum there is attained at $b_{\mathrm{grp}}=1$ and equals $n-\frac{3}{2}\geq\frac{1}{2}>0$, using $n\geq 2$. Therefore $\frac{1}{b_{\mathrm{grp}}}-\frac{7}{2n}+\frac{2b_{\mathrm{grp}}}{n}>0$.
Hence
\begin{align*}
A_2(b_{\mathrm{grp}})\geq \frac{n}{2b_{\mathrm{grp}}}.
\end{align*}
On the other hand, because $b_{\mathrm{grp}}\leq n$,
\begin{align*}
A_2(b_{\mathrm{grp}})
&\leq \frac{n}{2b_{\mathrm{grp}}}+\frac{1}{b_{\mathrm{grp}}}+2
\leq \frac{n}{2b_{\mathrm{grp}}}+\frac{n}{b_{\mathrm{grp}}}+\frac{2n}{b_{\mathrm{grp}}}
= \frac{7n}{2b_{\mathrm{grp}}}.
\end{align*}
Therefore,
\begin{align*}
A_1(b_{\mathrm{grp}})=\Theta(n-b_{\mathrm{grp}}),
\qquad
A_2(b_{\mathrm{grp}})=\Theta\!\left(\frac{n}{b_{\mathrm{grp}}}\right).
\end{align*}
Since also $m+2b_{\mathrm{grp}}-1=\Theta(m+b_{\mathrm{grp}})$ and $\sqrt{u+v}=\Theta(\sqrt{u}+\sqrt{v})$ for all $u,v\geq 0$, \eqref{eq:ifo_ngtm_exact_b} becomes
\begin{align}
\label{eq:ifo_ngtm_simplified_b}
\mathrm{IFO}(b_{\mathrm{grp}})
=
\Theta\left(
nm +
\Phi(b_{\mathrm{grp}})\frac{\Delta_0}{\epsilon}
\right),
\qquad
\Phi(b_{\mathrm{grp}})\eqdef (m+b_{\mathrm{grp}})\left(
L+\delta_1\sqrt{n-b_{\mathrm{grp}}}+\delta_2\sqrt{\frac{n}{b_{\mathrm{grp}}}}
\right).
\end{align}
Thus, up to absolute constants, minimizing the total oracle complexity is equivalent to minimizing the one-dimensional function $\Phi(b_{\mathrm{grp}})$ over $b_{\mathrm{grp}}\in[n]$, i.e.,
\begin{align}
\label{eq:b_star_ngtm}
b_{\mathrm{grp}}^\star \in \argmin_{b_{\mathrm{grp}}\in[n]} \Phi(b_{\mathrm{grp}}).
\end{align}

The continuous relaxation already shows that the optimal batch size is parameter dependent. Indeed, for $b_{\mathrm{grp}}\in[1,n]$,
\begin{align*}
\Phi'(b_{\mathrm{grp}})
=
L
+\delta_1\frac{2n-m-3b_{\mathrm{grp}}}{2\sqrt{n-b_{\mathrm{grp}}}}
+\delta_2\frac{\sqrt n\,(b_{\mathrm{grp}}-m)}{2b_{\mathrm{grp}}^{3/2}}.
\end{align*}
Hence there is no universal closed-form choice of $b_{\mathrm{grp}}$ that minimizes all terms simultaneously. This is also clear from the three contributions in \eqref{eq:ifo_ngtm_simplified_b}:
\begin{align*}
\Phi_L(b_{\mathrm{grp}}) &\eqdef (m+b_{\mathrm{grp}})L,\\
\Phi_{\delta_1}(b_{\mathrm{grp}}) &\eqdef (m+b_{\mathrm{grp}})\sqrt{n-b_{\mathrm{grp}}}\,\delta_1,\\
\Phi_{\delta_2}(b_{\mathrm{grp}}) &\eqdef (m+b_{\mathrm{grp}})\sqrt{\frac{n}{b_{\mathrm{grp}}}}\,\delta_2.
\end{align*}
The smoothness term $\Phi_L(b_{\mathrm{grp}})$ is minimized by the smallest possible batch size. The $\delta_2$-term is minimized exactly at $b_{\mathrm{grp}}=m$, because
\begin{align}
\label{eq:delta2_term_minimized_at_m}
(m+b_{\mathrm{grp}})\sqrt{\frac{n}{b_{\mathrm{grp}}}}
=
\sqrt n\left(\sqrt b_{\mathrm{grp}}+\frac{m}{\sqrt b_{\mathrm{grp}}}\right)
\geq 2\sqrt{nm},
\end{align}
with equality if and only if $b_{\mathrm{grp}}=m$. Finally, the $\delta_1$-term is completely eliminated by taking $b_{\mathrm{grp}}=n$.

Two particularly relevant choices are therefore $b_{\mathrm{grp}}=m$ and $b_{\mathrm{grp}}=n$. If $b_{\mathrm{grp}}=m$, then \eqref{eq:ifo_ngtm_simplified_b} yields
\begin{align}
\label{eq:ifo_ngtm_b_eq_m}
\mathrm{IFO}(m)
=
\cO\rb{
nm +
\rb{
mL + m\sqrt{n-m}\,\delta_1 + \sqrt{nm}\,\delta_2
}\frac{\Delta_0}{\epsilon}
},
\end{align}
where we used \eqref{eq:delta2_term_minimized_at_m}. If $b_{\mathrm{grp}}=n$, then $A_1(n)=0$ and
\begin{align*}
A_2(n)=\frac{n^2+2n-7n+4n^2}{2n^2}
=\frac{5}{2}\left(1-\frac{1}{n}\right),
\end{align*}
so \eqref{eq:ifo_ngtm_exact_b} gives
\begin{align}
\label{eq:ifo_ngtm_b_eq_n}
\mathrm{IFO}(n)
=
\cO\rb{
nm +
n\rb{L+\delta_2}\frac{\Delta_0}{\epsilon}
}.
\end{align}

Therefore, choosing $b_{\mathrm{grp}}=n$ is \emph{not} optimal in general. At the level of the upper bounds \eqref{eq:ifo_ngtm_b_eq_m}--\eqref{eq:ifo_ngtm_b_eq_n}, the choice $b_{\mathrm{grp}}=n$ can only improve upon $b_{\mathrm{grp}}=m$ in the regime
\begin{align*}
m\sqrt{n-m}\,\delta_1
=
\Omega\left(
(n-m)L + (n-\sqrt{nm})\delta_2
\right),
\end{align*}
that is, when the benefit of removing the $\delta_1$-term completely outweighs the larger $n$-scale smoothness cost and the larger $\delta_2$-contribution. Otherwise, the balanced choice $b_{\mathrm{grp}}=m$ is preferable. In summary, the optimal batch size in the case $n>m$ is the minimizer $b_{\mathrm{grp}}^\star$ in \eqref{eq:b_star_ngtm}, and the explicit choices $b_{\mathrm{grp}}=m$ and $b_{\mathrm{grp}}=n$ capture the two main regimes of interest. Rather than reporting a separate ablation, the batch-size grid search of Appendix~\ref{sec:hyperparameter_tuning} navigates this tradeoff---per-iteration cost $\cO\rb{m+b_{\mathrm{grp}}}$ growing with $b_{\mathrm{grp}}$, the $\delta_2$-term minimized at $b_{\mathrm{grp}}=m$, and the $\delta_1$-term removed only at $b_{\mathrm{grp}}=n$---empirically, selecting the $b_{\mathrm{grp}}$ that minimizes oracle complexity on each instance.


\newpage
\section{Supplementary Experimental Details}
\label{sec:experiments_similarity_logreg}

This appendix extends the experimental evaluation presented in Section~\ref{sec:experiments}, providing comprehensive details on the dataset generation process, the formal Hessian interpretation of the nested similarity parameters, and our estimation protocols.

Our empirical evaluation is organized around the nested finite-sum structure
\begin{align}
\label{eq:exp_dbl_sum}
f(x)\eqdef \frac{1}{n}\sum_{i=1}^n f_i(x),
\qquad
f_i(x)\eqdef \frac{1}{m}\sum_{j=1}^m f_{i,j}(x),
\end{align}
in which the partition into $n$ groups of size $m$ is \emph{chosen by the experimenter}. Our theory predicts that optimization is highly sensitive to this choice: it is beneficial to group the data so that the within-group similarity is high (small $\delta_2$), while the influence of across-group similarity ($\delta_1$) depends on the dimensional regime ($m \ge n$ versus $n > m$). The rest of this appendix specifies the logistic-regression task and the synthetic data used in our experiments, derives the Hessian-based characterization of $\delta_1$ and $\delta_2$, describes both theoretical and empirical procedures for estimating $L$, $\delta_1$, and $\delta_2$, and reports hardware and hyperparameter-tuning details.

\subsection{Binary logistic regression: loss, gradient, Hessian, and data geometry}
\label{sec:exp_logreg_hessian}

Our empirical study uses binary logistic regression augmented with a smooth nonconvex regularizer. Each component takes the form
\begin{align}
\label{eq:exp_logreg_component}
f_{i,j}(x)
\;\eqdef\;
\ell(x;a_{i,j},y_{i,j}) + \lambda\, r(x),
\qquad
\ell(x;a,y)\eqdef \log\bigl(1+\exp(-y\,a^\top x)\bigr),
\end{align}
where $y_{i,j}\in\{\pm 1\}$ are the labels and $a_{i,j}\in\R^d$ are the feature vectors. The regularizer $r:\R^d\to\R$ introduces nonconvexity while remaining twice continuously differentiable:
\begin{align}
\label{eq:exp_noncvx_reg}
r(x)\eqdef \sum_{\ell=1}^d \frac{x_\ell^2}{1+x_\ell^2}.
\end{align}

Letting $\sigma(t)\eqdef (1+e^{-t})^{-1}$ denote the standard sigmoid, the gradient and Hessian of the unregularized logistic loss are
\begin{align}
\nabla \ell(x;a,y)
&= -y\,a\,\sigma(-y\,a^\top x),
\\
\nabla^2 \ell(x;a,y)
&= \alpha(x;a,y)\,a a^\top,
\qquad
\alpha(x;a,y)\eqdef \sigma(z)\bigl(1-\sigma(z)\bigr),
\quad
z\eqdef -y\,a^\top x,
\end{align}
so that $0<\alpha(x;a,y)\le \tfrac{1}{4}$ for all $x$, $a$, and $y$. For the regularizer~\eqref{eq:exp_noncvx_reg}, the coordinate-wise gradient and Hessian are
\begin{align}
\label{eq:exp_noncvx_reg_grad_hess}
\rb{\nabla r(x)}_\ell
= \frac{2x_\ell}{(1+x_\ell^2)^2},
\qquad
\nabla^2 r(x)
= \Diag\!\rb{\frac{2(1-3x_\ell^2)}{(1+x_\ell^2)^3}}_{\ell=1}^d.
\end{align}
Combining these contributions, the full component gradient and Hessian are
\begin{align}
\nabla f_{i,j}(x)
&= -y_{i,j}\,a_{i,j}\,\sigma(-y_{i,j}\,a_{i,j}^\top x) + \lambda \nabla r(x),
\\
\label{eq:exp_logreg_hess}
\nabla^2 f_{i,j}(x)
&= \alpha_{i,j}(x)\,a_{i,j}a_{i,j}^\top + \lambda \nabla^2 r(x),
\qquad
0<\alpha_{i,j}(x)\le \tfrac{1}{4}.
\end{align}

Two structural properties of~\eqref{eq:exp_logreg_hess} drive the rest of the analysis. First, the data-dependent term is a scaled rank-one matrix whose magnitude is governed by the feature norm, $\opnorm{a a^\top}=\sqn{a}$. Second, although the regularizer is nonconvex and $\nabla^2 r(x)$ depends on $x$, it is \emph{identical across all components} in~\eqref{eq:exp_logreg_component}. Consequently, the regularizer contribution cancels exactly in every nested Hessian difference, such as $\nabla^2 f_{i,j}(x)-\nabla^2 f_i(x)$ and $\nabla^2 f_i(x)-\nabla^2 f(x)$. Therefore, the similarity parameters $\delta_1$ and $\delta_2$ are governed entirely by the underlying \emph{data geometry} (through the rank-one terms) and by the chosen grouping partition.

As a concrete illustration, within a fixed group $i$ the deviation of an individual sample's Hessian from the local group average reads
\begin{align}
\label{eq:exp_within_group_hess_diff}
\nabla^2 f_{i,j}(x)-\nabla^2 f_i(x)
&= \alpha_{i,j}(x)\,a_{i,j}a_{i,j}^\top
-\frac{1}{m}\sum_{k=1}^m \alpha_{i,k}(x)\,a_{i,k}a_{i,k}^\top.
\end{align}
This isolates the precise structural quantity that drives the Hessian-based characterization of similarity established in Lemma~\ref{lem:hessian_characterization_delta}.

\subsection{Partition regimes and synthetic datasets}
\label{subsec:partition-regimes}

To empirically validate our theory, we generate synthetic grouped datasets that natively respect the nested finite-sum structure of~\eqref{eq:exp_dbl_sum}. We investigate two aspect ratios, $(n,m)=(50,250)$ and $(n,m)=(250,50)$, corresponding to the $m\ge n$ and $n>m$ regimes, respectively.

For each aspect ratio, we construct four datasets that realize all four quadrants of the qualitative similarity space:
\begin{align*}
    &(\delta_1 \text{ small}, \delta_2 \text{ small}),\qquad
    (\delta_1 \text{ small}, \delta_2 \text{ large}),\qquad
    (\delta_1 \text{ large}, \delta_2 \text{ small}),\qquad
    (\delta_1 \text{ large}, \delta_2 \text{ large}).
\end{align*}
The datasets are generated directly in their grouped form, so a sample $(i,j)$ belongs to group $i$ by construction and no post-hoc clustering step is required. To evaluate methods that operate on a single-level finite sum, the same data are also stored in flattened format.

The base task is the binary logistic regression with the nonconvex regularizer $r(x) \eqdef \sum_{\ell=1}^d x_\ell^2 / (1+x_\ell^2)$ introduced in~\eqref{eq:exp_logreg_component}. For each feature--label pair $(a_{i,j},y_{i,j})\in \R^d\times\{-1,+1\}$, the component loss is
\begin{align*}
    f_{i,j}(x)
    &=
    \log\bigl(1+\exp(-y_{i,j}a_{i,j}^{\top}x)\bigr)
    +
    \lambda\, r(x),
    \qquad \lambda=200.
\end{align*}
Because the same regularizer $r(x)$ is added to every component, it cancels exactly in the Hessian differences used to define $\delta_1$ and $\delta_2$, while continuing to contribute to the global and maximum smoothness constants $L$ and $L_{\max}$.

The feature vectors $a_{i,j}$ are generated using a hierarchical latent-mixture model. We define $K$ latent components partitioned evenly into $T$ macro-families, such that each family contains $q=K/T$ components. We let
\begin{align*}
    t(k)&\eqdef 1+\left\lfloor\frac{k-1}{q}\right\rfloor,
    \qquad k\in[K],
\end{align*}
denote the macro-family index of component $k$. We sample $T+K$ orthonormal directions
\begin{align*}
    v_1,\ldots,v_T,u_1,\ldots,u_K\in\R^d
\end{align*}
and establish the cluster centers as:
\begin{align*}
    c_r &\eqdef r_{\mathrm{inter}}v_r,
    \qquad
    \mu_k \eqdef c_{t(k)} + r_{\mathrm{intra}}u_k .
\end{align*}
The scalar $r_{\mathrm{inter}}$ dictates the separation between distinct macro-families, while $r_{\mathrm{intra}}$ controls the variance among latent components within the same macro-family.

Each group $i \in [n]$ is assigned a prototype mixture vector $\beta_i\in\R^K$. To prevent degenerate zero probabilities when sampling, we smooth this prototype using a small constant $\varepsilon_{\mathrm{dir}}$:
\begin{align*}
    \widetilde\beta_i
    &=
    (1-\varepsilon_{\mathrm{dir}})\beta_i
    +
    \varepsilon_{\mathrm{dir}}\frac{\mathbf{1}}{K}.
\end{align*}
The data points are then sampled hierarchically:
\begin{align*}
    \pi_i&\sim \mathrm{Dirichlet}(\alpha_{\mathrm{dir}}\widetilde\beta_i),
    \qquad
    z_{i,j}\sim \mathrm{Categorical}(\pi_i),
    \qquad
    a_{i,j}\sim \mathcal{N}(\mu_{z_{i,j}},\sigma_{\mathrm{obs}}^2 I_d).
\end{align*}
All datasets are generated in dimension $d=1000$. The $(n,m)=(50,250)$ configuration uses seed $42$, and the $(n,m)=(250,50)$ configuration uses seed $142$.

Labels are generated via a Gaussian teacher logistic model. The teacher direction $x_\star$ is normalized and rescaled by the inverse median absolute raw logit to ensure class probabilities do not degenerate. Labels are sampled according to:
\begin{align*}
    \mathbb{P}(y_{i,j}=1\mid a_{i,j})
    &=
    \frac{1}{1+\exp(-a_{i,j}^{\top}x_\star)}.
\end{align*}
If the empirical positive-class fraction falls outside the $[0.3, 0.7]$ range, the teacher direction is resampled.

To formalize the four structural regimes, let $e_k$ denote the $k$-th standard basis vector in $\R^K$, and define the base indices:
\begin{align*}
    a_r&\eqdef (r-1)q+1,
    \qquad
    s(i)\eqdef 1+((i-1)\bmod T).
\end{align*}
By manipulating only the prototype mixture assignments $\beta_i$ and the geometric scalars, we construct the four distinct data profiles:

\paragraph{Small $\delta_1$, Small $\delta_2$.}
All groups share the same tightly concentrated prototype:
\begin{align*}
    \beta_i &= 0.95 e_1 + 0.05 e_2,
    \qquad \forall i\in[n].
\end{align*}
Consequently, the groups are mutually similar ($\delta_1$ is small), and each group is internally coherent ($\delta_2$ is small).

\paragraph{Small $\delta_1$, Large $\delta_2$.}
All groups share the same broad prototype, activating one component per macro-family:
\begin{align*}
    \beta_i
    &=
    \frac1T\sum_{r=1}^T e_{a_r},
    \qquad \forall i\in[n].
\end{align*}
Because the mixture is identical across groups, $\delta_1$ remains small; however, because each group mixes broadly separated macro-families, $\delta_2$ is fundamentally large.

\paragraph{Large $\delta_1$, Small $\delta_2$.}
Groups are assigned strictly different macro-family identities, yet each group remains concentrated exclusively inside its assigned family:
\begin{align*}
    \beta_i
    &=
    0.95 e_{a_{s(i)}} + 0.05 e_{a_{s(i)}+1}.
\end{align*}
Thus, different groups represent distinct statistical populations (large $\delta_1$), while remaining internally homogeneous (small $\delta_2$).

\paragraph{Large $\delta_1$, Large $\delta_2$.}
For this highly heterogeneous regime, we scale the dimensions to $K=64$, $T=8$, and $q=8$. We define:
\begin{align*}
    \gamma_r
    &\eqdef
    \frac12 e_{a_r}+\frac12 e_{a_r+4},
    \qquad r\in[T],
\end{align*}
and construct cyclic family-pair prototypes:
\begin{align*}
    \beta^{(r)}
    &\eqdef
    0.7\gamma_r + 0.3\gamma_{r+1},
    \qquad \gamma_{T+1}\eqdef \gamma_1.
\end{align*}
Each group is assigned $\beta_i=\beta^{(s(i))}$. This ensures each group maintains a distinct macro-family anchor (driving $\delta_1$ high), while simultaneously mixing widely separated components (driving $\delta_2$ high).

\begin{table}[t]
    \centering
    \small
    \setlength{\tabcolsep}{4pt}
    \begin{tabular}{llrrrrrr}
        \toprule
        \textbf{Setting} & \textbf{Regime} & $K$ & $T$ & $\alpha_{\rm dir}$ & $\varepsilon_{\rm dir}$ & $\sigma_{\rm obs}$ & $(r_{\rm inter},r_{\rm intra})$ \\
        \midrule
        $m\ge n$ & $(\delta_1 \text{ small}, \delta_2 \text{ small})$ & 32 & 4 & 2000 & $10^{-4}$ & 0.01 & $(8,0.02)$ \\
        $m\ge n$ & $(\delta_1 \text{ small}, \delta_2 \text{ large})$ & 32 & 4 & 1000 & $10^{-3}$ & 0.02 & $(25,0.5)$ \\
        $m\ge n$ & $(\delta_1 \text{ large}, \delta_2 \text{ small})$ & 32 & 4 & 1000 & $10^{-4}$ & 0.01 & $(25,0.02)$ \\
        $m\ge n$ & $(\delta_1 \text{ large}, \delta_2 \text{ large})$ & 64 & 8 & 1000 & $10^{-3}$ & 0.02 & $(30,0.5)$ \\
        \midrule
        $n>m$ & $(\delta_1 \text{ small}, \delta_2 \text{ small})$ & 32 & 4 & 4000 & $10^{-4}$ & 0.005 & $(8,0.01)$ \\
        $n>m$ & $(\delta_1 \text{ small}, \delta_2 \text{ large})$ & 32 & 4 & 1500 & $10^{-3}$ & 0.02 & $(30,0.5)$ \\
        $n>m$ & $(\delta_1 \text{ large}, \delta_2 \text{ small})$ & 32 & 4 & 1500 & $10^{-4}$ & 0.005 & $(30,0.01)$ \\
        $n>m$ & $(\delta_1 \text{ large}, \delta_2 \text{ large})$ & 64 & 8 & 1500 & $10^{-3}$ & 0.02 & $(30,0.5)$ \\
        \bottomrule
    \end{tabular}
    \caption{Generation parameters for the synthetic grouped logistic regression datasets.}
    \label{tab:synthetic-generation-params}
\end{table}


\subsubsection{Interpretation of the controlled regimes}
\label{subsubsec:synthetic-regime-interpretation}

We now explain why the construction above produces the intended relative magnitudes of $\delta_1$ and $\delta_2$. The discussion relies on the Hessian interpretation of the gradient-difference similarity assumptions (formalized in Lemma~\ref{lem:hessian_characterization_delta} below). For the logistic-regression components used here, let
\begin{align*}
    H_{i,j}(x) &\eqdef \nabla^2 f_{i,j}(x), \qquad
    H_i(x) \eqdef \frac1m\sum_{j=1}^m H_{i,j}(x), \qquad
    H(x) \eqdef \frac1n\sum_{i=1}^n H_i(x).
\end{align*}
Writing
\begin{align*}
    \alpha_{i,j}(x)
    &\eqdef
    \sigma(-y_{i,j}a_{i,j}^{\top}x)
    \bigl(1-\sigma(-y_{i,j}a_{i,j}^{\top}x)\bigr),
\end{align*}
with $\sigma(t)=(1+\exp(-t))^{-1}$, we have
\begin{align*}
    H_{i,j}(x)
    &=
    \alpha_{i,j}(x)a_{i,j}a_{i,j}^{\top}
    +
    \lambda \nabla^2 r(x).
\end{align*}
Since $\lambda\nabla^2 r(x)$ is identical across components, it cancels in both $H_{i,j}(x)-H_i(x)$ and $H_i(x)-H(x)$; hence the empirical similarities $\delta_1$ and $\delta_2$ are controlled exclusively by the geometry of the rank-one logistic terms $\alpha_{i,j}(x)a_{i,j}a_{i,j}^{\top}$. The regularizer does, however, contribute to the smoothness constants. Indeed,
\begin{align*}
    \nabla^2 r(x)
    &=
    \Diag\!\rb{
    \frac{2(1-3x_\ell^2)}{(1+x_\ell^2)^3}
    }_{\ell=1}^d,
    \qquad
    \opnorm{\nabla^2 r(x)}\le 2,
\end{align*}
so the choice $\lambda=200$ adds a smoothness contribution of order $400$ while leaving the nested similarity estimates untouched.

To connect this finite-sample construction to the similarity assumptions, fix $h=x-y$. By the fundamental theorem of calculus,
\begin{align*}
    &\nabla f_i(x)-\nabla f(x)-\nabla f_i(y)+\nabla f(y) \\
    &\qquad =
    \int_0^1
    \bigl(H_i(y+\tau h)-H(y+\tau h)\bigr)h
    \,d\tau,
\end{align*}
and similarly
\begin{align*}
    &\nabla f_{i,j}(x)-\nabla f_i(x)-\nabla f_{i,j}(y)+\nabla f_i(y) \\
    &\qquad =
    \int_0^1
    \bigl(H_{i,j}(y+\tau h)-H_i(y+\tau h)\bigr)h
    \,d\tau.
\end{align*}
Therefore, small average deviations of the group Hessians $H_i$ from the global Hessian $H$ produce a small $\delta_1$, while small average deviations of the component Hessians $H_{i,j}$ from the group Hessian $H_i$ produce a small $\delta_2$.

At the population level, this mechanism is transparent. Let $M_k(x)$ denote the average logistic Hessian contribution of latent component $k$ at point $x$ (excluding the common regularizer). Conditional on the group mixture $\pi_i$, the population analogue of the group Hessian has the form
\begin{align*}
    \overline H_i(x)
    &\approx
    \sum_{k=1}^K \pi_{i,k} M_k(x)
    +
    \lambda\nabla^2 r(x).
\end{align*}
Writing $\overline\pi\eqdef \frac1n\sum_{i=1}^n \pi_i$, the between-group deviation is governed by
\begin{align*}
    \overline H_i(x)-\overline H(x)
    &\approx
    \sum_{k=1}^K
    (\pi_{i,k}-\overline\pi_k)M_k(x),
\end{align*}
whereas the within-group dispersion is governed by
\begin{align*}
    \sum_{k=1}^K \pi_{i,k}
    \opnorm{
    M_k(x)-\sum_{\ell=1}^K \pi_{i,\ell}M_\ell(x)
    }^2 .
\end{align*}
Hence $\delta_1$ is small when the group mixtures $\pi_i$ are similar across $i$ and large when different groups place mass on different latent families. Analogously, $\delta_2$ is small when each $\pi_i$ concentrates on one local component (or on nearby components) and large when each $\pi_i$ spreads mass over well-separated components.

The Dirichlet concentration parameter $\alpha_{\mathrm{dir}}$ controls how closely each $\pi_i$ tracks its prototype $\beta_i$. Since
\begin{align*}
    \mathbb{E}[\pi_i] &= \widetilde\beta_i,\\
    \operatorname{Var}(\pi_{i,k})
    &=
    \frac{\widetilde\beta_{i,k}(1-\widetilde\beta_{i,k})}
    {\alpha_{\mathrm{dir}}+1},
\end{align*}
large values of $\alpha_{\mathrm{dir}}$ force the realized mixtures $\pi_i$ to concentrate around the prescribed prototypes. The small smoothing parameter $\varepsilon_{\mathrm{dir}}$ only prevents degenerate zero entries in the Dirichlet parameters and has negligible effect on the intended geometry.

We now interpret each of the four regimes in turn.

\paragraph{Small $\delta_1$, Small $\delta_2$.}
All groups share the same tightly concentrated prototype
\begin{align*}
    \beta_i = 0.95e_1+0.05e_2.
\end{align*}
Up to small Dirichlet fluctuations, every group then has the same mixture; hence $H_i(x)$ is close to $H(x)$ for every $i$, and $\delta_1$ is small. The prototype is itself nearly pure, and the geometry parameters use very small observation noise and within-family radius,
\begin{align*}
    \sigma_{\rm obs}\in\{0.01,0.005\},
    \qquad
    r_{\rm intra}\in\{0.02,0.01\}.
\end{align*}
Consequently, samples within each group concentrate around almost the same latent component, so the matrices $H_{i,j}(x)$ stay close to their local average $H_i(x)$ and $\delta_2$ is small. The large values of $\alpha_{\rm dir}$ and the very small $\varepsilon_{\rm dir}=10^{-4}$ further suppress accidental mixture variability.

\paragraph{Small $\delta_1$, Large $\delta_2$.}
All groups again share the same prototype, but the prototype is now broad,
\begin{align*}
    \beta_i
    =
    \frac1T\sum_{r=1}^T e_{a_r}.
\end{align*}
Because the same mixture is used for every group, the group Hessians $H_i(x)$ remain close to each other and to $H(x)$, so $\delta_1$ stays small. Each group, however, now mixes one active component from every macro-family. The macro-family centers are separated at scale
\begin{align*}
    r_{\rm inter}=25
    \quad\text{for } m\ge n,
    \qquad
    r_{\rm inter}=30
    \quad\text{for } n>m,
\end{align*}
and the directions defining the macro-centers are orthonormal. For components belonging to different macro-families, the corresponding rank-one matrices are therefore substantially different: at the level of the idealized centers, if $r\neq r'$, the matrices $c_rc_r^\top$ and $c_{r'}c_{r'}^\top$ have operator-norm separation of order $r_{\rm inter}^2$. As a result, within a single group, the component Hessians $H_{i,j}(x)$ fluctuate strongly around $H_i(x)$, producing a large $\delta_2$. The concentration values $\alpha_{\rm dir}\in\{1000,1500\}$ keep this broad mixture stable across groups, preventing $\delta_1$ from becoming comparable to $\delta_2$.

\paragraph{Large $\delta_1$, Small $\delta_2$.}
The group prototype now depends on the assigned macro-family $s(i)$:
\begin{align*}
    \beta_i
    =
    0.95e_{a_{s(i)}}+0.05e_{a_{s(i)}+1}.
\end{align*}
Different groups are thus anchored to different macro-families. Since the macro-centers are separated by $r_{\rm inter}\in\{25,30\}$, the corresponding group-level Hessian averages are far apart, and the terms $H_i(x)-H(x)$ become large on average; this is the mechanism producing a large $\delta_1$.

At the same time, each group is nearly pure within its assigned family, with very small within-family radius and observation noise,
\begin{align*}
    r_{\rm intra}\in\{0.02,0.01\},
    \qquad
    \sigma_{\rm obs}\in\{0.01,0.005\}.
\end{align*}
Therefore, although different groups are far from one another, samples inside a fixed group concentrate around a single local component. Hence $H_{i,j}(x)$ stays close to $H_i(x)$ within each group and $\delta_2$ remains small. This regime cleanly isolates between-group heterogeneity from within-group heterogeneity.

\paragraph{Large $\delta_1$, Large $\delta_2$.}
The last regime combines the two mechanisms above. We enlarge the latent dictionary to $K=64$ and the number of macro-families to $T=8$, giving more distinct family identities. For each family we define
\begin{align*}
    \gamma_r
    =
    \frac12 e_{a_r}+\frac12 e_{a_r+4},
\end{align*}
so that $\gamma_r$ already mixes two distinct components inside macro-family $r$. The group prototypes are
\begin{align*}
    \beta^{(r)}
    &=
    0.7\gamma_r + 0.3\gamma_{r+1},
    \qquad \gamma_{T+1}=\gamma_1,\\
    \beta_i
    &=
    \beta^{(s(i))}.
\end{align*}
The dominant term $0.7\gamma_{s(i)}$ gives each group a distinct macro-family anchor: since $s(i)$ varies across groups, the group Hessians $H_i(x)$ have different dominant macro-family contributions, which drives $\delta_1$ large. The secondary term $0.3\gamma_{s(i)+1}$ guarantees that every group also mixes an additional separated macro-family. Together with $r_{\rm inter}=30$ and $r_{\rm intra}=0.5$, this yields substantial variation among the component Hessians inside each group, keeping $\delta_2$ large as well.

This regime is deliberately more structured than a symmetric broad mixture. A fully symmetric broad mixture would make every group internally diverse but too similar to every other group, yielding large $\delta_2$ but only moderate $\delta_1$. The cyclic family-anchored construction avoids this by assigning each group a different dominant family while retaining internal diversity through the secondary family and the two-component mixtures $\gamma_r$.

\subsection{Estimating $L$, $\delta_1$, $\delta_2$ in theory}
\label{sec:exp_estimating_deltas_theory}

We first record an explicit upper bound on the Lipschitz smoothness constant of the logistic regression objective~\eqref{eq:prob} obtained directly from the component Hessian in~\eqref{eq:exp_logreg_hess}. Because the regularizer is nonconvex, the relevant quantity for smoothness is the operator norm of the Hessian, not only a Loewner upper bound. We therefore bound $\opnorm{\nabla^2 f(x)}$ term by term. The logistic part of the component Hessian is positive semidefinite, and since $\alpha_{i,j}(x)\le \tfrac{1}{4}$ for all $x$, $i$, and $j$, averaging over the $nm$ samples and using monotonicity of the operator norm on PSD matrices gives
\begin{align*}
\opnorm{\frac{1}{nm}\sum_{i=1}^n\sum_{j=1}^m \alpha_{i,j}(x)\,a_{i,j}a_{i,j}^\top}
\;\le\;
\opnorm{\frac{1}{4nm}\,\mathbf{A}^\top \mathbf{A}}
\;=\;
\lambda_{\max}\!\rb{\frac{1}{4nm}\,\mathbf{A}^\top \mathbf{A}},
\end{align*}
where $\mathbf{A}\in\R^{nm\times d}$ is the design matrix stacking all feature vectors $a_{i,j}^\top$ as rows. The regularizer Hessian in~\eqref{eq:exp_noncvx_reg_grad_hess} is indefinite but bounded in operator norm: its diagonal entries $2(1-3x_\ell^2)/(1+x_\ell^2)^3$ lie in $[-1/2,\,2]$, so $\opnorm{\nabla^2 r(x)}\le 2$. Combining the two contributions via the triangle inequality for the operator norm yields
\begin{align*}
\opnorm{\nabla^2 f(x)}
\;\le\;
\lambda_{\max}\!\rb{\frac{1}{4nm}\,\mathbf{A}^\top \mathbf{A}}
+ 2\lambda,
\end{align*}
and a valid theoretical choice for the smoothness constant is therefore
\begin{align*}
L \;=\; \lambda_{\max}\!\rb{\frac{1}{4nm}\,\mathbf{A}^\top \mathbf{A}} + 2\lambda.
\end{align*}

For the similarity constants $\delta_1$ and $\delta_2$, a closed-form expression of this kind is not available. Instead, we provide an equivalent characterization of the gradient-difference conditions~\eqref{eq:exp_delta1}--\eqref{eq:exp_delta2} as operator-norm constraints on stacked Hessian deviations. This characterization both clarifies how the data-generation mechanisms of Section~\ref{subsubsec:synthetic-regime-interpretation} translate into the magnitudes of $\delta_1$ and $\delta_2$, and motivates the empirical estimation protocol of Section~\ref{sec:exp_estimating_deltas}.

\subsubsection{Hessian characterization of $\delta_1$- and $\delta_2$-similarity}
\label{sec:exp_similarity_hessian}

Recall the gradient-difference similarity conditions: there exist $\delta_1,\delta_2\ge 0$ such that for all $x,y\in\R^d$,
\begin{align}
\label{eq:exp_delta1}
\frac{1}{n}\sum_{i=1}^n \sqn{\nabla f_i(x)-\nabla f(x)-\nabla f_i(y)+\nabla f(y)} &\le \delta_1^2\sqn{x-y}, \\
\label{eq:exp_delta2}
\frac{1}{nm}\sum_{i=1}^n\sum_{j=1}^m \sqn{\nabla f_{i,j}(x)-\nabla f_i(x)-\nabla f_{i,j}(y)+\nabla f_i(y)} &\le \delta_2^2\sqn{x-y}.
\end{align}

Because these conditions are stated in terms of gradient differences, it is useful to connect them to second-order variability when the component losses are twice continuously differentiable. The following lemma provides an equivalent characterization in terms of Hessian deviations, which directly guides how we estimate $(\delta_1,\delta_2)$ in practice.

\begin{lemma}[Hessian characterization of $\delta_1$- and $\delta_2$-similarity]
\label{lem:hessian_characterization_delta}
Let $f(x)\eqdef \frac{1}{n}\sum_{i=1}^n f_i(x)$ and
$f_i(x)\eqdef \frac{1}{m}\sum_{j=1}^m f_{i,j}(x)$.

\smallskip
\noindent
\textbf{(i) $\delta_1$-similarity.}
Assume that $f_i$ is twice continuously differentiable for all $i\in[n]$.
For each $x\in\R^d$, define the linear operator
$\mathcal{A}_1(x):\R^d\to(\R^d)^n$ by
\begin{align}
    \mathcal{A}_1(x)s
    \eqdef
    \frac{1}{\sqrt n}
    \big(
    (\nabla^2 f_i(x)-\nabla^2 f(x))s
    \big)_{i=1}^n .
    \label{eq:A1_definition}
\end{align}
Then \eqref{eq:exp_delta1} holds if and only if
\begin{align}
    \opnorm{\mathcal{A}_1(x)}\le \delta_1,
    \qquad
    \forall x\in\R^d .
    \label{eq:hess_delta1_stacked_opnorm}
\end{align}
Equivalently,
\begin{align}
    \lambda_{\max}\rb{
    \frac1n\sum_{i=1}^n
    \rb{\nabla^2 f_i(x)-\nabla^2 f(x)}^\top
    \rb{\nabla^2 f_i(x)-\nabla^2 f(x)}
    }
    \le \delta_1^2,
    \qquad
    \forall x\in\R^d .
    \label{eq:hess_delta1_spectral}
\end{align}

\smallskip
\noindent
\textbf{(ii) $\delta_2$-similarity.}
Assume that $f_{i,j}$ is twice continuously differentiable for all
$i\in[n]$ and $j\in[m]$.
For each $x\in\R^d$, define the linear operator
$\mathcal{A}_2(x):\R^d\to(\R^d)^{nm}$ by
\begin{align}
    \mathcal{A}_2(x)s
    \eqdef
    \frac{1}{\sqrt{nm}}
    \big(
    (\nabla^2 f_{i,j}(x)-\nabla^2 f_i(x))s
    \big)_{i\in[n],\,j\in[m]} .
    \label{eq:A2_definition}
\end{align}
Then \eqref{eq:exp_delta2} holds if and only if
\begin{align}
    \opnorm{\mathcal{A}_2(x)}\le \delta_2,
    \qquad
    \forall x\in\R^d .
    \label{eq:hess_delta2_stacked_opnorm}
\end{align}
Equivalently,
\begin{align}
    \lambda_{\max}\rb{
    \frac1{nm}\sum_{i=1}^n\sum_{j=1}^m
    \rb{\nabla^2 f_{i,j}(x)-\nabla^2 f_i(x)}^\top
    \rb{\nabla^2 f_{i,j}(x)-\nabla^2 f_i(x)}
    }
    \le \delta_2^2,
    \qquad
    \forall x\in\R^d .
    \label{eq:hess_delta2_spectral}
\end{align}
\end{lemma}

\begin{proof}
We prove the two equivalences separately.

\paragraph{(i) $\delta_1$-similarity.}
Define the centered gradients
\begin{align}
    G_i(x)\eqdef \nabla f_i(x)-\nabla f(x).
    \notag
\end{align}
Then \eqref{eq:exp_delta1} is equivalent to
\begin{align}
    \frac1n\sum_{i=1}^n
    \sqn{G_i(x)-G_i(y)}
    \le
    \delta_1^2\sqn{x-y},
    \qquad
    \forall x,y\in\R^d .
    \label{eq:delta1_G_form}
\end{align}

\smallskip
\noindent
\emph{($\Rightarrow$)}
Fix $x\in\R^d$, $s\in\R^d$, and $\alpha>0$, and set
$y=x+\alpha s$. Applying \eqref{eq:delta1_G_form} gives
\begin{align}
    \frac1n\sum_{i=1}^n
    \sqn{
    \frac{G_i(x+\alpha s)-G_i(x)}{\alpha}}
    \le
    \delta_1^2\sqn{s}.
    \label{eq:delta1_scaled_difference}
\end{align}
Since each $f_i$ is twice continuously differentiable,
\begin{align}
    \frac{G_i(x+\alpha s)-G_i(x)}{\alpha}
    \xrightarrow[\alpha\downarrow 0]{}
    \rb{\nabla^2 f_i(x)-\nabla^2 f(x)}s .
    \notag
\end{align}
Taking $\alpha\downarrow 0$ in \eqref{eq:delta1_scaled_difference} yields
\begin{align}
    \frac1n\sum_{i=1}^n
    \sqn{
    \rb{\nabla^2 f_i(x)-\nabla^2 f(x)}s}
    \le
    \delta_1^2\sqn{s},
    \qquad
    \forall x,s\in\R^d .
    \label{eq:delta1_directional_hessian}
\end{align}
By the definition of $\mathcal{A}_1(x)$, this is exactly
$\sqn{\mathcal{A}_1(x)s}\le \delta_1^2\sqn{s}$ for all $s$,
which is equivalent to $\opnorm{\mathcal{A}_1(x)}\le \delta_1$.

\smallskip
\noindent
\emph{($\Leftarrow$)}
Assume \eqref{eq:hess_delta1_stacked_opnorm}. Fix $x,y\in\R^d$
and set $s=y-x$. By the fundamental theorem of calculus,
\begin{align}
    G_i(y)-G_i(x)
    =
    \int_0^1
    \rb{\nabla^2 f_i(x+\tau s)-\nabla^2 f(x+\tau s)}s
    \,d\tau .
    \notag
\end{align}
Stacking the vectors and using Jensen's inequality in the product space
$(\R^d)^n$, we obtain
\begin{align}
    \frac1n\sum_{i=1}^n \sqn{G_i(y)-G_i(x)}
    &=
    \sqn{
    \int_0^1
    \mathcal{A}_1(x+\tau s)s
    \,d\tau}
    \notag
    \\
    &\le
    \int_0^1
    \sqn{\mathcal{A}_1(x+\tau s)s}
    \,d\tau
    \notag
    \\
    &\le
    \int_0^1
    \delta_1^2\sqn{s}
    \,d\tau
    =
    \delta_1^2\sqn{y-x}.
    \notag
\end{align}
This is \eqref{eq:exp_delta1}.

\paragraph{(ii) $\delta_2$-similarity.}
Define the within-group centered gradients
\begin{align}
    G_{i,j}(x)\eqdef \nabla f_{i,j}(x)-\nabla f_i(x).
    \notag
\end{align}
Then \eqref{eq:exp_delta2} is equivalent to
\begin{align}
    \frac1{nm}\sum_{i=1}^n\sum_{j=1}^m
    \sqn{G_{i,j}(x)-G_{i,j}(y)}
    \le
    \delta_2^2\sqn{x-y},
    \qquad
    \forall x,y\in\R^d .
    \label{eq:delta2_H_form}
\end{align}

\smallskip
\noindent
\emph{($\Rightarrow$)}
Fix $x\in\R^d$, $s\in\R^d$, and $\alpha>0$, and set
$y=x+\alpha s$. Applying \eqref{eq:delta2_H_form} gives
\begin{align}
    \frac1{nm}\sum_{i=1}^n\sum_{j=1}^m
    \sqn{
    \frac{G_{i,j}(x+\alpha s)-G_{i,j}(x)}{\alpha}}
    \le
    \delta_2^2\sqn{s}.
    \label{eq:delta2_scaled_difference}
\end{align}
Since the functions $f_{i,j}$ are twice continuously differentiable,
\begin{align}
    \frac{G_{i,j}(x+\alpha s)-G_{i,j}(x)}{\alpha}
    \xrightarrow[\alpha\downarrow 0]{}
    \rb{\nabla^2 f_{i,j}(x)-\nabla^2 f_i(x)}s .
    \notag
\end{align}
Taking $\alpha\downarrow 0$ in \eqref{eq:delta2_scaled_difference} yields
\begin{align}
    \frac1{nm}\sum_{i=1}^n\sum_{j=1}^m
    \sqn{
    \rb{\nabla^2 f_{i,j}(x)-\nabla^2 f_i(x)}s}
    \le
    \delta_2^2\sqn{s},
    \qquad
    \forall x,s\in\R^d .
    \label{eq:delta2_directional_hessian}
\end{align}
By the definition of $\mathcal{A}_2(x)$, this is exactly
$\sqn{\mathcal{A}_2(x)s}\le \delta_2^2\sqn{s}$ for all $s$,
which is equivalent to $\opnorm{\mathcal{A}_2(x)}\le \delta_2$.

\smallskip
\noindent
\emph{($\Leftarrow$)}
Assume \eqref{eq:hess_delta2_stacked_opnorm}. Fix $x,y\in\R^d$
and set $s=y-x$. By the fundamental theorem of calculus,
\begin{align}
    G_{i,j}(y)-G_{i,j}(x)
    =
    \int_0^1
    \rb{\nabla^2 f_{i,j}(x+\tau s)-\nabla^2 f_i(x+\tau s)}s
    \,d\tau .
    \notag
\end{align}
Stacking the vectors and using Jensen's inequality in the product space
$(\R^d)^{nm}$, we obtain
\begin{align}
    \frac1{nm}\sum_{i=1}^n\sum_{j=1}^m
    \sqn{G_{i,j}(y)-G_{i,j}(x)}
    &=
    \sqn{
    \int_0^1
    \mathcal{A}_2(x+\tau s)s
    \,d\tau}
    \notag
    \\
    &\le
    \int_0^1
    \sqn{\mathcal{A}_2(x+\tau s)s}
    \,d\tau
    \notag
    \\
    &\le
    \int_0^1
    \delta_2^2\sqn{s}
    \,d\tau
    =
    \delta_2^2\sqn{y-x}.
    \notag
\end{align}
This is \eqref{eq:exp_delta2}.
\end{proof}

\paragraph{Practical takeaway.}
Lemma~\ref{lem:hessian_characterization_delta} reduces the task of understanding $(\delta_1,\delta_2)$ to quantifying \emph{Hessian variability} across groups (for $\delta_1$) and within groups (for $\delta_2$). This perspective guides both the partition design described above and the empirical measurement protocol introduced next.

\subsection{Estimating $L$, $\delta_1$, $\delta_2$ in practice}
\label{sec:exp_estimating_deltas}

We estimate the smoothness and similarity constants directly on the logistic-regression instances used in the experiments. For a data point $(a_{i,j},y_{i,j})$, the logistic part of the component Hessian has the form
\begin{align}
    H^{\rm log}_{i,j}(x)
    &=
    \alpha_{i,j}(x)\,a_{i,j}a_{i,j}^{\top},
    \qquad
    \alpha_{i,j}(x)
    =
    \sigma(-y_{i,j}a_{i,j}^{\top}x)
    \bigl(1-\sigma(-y_{i,j}a_{i,j}^{\top}x)\bigr),
\end{align}
so that $0\leq \alpha_{i,j}(x)\leq 1/4$. The nonconvex regularizer contributes the same diagonal Hessian $\lambda\nabla^2 r(x)$ to every component, where
\begin{align}
    \nabla^2 r(x)
    =
    \Diag\!\rb{
        \frac{2(1-3x_\ell^2)}{(1+x_\ell^2)^3}
    }_{\ell=1}^d .
\end{align}
This regularization term therefore contributes to the smoothness constants but cancels exactly in the Hessian differences that define $\delta_1$ and $\delta_2$.

We report two complementary families of estimates: \emph{data-only worst-case surrogates}, obtained by replacing every $\alpha_{i,j}(x)$ with its uniform upper bound $1/4$, and \emph{empirical constants}, computed on a finite probe set drawn from the optimization trajectory. The empirical constants are the ones actually used for hyperparameter tuning in the main experiments.

\paragraph{Data-only worst-case surrogates.}
These quantities are independent of the optimization trajectory: they use only $\alpha_{i,j}(x)\leq 1/4$ and $\opnorm{\nabla^2 r(x)}\leq 2$. Specifically,
\begin{align}
    L_{\max}^{\rm wc}
    &\eqdef
    \max_{i,j}
    \rb{
        \tfrac{1}{4}\sqn{a_{i,j}} + 2\lambda
    }, \\
    L^{\rm wc}
    &\eqdef
    \frac{1}{4}\,
    \opnorm{
        \frac{1}{nm}\sum_{i=1}^n\sum_{j=1}^m
        a_{i,j}a_{i,j}^{\top}
    }
    + 2\lambda,
\end{align}
and the data-only surrogate similarity constants are
\begin{align}
    \delta_{1,{\rm wc}}
    &\eqdef
    \sqrt{
    \frac{1}{n}\sum_{i=1}^n
    \opnorm{
        \frac{1}{4m}\sum_{j=1}^m a_{i,j}a_{i,j}^{\top}
        -
        \frac{1}{4nm}\sum_{r=1}^n\sum_{s=1}^m
        a_{r,s}a_{r,s}^{\top}
    }^2
    }, \\
    \delta_{2,{\rm wc}}
    &\eqdef
    \sqrt{
    \frac{1}{nm}\sum_{i=1}^n\sum_{j=1}^m
    \opnorm{
        \frac{1}{4}a_{i,j}a_{i,j}^{\top}
        -
        \frac{1}{4m}\sum_{k=1}^m a_{i,k}a_{i,k}^{\top}
    }^2
    } .
\end{align}
These quantities are independent of the optimization trajectory and provide a useful scale for the geometry of the design matrices. Because they replace the logistic weights $\alpha_{i,j}(x)$ by $1/4$ \emph{before} centering, they should be interpreted as conservative diagnostics rather than as exact global upper bounds on the best similarity constants $\delta_1$, $\delta_2$.

\paragraph{Empirical constants on the optimization region.}
The constants used to tune the main experiments are empirical estimates computed on a finite probe set. Starting from $x^0=0$, we run $20$ full-gradient steps,
\begin{align}
    x^{s+1}=x^s-\eta_{\rm probe}\nabla f(x^s),
    \qquad s=0,\ldots,19,
\end{align}
and collect
\begin{align}
    \mathcal{X}_{\rm probe}
    =
    \{x^0,x^{s_1},\ldots,x^{s_5}\},
\end{align}
where $s_1,\ldots,s_5$ are five evenly spaced indices along this short trajectory. The probe stepsize is set from the precomputed worst-case smoothness estimate as
\begin{align}
    \eta_{\rm probe}=\frac{1}{2L^{\rm wc}}.
\end{align}

The empirical constants are then defined as the maxima of the relevant quantities over $\mathcal{X}_{\rm probe}$,
\begin{align}
    \widehat L
    &\eqdef
    \max_{x\in\mathcal{X}_{\rm probe}}
    \opnorm{\nabla^2 f(x)}, \\
    \widehat L_{\max}
    &\eqdef
    \max_{x\in\mathcal{X}_{\rm probe}}
    \max_{i,j}
    \opnorm{\nabla^2 f_{i,j}(x)}, \\
    \widehat\delta_1
    &\eqdef
    \max_{x\in\mathcal{X}_{\rm probe}}
    \sqrt{
    \frac{1}{n}\sum_{i=1}^n
    \opnorm{\nabla^2 f_i(x)-\nabla^2 f(x)}^2
    }, \\
    \widehat\delta_2
    &\eqdef
    \max_{x\in\mathcal{X}_{\rm probe}}
    \sqrt{
    \frac{1}{nm}\sum_{i=1}^n\sum_{j=1}^m
    \opnorm{\nabla^2 f_{i,j}(x)-\nabla^2 f_i(x)}^2
    } .
\end{align}
The empirical quantities $\widehat{\delta}_1$ and $\widehat{\delta}_2$ reported below are conservative operator-norm diagnostics: they average individual operator norms and therefore upper-bound the corresponding stacked-operator norms characterized in Lemma~\ref{lem:hessian_characterization_delta}. We use them as practical tuning proxies, not as exact evaluations of the best similarity constants.

For $\widehat L_{\max}$ we exploit the rank-one form of the component Hessian, which yields the bound
\begin{align}
    \opnorm{\nabla^2 f_{i,j}(x)}
    \leq
    \alpha_{i,j}(x)\sqn{a_{i,j}}
    +
    \lambda\max_{\ell}
    \left|
        \frac{2(1-3x_\ell^2)}{(1+x_\ell^2)^3}
    \right| .
\end{align}

\paragraph{Matrix-free operator-norm estimation.}
All operator norms are computed matrix-free, through Hessian--vector products. For instance,
\begin{align}
    H^{\rm log}_{i,j}(x)\,v
    &=
    \alpha_{i,j}(x)\,a_{i,j}(a_{i,j}^{\top}v), \\
    H^{\rm log}_{i}(x)\,v
    &=
    \frac{1}{m}\sum_{j=1}^m
    \alpha_{i,j}(x)\,a_{i,j}(a_{i,j}^{\top}v), \\
    H^{\rm log}(x)\,v
    &=
    \frac{1}{nm}\sum_{i=1}^n\sum_{j=1}^m
    \alpha_{i,j}(x)\,a_{i,j}(a_{i,j}^{\top}v).
\end{align}
For smoothness estimates of $L$ and $L_{\max}$, the Hessian--vector product includes the diagonal regularizer contribution $\lambda\nabla^2 r(x)\,v$. For similarity estimates, the same regularizer is shared by all components and cancels in the centered Hessian differences, so the corresponding Hessian-difference products use only the logistic-loss part, e.g.\ $(H_i^{\rm log}(x)-H^{\rm log}(x))\,v$ and $(H_{i,j}^{\rm log}(x)-H_i^{\rm log}(x))\,v$.

Each operator norm was estimated by power iteration. We used $20$ power iterations for both the smoothness and the similarity diagnostics, with tolerances $10^{-5}$ for smoothness and $10^{-4}$ for similarity estimates. The $\delta_2$ diagnostic enumerates all $nm$ components and estimates the norm of each $H_{i,j}^{\rm log}(x)-H_i^{\rm log}(x)$ separately; batching is used only for memory efficiency and does not subsample the components.

The resulting empirical constants $\widehat L$, $\widehat L_{\max}$, $\widehat\delta_1$, and $\widehat\delta_2$ used for the synthetic experiments (rounded to two decimal places) are reported in Table~\ref{tab:synthetic-regimes} of Section~\ref{sec:experiments}.

\subsection{Hardware and implementation details}

All experiments were implemented in Python 3.12.8 using PyTorch 2.5.1 (CUDA 11.8 backend, cuDNN 9.1.0). Execution was performed on a workstation featuring dual Intel Xeon Gold 6246 CPUs (24 physical cores, 48 logical threads) alongside an NVIDIA GeForce RTX 3090 (24\,GB) and an RTX 2080 Ti (12\,GB).


\subsection{Hyperparameter tuning}
\label{sec:hyperparameter_tuning}

\subsubsection{Stepsizes}

For all methods in the main synthetic comparison, we used the largest theoretically admissible constant stepsize. Concretely, after fixing the final batch size from the batch-tuning stage (and, for \algname{D-ZeroSARAH}, the final pair $(s,b)$ of client subset size and local batch size), we substituted these values into the theoretical stepsize expressions. \algname{SILAGE}$(m>n)$ involves no batch-size tuning; for \algname{SILAGE}$(n>m)$ we plug in the selected batch size $b_{\mathrm{grp}}$; for flattened \algname{ZeroSARAH} and \algname{SILVER} we plug in the selected flat minibatch size $b$; and for \algname{D-ZeroSARAH} we plug in the selected pair $(s,b)$. For flattened \algname{ZeroSARAH} and \algname{D-ZeroSARAH}, whose finer-grid runs did not reach the target tolerance within the $40$-epoch budget, the selected batch size (or pair $(s,b)$) is the one chosen by the fallback rule described below, namely the configuration achieving the smallest final squared gradient norm at the end of the $40$-epoch budget. The resulting formulas are summarized in Table~\ref{tab:stepsize-formulas}, and the corresponding numerical values are listed in Table~\ref{tab:stepsize-values}. All numerical values are reported rounded to six significant digits.

\begin{table}[t]
\centering
\small
\setlength{\tabcolsep}{5pt}
\caption{Theoretical constant stepsize formulas used in the synthetic experiments. Here $N=nm$ denotes the total sample size, $b$ the minibatch size, and $s$ the client subset size.}
\label{tab:stepsize-formulas}
\begin{tabular}{ll}
\toprule
\textbf{Method} & \textbf{Stepsize formula} \\
\midrule
\algname{SILAGE}$(m>n)$ &
$\gamma = \rb{L + \delta_2 \sqrt{\frac{n-p}{np}}}^{-1}, \qquad p=\frac{n}{m}$ \\[0.75em]

\algname{SILAGE}$(n>m)$ &
$\gamma = \rb{L + \sqrt{\delta_1^2 \frac{(n-b_{\mathrm{grp}})(5n+4)}{n-1}
+ \delta_2^2 \frac{n^2+2n-7b_{\mathrm{grp}}+4b_{\mathrm{grp}}^2}{2b_{\mathrm{grp}}n}}}^{-1}$ \\[0.75em]

\algname{ZeroSARAH} &
$\gamma = \rb{L_{\max}\rb{1+\frac{\sqrt{8N}}{b}}}^{-1}$ \\[0.75em]

\algname{SILVER} &
$\gamma = \min\!\left\{\frac{1}{L_{\max}}, \frac{b}{\delta_{\mathrm{flat}}\sqrt{N}}\right\}$ \\[0.75em]

\algname{D-ZeroSARAH} &
$\gamma = \rb{L_{\max}\rb{1+\frac{\sqrt{8nm}}{sb}}}^{-1}$ \\
\bottomrule
\end{tabular}
\end{table}

\begin{longtable}{lllccc}
\caption{Numerical constant stepsizes used in the synthetic experiments, obtained by evaluating the formulas of Table~\ref{tab:stepsize-formulas} at the final selected batch size $b$ and, where applicable, client subset size $s$. A dash (---) indicates that the corresponding hyperparameter is not used by the method.}
\label{tab:stepsize-values}\\
\toprule
\textbf{Method} & \textbf{Setting} & \textbf{Regime} & $b$ & $s$ & $\gamma$ \\
\midrule
\endfirsthead

\toprule
\textbf{Method} & \textbf{Setting} & \textbf{Regime} & $b$ & $s$ & $\gamma$ \\
\midrule
\endhead

\bottomrule
\endfoot

\algname{SILAGE}$(m>n)$ & $m>n$ & $(\delta_1 \text{ small}, \delta_2 \text{ small})$ & --- & --- & $2.490563 \times 10^{-3}$ \\
\algname{SILAGE}$(m>n)$ & $m>n$ & $(\delta_1 \text{ small}, \delta_2 \text{ large})$ & --- & --- & $1.492335 \times 10^{-3}$ \\
\algname{SILAGE}$(m>n)$ & $m>n$ & $(\delta_1 \text{ large}, \delta_2 \text{ small})$ & --- & --- & $2.408654 \times 10^{-3}$ \\
\algname{SILAGE}$(m>n)$ & $m>n$ & $(\delta_1 \text{ large}, \delta_2 \text{ large})$ & --- & --- & $1.579522 \times 10^{-3}$ \\
\midrule
\algname{SILAGE}$(n>m)$ & $n>m$ & $(\delta_1 \text{ small}, \delta_2 \text{ small})$ & 6 & --- & $2.486395 \times 10^{-3}$ \\
\algname{SILAGE}$(n>m)$ & $n>m$ & $(\delta_1 \text{ small}, \delta_2 \text{ large})$ & 1 & --- & $4.107040 \times 10^{-4}$ \\
\algname{SILAGE}$(n>m)$ & $n>m$ & $(\delta_1 \text{ large}, \delta_2 \text{ small})$ & 1 & --- & $1.562372 \times 10^{-4}$ \\
\algname{SILAGE}$(n>m)$ & $n>m$ & $(\delta_1 \text{ large}, \delta_2 \text{ large})$ & 1 & --- & $1.943185 \times 10^{-4}$ \\
\midrule
\algname{ZeroSARAH} & $m>n$ & $(\delta_1 \text{ small}, \delta_2 \text{ small})$ & 192 & --- & $9.07760 \times 10^{-4}$ \\
\algname{ZeroSARAH} & $m>n$ & $(\delta_1 \text{ small}, \delta_2 \text{ large})$ & 192 & --- & $6.77773 \times 10^{-4}$ \\
\algname{ZeroSARAH} & $m>n$ & $(\delta_1 \text{ large}, \delta_2 \text{ small})$ & 11  & --- & $6.03772 \times 10^{-5}$ \\
\algname{ZeroSARAH} & $m>n$ & $(\delta_1 \text{ large}, \delta_2 \text{ large})$ & 11  & --- & $5.36743 \times 10^{-5}$ \\
\algname{ZeroSARAH} & $n>m$ & $(\delta_1 \text{ small}, \delta_2 \text{ small})$ & 192 & --- & $9.07937 \times 10^{-4}$ \\
\algname{ZeroSARAH} & $n>m$ & $(\delta_1 \text{ small}, \delta_2 \text{ large})$ & 31  & --- & $1.42545 \times 10^{-4}$ \\
\algname{ZeroSARAH} & $n>m$ & $(\delta_1 \text{ large}, \delta_2 \text{ small})$ & 31  & --- & $1.42782 \times 10^{-4}$ \\
\algname{ZeroSARAH} & $n>m$ & $(\delta_1 \text{ large}, \delta_2 \text{ large})$ & 31  & --- & $1.42554 \times 10^{-4}$ \\
\midrule
\algname{SILVER} & $m>n$ & $(\delta_1 \text{ small}, \delta_2 \text{ small})$ & 46  & --- & $2.402861 \times 10^{-3}$ \\
\algname{SILVER} & $m>n$ & $(\delta_1 \text{ small}, \delta_2 \text{ large})$ & 128 & --- & $1.794076 \times 10^{-3}$ \\
\algname{SILVER} & $m>n$ & $(\delta_1 \text{ large}, \delta_2 \text{ small})$ & 96  & --- & $1.796170 \times 10^{-3}$ \\
\algname{SILVER} & $m>n$ & $(\delta_1 \text{ large}, \delta_2 \text{ large})$ & 96  & --- & $1.596704 \times 10^{-3}$ \\
\algname{SILVER} & $n>m$ & $(\delta_1 \text{ small}, \delta_2 \text{ small})$ & 46  & --- & $2.403326 \times 10^{-3}$ \\
\algname{SILVER} & $n>m$ & $(\delta_1 \text{ small}, \delta_2 \text{ large})$ & 128 & --- & $1.596628 \times 10^{-3}$ \\
\algname{SILVER} & $n>m$ & $(\delta_1 \text{ large}, \delta_2 \text{ small})$ & 96  & --- & $1.599284 \times 10^{-3}$ \\
\algname{SILVER} & $n>m$ & $(\delta_1 \text{ large}, \delta_2 \text{ large})$ & 50  & --- & $1.596730 \times 10^{-3}$ \\
\midrule
\algname{D-ZeroSARAH} & $m>n$ & $(\delta_1 \text{ small}, \delta_2 \text{ small})$ & 6 & 1  & $4.47422 \times 10^{-5}$ \\
\algname{D-ZeroSARAH} & $m>n$ & $(\delta_1 \text{ small}, \delta_2 \text{ large})$ & 1 & 31 & $1.60172 \times 10^{-4}$ \\
\algname{D-ZeroSARAH} & $m>n$ & $(\delta_1 \text{ large}, \delta_2 \text{ small})$ & 1 & 31 & $1.60359 \times 10^{-4}$ \\
\algname{D-ZeroSARAH} & $m>n$ & $(\delta_1 \text{ large}, \delta_2 \text{ large})$ & 1 & 31 & $1.42552 \times 10^{-4}$ \\
\algname{D-ZeroSARAH} & $n>m$ & $(\delta_1 \text{ small}, \delta_2 \text{ small})$ & 1 & 41 & $2.75836 \times 10^{-4}$ \\
\algname{D-ZeroSARAH} & $n>m$ & $(\delta_1 \text{ small}, \delta_2 \text{ large})$ & 1 & 96 & $3.71824 \times 10^{-4}$ \\
\algname{D-ZeroSARAH} & $n>m$ & $(\delta_1 \text{ large}, \delta_2 \text{ small})$ & 1 & 96 & $3.72454 \times 10^{-4}$ \\
\algname{D-ZeroSARAH} & $n>m$ & $(\delta_1 \text{ large}, \delta_2 \text{ large})$ & 1 & 96 & $3.71847 \times 10^{-4}$ \\
\end{longtable}

\subsubsection{Batch-size tuning}
We used method- and setting-specific grids for batch-size tuning. No batch-size tuning was performed for \algname{SILAGE}$(m>n)$. For \algname{SILAGE}$(n>m)$ in the $n>m$ setting, we searched over
\begin{equation*}
\begin{aligned}
\{ &1,\, 6,\, 11,\, 16,\, 21,\, 26,\, 31,\, 36,\, 41,\, 46,\, 51,\, 56,\, 61,\, 66,\, 71,\, 76,\, 81,\, 86,\, 91,\, 96,\, \\
   &101,\, 106,\, 111,\, 116,\, 121,\, 126,\, 131,\, 136,\, 141,\, 146,\, 151,\, 156,\, 161,\, 166,\, 171,\, 176,\, \\
   &181,\, 186,\, 191,\, 196,\, 201,\, 206,\, 211,\, 216,\, 221,\, 226,\, 231,\, 236,\, 241,\, 246,\, 250 \}.
\end{aligned}
\end{equation*}
For flattened \algname{ZeroSARAH}, in both the $m>n$ and $n>m$ settings, we searched over
\begin{equation*}
\begin{aligned}
\{ &1,\, 6,\, 11,\, 16,\, 21,\, 26,\, 31,\, 36,\, 41,\, 46,\, 64,\, 96,\, 128,\, 192,\, 256,\, 384,\, 512,\, \\
   &768,\, 1024,\, 1536,\, 2048,\, 3072,\, 4096,\, 6144,\, 8192,\, 10240,\, 12500 \}.
\end{aligned}
\end{equation*}
For \algname{SILVER}, we used the combined grid
\begin{equation*}
\begin{aligned}
\{ &1,\, 6,\, 11,\, 16,\, 21,\, 26,\, 31,\, 36,\, 41,\, 46,\, 50,\, 64,\, 96,\, 128,\, 192,\, 256,\, 384,\, 512,\, \\
   &768,\, 1024,\, 1536,\, 2048,\, 3072,\, 4096,\, 6144,\, 8192,\, 10240,\, 12500 \},
\end{aligned}
\end{equation*}
whose lower part had already been computed and whose upper part was added later. For \algname{D-ZeroSARAH}, we used setting-dependent grids: in the $m>n$ setting,
\begin{equation*}
\begin{aligned}
b &\in \{1,\, 6,\, 11,\, 16,\, 21,\, 26,\, 31,\, 36,\, 41,\, 46,\, 56,\, 66,\, 76,\, 86,\, 96,\, 111,\, 126,\, 151,\, 176,\, 201,\, 226,\, 250\}, \\
s &\in \{1,\, 6,\, 11,\, 16,\, 21,\, 26,\, 31,\, 36,\, 41,\, 46,\, 50\},
\end{aligned}
\end{equation*}
where $b$ is the local batch size and $s$ is the client subset size; in the $n>m$ setting,
\begin{equation*}
\begin{aligned}
b &\in \{1,\, 6,\, 11,\, 16,\, 21,\, 26,\, 31,\, 36,\, 41,\, 46,\, 50\}, \\
s &\in \{1,\, 6,\, 11,\, 16,\, 21,\, 26,\, 31,\, 36,\, 41,\, 46,\, 56,\, 66,\, 76,\, 86,\, 96,\, 111,\, 126,\, 151,\, 176,\, 201,\, 226,\, 250\}.
\end{aligned}
\end{equation*}
Whenever a tuning grid contained runs that reached the target tolerance within the $40$-epoch budget, we selected the final batch size (or pair $(s,b)$ for \algname{D-ZeroSARAH}) as the one requiring the smallest number of epochs to converge. For the finer-grid \algname{ZeroSARAH} and \algname{D-ZeroSARAH} sweeps, none of the runs reached the target tolerance within $40$ epochs; in those cases, we applied a fallback rule and selected the configuration with the smallest final squared gradient norm at the end of the budget.

\begin{table}[t]
\centering
\small
\setlength{\tabcolsep}{5pt}
\caption{Final batch sizes selected for \algname{SILAGE}$(n>m)$ in the $n>m$ setting across the four similarity regimes.}
\begin{tabular}{lc}
\toprule
\textbf{Regime} & $b_{\mathrm{grp}}$ \\
\midrule
$(\delta_1 \text{ small}, \delta_2 \text{ small})$ & 6 \\
$(\delta_1 \text{ small}, \delta_2 \text{ large})$ & 1 \\
$(\delta_1 \text{ large}, \delta_2 \text{ small})$ & 1 \\
$(\delta_1 \text{ large}, \delta_2 \text{ large})$ & 1 \\
\bottomrule
\end{tabular}
\end{table}

\begin{table}[t]
\centering
\small
\setlength{\tabcolsep}{5pt}
\caption{Final flat minibatch sizes selected for flattened \algname{ZeroSARAH} via the fallback rule on the finer grid.}
\begin{tabular}{llc}
\toprule
\textbf{Setting} & \textbf{Regime} & $b$ \\
\midrule
$m>n$ & $(\delta_1 \text{ small}, \delta_2 \text{ small})$ & 192 \\
$m>n$ & $(\delta_1 \text{ small}, \delta_2 \text{ large})$ & 192 \\
$m>n$ & $(\delta_1 \text{ large}, \delta_2 \text{ small})$ & 11 \\
$m>n$ & $(\delta_1 \text{ large}, \delta_2 \text{ large})$ & 11 \\
\midrule
$n>m$ & $(\delta_1 \text{ small}, \delta_2 \text{ small})$ & 192 \\
$n>m$ & $(\delta_1 \text{ small}, \delta_2 \text{ large})$ & 31 \\
$n>m$ & $(\delta_1 \text{ large}, \delta_2 \text{ small})$ & 31 \\
$n>m$ & $(\delta_1 \text{ large}, \delta_2 \text{ large})$ & 31 \\
\bottomrule
\end{tabular}
\end{table}

\begin{table}[t]
\centering
\small
\setlength{\tabcolsep}{5pt}
\caption{Final flat minibatch sizes selected for \algname{SILVER} in each setting and similarity regime.}
\begin{tabular}{llc}
\toprule
\textbf{Setting} & \textbf{Regime} & $b$ \\
\midrule
$m>n$ & $(\delta_1 \text{ small}, \delta_2 \text{ small})$ & 46 \\
$m>n$ & $(\delta_1 \text{ small}, \delta_2 \text{ large})$ & 128 \\
$m>n$ & $(\delta_1 \text{ large}, \delta_2 \text{ small})$ & 96 \\
$m>n$ & $(\delta_1 \text{ large}, \delta_2 \text{ large})$ & 96 \\
\midrule
$n>m$ & $(\delta_1 \text{ small}, \delta_2 \text{ small})$ & 46 \\
$n>m$ & $(\delta_1 \text{ small}, \delta_2 \text{ large})$ & 128 \\
$n>m$ & $(\delta_1 \text{ large}, \delta_2 \text{ small})$ & 96 \\
$n>m$ & $(\delta_1 \text{ large}, \delta_2 \text{ large})$ & 50 \\
\bottomrule
\end{tabular}
\end{table}

\begin{table}[t]
\centering
\small
\setlength{\tabcolsep}{5pt}
\caption{Final local batch sizes $b$ and client subset sizes $s$ selected for \algname{D-ZeroSARAH} via the fallback rule on the finer grid.}
\begin{tabular}{llcc}
\toprule
\textbf{Setting} & \textbf{Regime} & $b$ & $s$ \\
\midrule
$m>n$ & $(\delta_1 \text{ small}, \delta_2 \text{ small})$ & 6 & 1 \\
$m>n$ & $(\delta_1 \text{ small}, \delta_2 \text{ large})$ & 1 & 31 \\
$m>n$ & $(\delta_1 \text{ large}, \delta_2 \text{ small})$ & 1 & 31 \\
$m>n$ & $(\delta_1 \text{ large}, \delta_2 \text{ large})$ & 1 & 31 \\
\midrule
$n>m$ & $(\delta_1 \text{ small}, \delta_2 \text{ small})$ & 1 & 41 \\
$n>m$ & $(\delta_1 \text{ small}, \delta_2 \text{ large})$ & 1 & 96 \\
$n>m$ & $(\delta_1 \text{ large}, \delta_2 \text{ small})$ & 1 & 96 \\
$n>m$ & $(\delta_1 \text{ large}, \delta_2 \text{ large})$ & 1 & 96 \\
\bottomrule
\end{tabular}
\end{table}

\end{document}